\documentclass[11pt]{article}
\usepackage{amsmath, amsthm, amssymb, bbm, enumitem, xparse, xspace, multirow, tablefootnote}%bbm used for \mathbbm{1}
\usepackage[margin=1in]{geometry}%geometry is the best package for setting margins
\usepackage[round]{natbib}
\usepackage[dvipsnames]{xcolor}
\usepackage{graphicx}
\usepackage{float}
\usepackage{subcaption}
\usepackage{hyperref}
\hypersetup{
   breaklinks=true,
   pdfusetitle=true,
}
\usepackage{makecell}
\usepackage{cleveref}
\crefname{subsection}{subsection}{subsections}
\usepackage{algorithm,algpseudocode}

\allowdisplaybreaks[3]

\theoremstyle{plain}
\newtheorem{theorem}{Theorem}%[section]

\theoremstyle{definition}
\newtheorem{definition}[theorem]{Definition}

\newcommand{\Ber}{\mathrm{Ber}}
\newcommand{\Bin}{\mathrm{Bin}}

\newcommand{\eps}{\varepsilon}
\newcommand{\bI}{\mathbbm{1}}
\newcommand{\bE}{\mathbb{E}}

\newcommand{\ha}{{\hat{a}}}
\newcommand{\hF}{{\hat{F}}}

\usepackage{color}              % Need the color package
\usepackage{color-edits}
\usepackage[english]{babel}
\usepackage[autostyle, english = american]{csquotes}
\MakeOuterQuote{"}

\NewDocumentEnvironment{myproof}{o}
{\IfNoValueTF{#1}{\paragraph{{Proof.} }} {\paragraph{{#1.} }} }
{\hfill$\qedsymbol$}

\title{Survey of Data-driven Newsvendor: Unified Analysis and Spectrum of Achievable Regrets}
\author{
Zhuoxin Chen\thanks{Weiyang College, Tsinghua University.  \href{mailto:zhuoxin-22@mails.tsinghua.edu.cn}{zhuoxin-22@mails.tsinghua.edu.cn}.}
\\ Tsinghua University
\and
Will Ma\thanks{Graduate School of Business and Data Science Institute, Columbia University.  \href{mailto:wm2428@gsb.columbia.edu}{wm2428@gsb.columbia.edu}.}
\\ Columbia University
}
\date{}
\begin{document}
\maketitle

\begin{abstract}
In the Newsvendor problem, the goal is to guess the number that will be drawn from some distribution, with asymmetric consequences for guessing too high vs.\ too low.
In the data-driven version, the distribution is unknown, and one must work with samples from the distribution.
Data-driven Newsvendor has been studied under many variants: additive vs.\ multiplicative regret, high-probability vs.\ expectation bounds, and different distribution classes.
This paper studies all combinations of these variants, filling many gaps in the literature and simplifying many proofs.
In particular, we provide a unified analysis based on a notion of clustered distributions, which in conjunction with our new lower bounds,  shows that the entire spectrum of regrets between $1/\sqrt{n}$ and $1/n$ is possible.
Simulations on commonly-used distributions demonstrate that our notion is the "correct" predictor of empirical regret across varying data sizes.

\textbf{Keywords}: data-driven decision-making, Newsvendor, distribution classes, learning theory
\end{abstract}
%\clearpage

\section{Introduction}

\linespread{1.5}
In decision-making under uncertainty, one chooses an action $a$ in the face of an uncertain outcome $Z$, and the loss incurred $\ell(a,Z)$ follows a given function $\ell$.
In stochastic optimization, the outcome $Z$ is drawn from a known distribution $F$, and the goal is to minimize the expected loss $\bE_{Z\sim F}[\ell(a,Z)]$.
We let $L(a)$ denote the expected loss of an action $a$, and $a^*$ denote an optimal action for which $L(a^*)=\inf_a L(a)$.
In data-driven optimization, the distribution $F$ is unknown, and one must instead work with independent and identically distributed (IID) samples drawn from $F$.
A data-driven algorithm prescribes an action $\ha$ based on these samples, and one is interested in how its expected loss $L(\ha)$ compares to the optimal expected loss $L(a^*)$ from stochastic optimization.

This comparison can be made in a multitude of ways, differing along various dimensions.
First, one can measure either the difference $L(\ha)-L(a^*)$ which is called the \textit{additive regret}, or the scaled difference $(L(\ha)-L(a^*))/L(a^*)$ which is called the \textit{multiplicative regret}.
Second, note that both of these regrets are random variables, because $L(\ha)$ depends on the IID samples drawn;
therefore, one can analyze either the probability that the regret is below some threshold, or analyze the expected regret.
Finally, different restrictions can be placed on the unknown distribution $F$.

% In this paper we consider the data-driven Newsvendor problem, in which action $a$ represents an amount of inventory to stock, and outcome $Z$ represents an uncertain demand to occur.
In this paper we consider the multitude of ways in which $L(\ha)$ has been compared to $L(a^*)$ in the data-driven Newsvendor problem, starting with the work of \citet{levi2007provably}.
In the Newsvendor problem, action $a$ represents an amount of inventory to stock, and outcome $Z$ represents an uncertain demand to occur.
The loss function is given by
$$\ell(a,Z)=c_u\max\{Z-a,0\}+c_o\max\{a-Z,0\},$$
where $c_u,c_o>0$ represent the unit costs of understocking, overstocking respectively.
The goal in Newsvendor is to stock inventory close to demand, but err on the side of understocking or overstocking depending on how the costs $c_u,c_o$ compare.
The optimal action when $F$ is known involves defining $q=\frac{c_u}{c_u+c_o}$, and then setting $a^*$ to be a $q$'th percentile realization from $F$, with $q$ being called the \textit{critical quantile}.

\subsection{Existing and New Results}

We first define a restriction to be placed on the unknown distribution $F$, that is
similar to the notion of clustered distributions from \citet{besbes2022multi}, but used for a completely different problem.

\begin{definition}
Fix a Newsvendor loss function with critical quantile $q\in(0,1)$.
For constants $\beta\in[0,\infty]$ and $\gamma,\zeta>0$, a distribution with CDF $F$ is said to be \textit{$(\beta,\gamma,\zeta)$-clustered} if
\begin{align} \label{eqn:clustered}
|a-a^*| &\le \frac1\gamma |F(a)-q|^{\frac1{\beta+1}} &\forall a\in[a^*-\zeta,a^*+\zeta].
\end{align}
\end{definition}

% We make the following remarks on the definition of $(\beta,\gamma,\zeta)$-clustered distributions.

For Newsvendor, the notion of $(\beta,\gamma,\zeta)$-clustered distributions captures how far an action $a$ can deviate from the optimal action $a^*$, based on how far away $F(a)$ is from the critical quantile $q$.  Data-driven Newsvendor algorithms typically provide a guarantee on $|F(\ha)-q|$ for their action $\ha$, and hence~\eqref{eqn:clustered} would imply a guarantee on $|\ha-a^*|$, affecting the regret.

Constraint~\eqref{eqn:clustered} is most restrictive and leads to the smallest regrets when $\beta=0$.  Technically, $\beta=0$ can be satisfied for any distribution with a density at $a^*$, that is lower-bounded by $\gamma$ for a sufficiently small $\zeta$, a notion studied in past works \citep{besbes2013implications,lin2022data}.  However, our notion is based on the CDF instead of the minimum density, and we will show that it is the "correct" notion for making data-size-dependent comparisons that explain how empirical regrets compare across commonly-used distributions (see our simulations in \Cref{sec:simulation}). 
% and we will see that this is the "correct" predictor of empirical regret that accounts for different data sizes.

Meanwhile, a distribution whose CDF has a discrete jump at $a^*$ can only be captured by $\beta=\infty$, leading to the slowest convergence rates on regret.  For every intermediate value in $(0,\infty)$, we also construct a distribution that can only be captured by a $\beta$ at least that value, in \Cref{apx:full_spectrum_egs}.

In general, a distribution may admit multiple valid combinations of $(\beta,\gamma,\zeta)$, because the value of $\beta$ depends on $\gamma$ and $\zeta$. In \Cref{apx:full_spectrum_egs}, we also illustrate how to compute the minimum possible value of $\beta$ given fixed $\gamma$ and $\zeta$, for several commonly-used distributions.

\begin{table}[]
\centering
\begin{tabular}{|c|c|c|}
\hline
& Additive Regret\tablefootnote{When $\beta=\infty$, upper bounds for additive regret necessarily require the additional assumption that the distribution has bounded mean. Our expectation upper bounds for additive regret assume bounded mean for all $\beta\in[0,\infty]$, noting that \citet{lin2022data} also assume bounded mean for both $\beta=0$ and $\beta=\infty$.} & Multiplicative Regret\tablefootnote{Our results for multiplicative regret require the additional assumption that $F(a^*-\zeta),F(a^*+\zeta)$ are bounded away from $0,1$ respectively, which is necessary to exploit the restriction of $(\beta,\gamma,\zeta)$-clustered distributions.  Previous papers did not require this assumption because they only considered $\beta=\infty$ (equivalent to having no restriction).} \\
\hline
High-probability & $O\left((\frac{\log(1/\delta)}{n})^{\frac{\beta+2}{2\beta+2}}\right)$ (\textbf{\Cref{thm:hpAdd}}) & $O\left((\frac{\log(1/\delta)}{n})^{\frac{\beta+2}{2\beta+2}}\right)$ (\textbf{\Cref{thm:hpMult}}) \\
Upper Bound & & $\beta=\infty$ known \citep{levi2007provably} \\
\hline
Expectation & $O\left(n^{-\frac{\beta+2}{2\beta+2}}\right)$ (\textbf{\Cref{thm:expAdd}}) & $O\left(n^{-\frac{\beta+2}{2\beta+2}}\right)$ (\textbf{\Cref{thm:expMult}}) \\
Upper Bound& $\beta=0,\infty$ known \citep{lin2022data} & $\beta=\infty$ known \citep{besbes2023big}
\tablefootnote{\citet[Lem.~E-5]{besbes2023big} attribute this result to \citet{levi2015data}, but to the best of our understanding, \citet[Thm.~2]{levi2015data} is insufficient because its proof only holds for $\epsilon\le 1$.  Therefore, we attribute this result to \citet[Thm.~5]{besbes2023big} instead.  We note that \citet[Lem.~E-5]{besbes2023big} works if an upper bound on the mean is known and one uses \textit{projected SAA} instead---see \Cref{sec:projSAA}.
} \\
\hline
\multirow{2}{*}{Lower Bound} & \multicolumn{2}{c|}{$\Omega\left(n^{-\frac{\beta+2}{2\beta+2}}\right)$ (\textbf{\Cref{thm:lowAdd}})} \\
& \multicolumn{2}{c|}{$\beta=0,\infty$ known \citep{zhang2020closing,lyu2024closing}} \\
\hline
\end{tabular}
\caption{High-probability (with probability at least $1-\delta$) and Expectation Upper Bounds on the Additive and Multiplicative Regrets of SAA, when there are $n$ samples and $F$ is restricted to be $(\beta,\gamma,\zeta)$-clustered.
Some results for $\beta=\infty$ (no restriction) and $\beta=0$ (under the stronger assumption that $F$ has density at least $\gamma$ over the interval $[a^*-\zeta,a^*+\zeta]$) were previously known.  The Lower Bound holds with a constant probability.}
\label{tab:outline}
\end{table}

Having defined $(\beta,\gamma,\zeta)$-clustered distributions, our main results are summarized in \Cref{tab:outline}.
To elaborate, we consider the standard Sample Average Approximation (SAA) algorithm for Newsvendor, which sets $\ha$ equal to the $q$'th percentile of the empirical distribution formed by $n$ IID samples.
We provide upper bounds on its additive and multiplicative regrets, that hold with high probability (i.e., with probability at least $1-\delta$ for some small $\delta$) and in expectation.
The $O(\cdot)$ notation highlights the dependence on $n$ and $\delta$, noting that the parameter $\beta$ affects the rate of convergence as $n\to\infty$, whereas the other parameters $q,\gamma,\zeta$ may only affect the constants in front which are second order and hidden.
We recover convergence rates of $n^{-1/2}$ when $\beta=\infty$ and $n^{-1}$ when $\beta=0$, which were previously known\footnote{These results are sometimes stated in terms of \textit{cumulative} regret in their respective papers, in which case the $n^{-1/2}$ rate translates to $\sum_{n=1}^N n^{-1/2}=\Theta(\sqrt{N})$ cumulative regret while the $n^{-1}$ rate translates to $\sum_{n=1}^N n^{-1}=\Theta(\log N)$ cumulative regret.} in some cases as outlined in \Cref{tab:outline}.
Our results establish these convergence rates in all cases, unifying the literature, and moreover showing that the entire spectrum of rates from $1/\sqrt{n}$ (slowest) to $1/n$ (fastest) is possible as $\beta$ ranges from $\infty$ to 0.

Our general upper bound of $n^{-\frac{\beta+2}{2\beta+2}}$ was achieved by the SAA algorithm, which did not need to know any of the parameters $\beta,\gamma,\zeta$ for the clustered distributions.
Meanwhile, our lower bound states that even knowing these parameters, any data-driven algorithm that draws $n$ samples will incur $\Omega\left(n^{-\frac{\beta+2}{2\beta+2}}\right)$ additive regret with a constant probability.  This is then translated into similar lower bounds for multiplicative regret and in expectation.

\paragraph{Technical highlights.}
Our high-probability upper bounds are proven using the fact that $F(\ha)$ is usually close to $q$, which follows the proof framework of \citet{levi2007provably}.  We extend their analysis to additive regret, and also show how to exploit assumptions about lower-bounded density (i.e.,~$\beta=0$) under this proof framework.
Moreover, we introduce the notion of clustered distributions for data-driven Newsvendor, which connects the two extremes cases of no assumption ($\beta=\infty$) and lower-bounded density ($\beta=0$).

Our expectation upper bounds are proven by analyzing an integral (see~\eqref{eqn:expDiff}) which follows \citet{lin2022data}, who bounded the expected additive regret for $\beta=0,\infty$.
We unify their results by considering all $\beta\in[0,\infty]$, and our $\beta=0$ result additionally allows for discrete distributions that are $(0,\gamma,\zeta)$-clustered, instead of imposing that the distribution has a density.
Our proof also uses Chebyshev's inequality to provide tail bounds for extreme quantiles, which simplifies the proof from \citet{lin2022data}. Importantly, this allows for a linear dependence on the mean of the distribution, instead of the quadratic dependence from \citet[]{lin2022data}.
Finally, we recycle their integral to analyze expected multiplicative regret, which when $\beta=\infty$ leads to a simplified proof of \citet[Thm.~2]{besbes2023big} on the exact worst-case expected multiplicative regret of SAA.

Our lower bound is based on a single construction that establishes the tight rate of $\Theta(n^{-\frac{\beta+2}{2\beta+2}})$ for the entire spectrum of $\beta\in[0,\infty]$.
We construct distributions with low Hellinger distance between them \citep[see e.g.][]{guo2021generalizing,jin2024sample}, which leads to simpler distributions and arguably simpler analysis compared to other lower bounds in the data-driven Newsvendor literature (e.g. \citet[Prop.~1]{zhang2020closing}, \citet[Thm.~1]{lin2022data}, \citet[Thm.~2]{lyu2024closing}).  In the special case where $\beta=0$, we establish a lower bound of $\Omega(1/n)$ using a completely different approach than the Bayesian inference and van Trees inequality approach used in \citet{besbes2013implications,lyu2024closing}.  We come up with two candidate distributions, instead of a Bayesian prior over a continuum of candidate distributions; our lower bound holds with constant probability, instead of only in expectation; however, our two distributions change with $n$, whereas they design one prior distribution that works for all $n$.  We provide a self-contained construction for the $\beta=0$ case in \Cref{apx:continuous}.

\paragraph{High-level takeaway.} Our paper answers the question, "For which distributions are Newsvendor decisions hard to learn?"  Importantly, the answer depends on the data size $n$, where we empirically demonstrate in \Cref{sec:simulation} that a distribution $F_1$ may be harder than another distribution $F_2$ at small data sizes, but easier at large data sizes.  Our notion of $(\beta,\gamma,\zeta)$-clustered distributions based on the CDF captures this phenomenon, unlike previous notions based on the PDF.

We should note that both our theory and empirics assume the usage of SAA, which is the prevailing algorithm for data-driven Newsvendor.  In practice, if one suspects a distribution that is hard to learn for SAA at the given data size $n$, then two options are to use a robustified algorithm \citep[e.g.][]{perakis2008regret,gupta2022data,besbes2023big,besbes2025beyond} or to collect more data \citep[e.g.][]{zhang2024more}.

\subsection{Further Related Work}

\paragraph{Learning theory.}
Sample complexity has roots in statistical learning theory, which typically studies classification and regression problems under restricted hypothesis classes \citep{shalev-shwartz2014understanding,mohri2018foundations}.
Its concepts can also be extended to general decision problems \citep{balcan2020data,balcan2021much}, or even specific inventory policy classes \citep{xie2024vc}.
However, data-driven Newsvendor results differ for various reasons: considering multiplicative regret instead of only additive regret, having a specialized but unbounded loss function (there are no assumptions on demand being bounded), and typically requiring analyses that are tighter than uniform convergence.
In data-driven Newsvendor, it is also difficult to directly convert high-probability bounds into expectation bounds unless one knows an upper bound on the mean, because the regret can be unbounded, while high-probability bounds only hold for small values of $\eps$ or equivalently large values of $n$ (see \Cref{sec:projSAA}).
Our results further differ by considering specific restrictions on the distribution $F$.

\paragraph{Generalizations of data-driven Newsvendor.}
Big-data Newsvendor is a generalization of data-driven Newsvendor where past demand samples are accompanied by contextual information, and the decision can be made knowing the future context.
This model was popularized by \citet{ban2019big}, and motivated by the notion of contexts from machine learning.
Meanwhile, data-driven inventory is a generalization of data-driven Newsvendor where one is re-stocking a durable good over multiple periods, that was also considered in the original paper by \citet{levi2007provably}.
Further variants include censored demands when sales are lost \citep[e.g.][]{huh2009nonparametric,besbes2013implications,zhang2020closing,hssaine2024data}, capacitated order sizes \citep[e.g.][]{cheung2019sampling}, and pricing \citep[e.g.][]{chen2021nonparametric,chen2022dynamic,chen2024optimal}.
Our paper focuses on a single period without contexts, and does not aim to cover these generalizations.

\paragraph{Notions related to clustered distributions.}
% \Zhuoxincomment{This part is finished. Do we need to include the proofs somewhere?}
Some conditions in the literature share a similar form with the notion of clustered distributions, for example the Tsybakov noise condition in supervised classification \citep{mammen1999smooth,tsybakov2004optimal}, and the margin condition in contextual bandits \citep{rigollet2010nonparametric}.
While algebraically similar in form, these conditions benefit data-driven algorithms in a different way: they typically improve the separability between two competing options, such as labels, sampling distributions, or reward functions. In contrast, our notion of clustered distributions focuses on the local property of a single distribution and helps the SAA algorithm by limiting the deviation of $\ha$ from $a^*$, given the deviation of $F(\ha)$ from $F(a^*)$, thus preventing large regret. 

Meanwhile, other works impose alternative assumptions on the underlying distribution to achieve similar faster rates for SAA on data-driven Newsvendor.
An example is the Increasing Failure Rate (IFR) property, which requires $1-F$ to be log-concave. 
Under this assumption, and the assumption that $F$ is a continuous distribution, \citet[Cor.~3]{zhang_sampling-based_2025} establishes a sample complexity of $O\left((1+\eps^{-1/2}+\eps^{-1})\log(1/\delta)\right)$, where $\eps$ is the multiplicative regret.
When $\eps$ is close to 0, this result implies that the high-probability multiplicative regret is $O\left(\frac{\log(1/\delta)}{n}\right)$, which is the same as our result for clustered distributions with $\beta=0$.
In fact, we show in \Cref{apx:IFR_proof} that any continuous distribution with the IFR property is $(0,\gamma,\zeta)$-clustered for some $\gamma$ and $\zeta$, and therefore our result for $\beta=0$ can be viewed as a generalization of their result.

Finally, our condition~\eqref{eqn:clustered} can be viewed as a "local" version of the condition from \citet{besbes2022multi}, where we only check for clustering in a small neighborhood around $a^*$.  This also resembles the local conditions used in \citet{balseiro2023survey,balseiro2023dynamic} for online resource allocation.

\section{Preliminaries}

In the Newsvendor problem, we make an ordering decision $a$, and then a random demand $Z$ is drawn from a distribution with CDF $F$.
The domain for $a$, $Z$, and $F$ is $[0,\infty)$.
The loss when we order $a$ and demand realizes to be $Z$ is defined as $$\ell(a,Z)=q\max\{Z-a,0\}+(1-q)\max\{a-Z,0\},$$
for some known $q\in(0,1)$, where we have normalized the unit costs of understocking, overstocking to be $q,1-q$ respectively so that the critical quantile (as defined in the Introduction) is exactly $q$.
The expected loss of a decision $a$ can be expressed as
\begin{align} \label{eqn:Loss}
L(a)=\bE_{Z\sim F}[\ell(a,Z)]=\int_0^a (1-q)F(z) dz + \int_a^\infty q(1-F(z))dz
\end{align}
following standard derivations based on Riemann-Stieltjes integration by parts.

The objective is to find an ordering decision $a$ that minimizes the loss function $L(a)$.
It is well-known that an ordering decision $a$ is optimal if $F(a)=q$.
In general there can be multiple optimal solutions, or no decision $a$ for which $F(a)$ equals $q$ exactly.
Regardless, an optimal solution $a^*=F^{-1}(q)=\inf\{a:F(a)\ge q\}$ can always be defined based on the inverse CDF, which takes the smallest optimal solution if there are multiple.
We note that by right-continuity of the CDF function, we have $F(a^*)\ge q$, and $F(a)<q$ for all $a<a^*$.

In the data-driven Newsvendor problem, the distribution $F$ is unknown, and instead must be inferred from $n$ demand samples $Z_1,\ldots,Z_n$ that are drawn IID from $F$.
A general algorithm for data-driven Newsvendor is a (randomized) mapping from the demand samples drawn to a decision.
We primarily consider the Sample Average Approximation (SAA) algorithm, which constructs the empirical CDF $\hF(z)=\frac1n \sum_{i=1}^n \bI(Z_i\le z)$ over $z\ge0$ based on the samples, and then makes the decision $\ha=\hF^{-1}(q)=\inf\{a:\hF(a)\ge q\}$.
Similarly, we have $\hF(\ha)\ge q$, and $\hF(a)<q$ for all $a<\ha$.

We are interested in the regret $L(\ha)-L(a^*)$, which measures the loss of the SAA decision $\ha$ in excess of that of the optimal decision $a^*$.  From~\eqref{eqn:Loss}, we can see that
\begin{align}
L(\ha)-L(a^*)
&=\begin{cases}
\int_{\ha}^{a^*} (q(1-F(z))-(1-q)F(z)) dz, & \text{if }\ha\le a^* \\
\int_{a^*}^{\ha} ((1-q)F(z)-q(1-F(z))) dz, & \text{if }\ha>a^*\\
\end{cases}
\nonumber
\\ &=\int_{\ha}^{a^*} (q-F(z)) dz.
\label{eqn:whpDiff}
\end{align}
We note $L(\ha)-L(a^*)$ is a random variable, depending on the random demand samples drawn.
If we want to calculate its expectation, then from the linearity of expectation we can see that
\begin{align}
\bE[L(\ha)]-L(a^*)
= &\bE\left[\int_0^\infty \left((1-q)F(z)\bI(\ha > z) + q(1-F(z))\bI(\ha \le z)\right)dz\right] \nonumber
\\ &-\int_0^{a^*} (1-q)F(z) dz -\int_{a^*}^\infty q(1-F(z))dz \nonumber
\\ = &\int_0^\infty ((F(z) - qF(z))\Pr[\ha > z] + (q - qF(z))\Pr[\ha \le z])dz \nonumber
\\ &-\int_0^{a^*} (F(z)-q F(z)) dz -\int_{a^*}^\infty (q-q F(z))dz \nonumber
\\ = &\int_0^{a^*} (F(z)\Pr[\ha > z] + q\Pr[\ha \le z]-F(z)) dz + \int_{a^*}^\infty (F(z)\Pr[\ha > z] + q\Pr[\ha \le z]-q)dz \nonumber
\\ = &\int_0^{a^*} (q-F(z))\Pr[\ha \le z] dz + \int_{a^*}^\infty (F(z)-q)\Pr[\ha > z] dz \nonumber
\\ = &\int_0^{a^*}(q-F(z))\Pr[\hF(z)\ge q]dz+\int_{a^*}^\infty(F(z)-q)\Pr[\hF(z)<q]dz. \label{eqn:expDiff}
\end{align}
To explain the final equality that leads to expression~\eqref{eqn:expDiff}:
if $\hF(z)\ge q$, then $\ha=\inf\{a:\hF(a)\ge q\}\le z$ from definition; otherwise, if $\hF(z)<q$, then it is not possible for $\inf\{a:\hF(a)\ge q\}$ to be as small as $z$ because the function $\hF$ is monotonic and right-continuous.

Hereafter we work only with expressions~\eqref{eqn:Loss},~\eqref{eqn:whpDiff}, and~\eqref{eqn:expDiff}, omitting the random variable $Z$ and implicitly capturing the dependence on random variables $Z_1,\ldots,Z_n$ through the empirical CDF $\hF$.

\paragraph{Assumptions on distributions.}
We assume that $F$ is $(\beta,\gamma,\zeta)$-clustered, as defined in~\eqref{eqn:clustered} in the Introduction.
Because $F(a)\in[0,1]$, in order for there to exist any distributions satisfying~\eqref{eqn:clustered}, one must have $\zeta\le\frac1\gamma (\min\{q,1-q\})^{\frac1{\beta+1}}$.  Therefore we will assume this about the parameters of $(\beta,\gamma,\zeta)$-clustered distributions.
We note that any distribution can be captured under this definition, for sufficient choices of the parameters $\beta,\gamma,\zeta$.

We also assume the distribution $F$ has finite mean, which is necessary in order for the expected loss $L(a)$ in~\eqref{eqn:Loss} to be well-defined.
Some of the additive regret bounds will also necessarily scale with the finite mean of the distribution $F$, which we denote using $\mu(F)$.  We emphasize that the SAA algorithm itself does not require knowing the mean $\mu(F)$.  If one did know $\mu(F)$ or more generally an upper bound on the mean of the distribution, then one could analyze a \textit{projected SAA} algorithm instead, which is simpler---see \Cref{sec:projSAA}.

\section{High-probability Upper Bounds}

We first upper-bound the additive regret $L(\ha)-L(a^*)$ incurred by the SAA algorithm.
When $\beta<\infty$, the regret upper bound depends on the parameters $\beta,\gamma$ from $(\beta,\gamma,\zeta)$-clustered distributions, and the value of $n$ at which our bound starts holding also depends on $\zeta$.
When $\beta=\infty$, parameters $\gamma,\zeta$ are irrelevant but the regret upper bound depends on $q$, being worse when $q$ is close to 1. We note that when $\beta=\infty$, the upper bound depends on $\mu(F)$ explicitly, while when $\beta<\infty$, the dependence on the mean and how the distribution is scaled is captured through the constant $\gamma$ (see definition~\eqref{eqn:clustered}).

\begin{theorem} \label{thm:hpAdd}
Fix $q\in(0,1)$ and $\beta\in[0,\infty],\gamma\in(0,\infty),\zeta\in(0,(\min\{q,1-q\})^{\frac{1}{\beta+1}}/\gamma]$.

If $\beta<\infty$, then whenever the number of samples satisfies $n>\frac{\log(2/\delta)}{2(\gamma\zeta)^{2\beta+2}}$, we have
\begin{align*}
L(\ha)-L(a^*)
\le\frac1\gamma\left(\frac{\log(2/\delta)}{2n}\right)^{\frac{\beta+2}{2\beta+2}}
=O\left(\left(\frac{\log(1/\delta)}{n}\right)^{\frac{\beta+2}{2\beta+2}}\right)
\end{align*}
with probability at least $1-\delta$, for any $\delta\in(0,1)$ and any $(\beta,\gamma,\zeta)$-clustered distribution.

If $\beta=\infty$, then whenever the number of samples satisfies $n\ge\frac{2\log(2/\delta)}{(1-q)^2}$, we have
\begin{align*}
L(\ha)-L(a^*)
\le \frac{2\mu(F)}{1-q}\sqrt{\frac{\log(2/\delta)}{2n}}
=O\left(\left(\frac{\log(1/\delta)}{n}\right)^{\frac12}\right)
\end{align*}
with probability at least $1-\delta$, for any $\delta\in(0,1)$.
\end{theorem}

To justify our lower bound on $n$, we note that if $n$ is small, then $L(\ha)$ has large variance in terms of the randomness in $\ha$, and the separation between high-probability vs.\ expected regret is higher---we provide some empirical evidence of this at the end of \Cref{apx:simulation_supp}.  Because we are proving high-probability upper bounds that will match our upper bounds on expected regret (to come in \Cref{sec:expUB}), these empirics suggest that we must impose a lower bound on $n$.

\begin{proof}[Proof of \Cref{thm:hpAdd}]
By the DKW inequality \citep[see e.g.][]{massart1990tight}, we know that
\begin{align*}
\Pr\left[\sup_{a\ge0}|\hF(a)-F(a)|\le\sqrt{\frac{\log(2/\delta)}{2n}}\right]
\ge 1-2\exp\left(-2n\left(\sqrt{\frac{\log(2/\delta)}{2n}}\right)^2\right)=1-\delta.
\end{align*}
Therefore, with probability at least $1-\delta$, we have 
\begin{align} \label{eqn:DKWoutcome}
\sup_{a\ge0}|\hF(a)-F(a)|\le\sqrt{\frac{\log(2/\delta)}{2n}}. 
\end{align}
We will show that~\eqref{eqn:DKWoutcome} implies $L(\ha)-L(a^*)\le\frac1\gamma\left(\frac{\log(2/\delta)}{2n}\right)^{\frac{\beta+2}{2\beta+2}}$ when $\beta\in[0,\infty)$
and $n>\frac{\log(2/\delta)}{2(\gamma\zeta)^{2\beta+2}}$
(\textbf{Case 1}), and~\eqref{eqn:DKWoutcome} implies $L(\ha)-L(a^*)\le \frac{2}{1-q}\sqrt{\frac{\log(2/\delta)}{2n}}$ when $\beta=\infty$
and $n\ge\frac{2\log(2/\delta)}{(1-q)^2}$
(\textbf{Case 2}).

To begin with, we note that if $\ha\le a^*$, then
\begin{align}
q-F(\ha)
=\hF(\ha)-\hF(\ha)+q-F(\ha)
\le\sup_{a\ge0}|\hF(a)-F(a)|
\label{whpDiffBound1}
\end{align}
where the inequality holds because $\hF(\ha)\ge q$ (by right-continuity of $\hF$). Otherwise if $\ha> a^*$, then
\begin{align}
\lim_{a\to\ha^-} (F(a)-q)
=\lim_{a\to\ha^-} (F(a)-q+\hF(a)-\hF(a) )
\le\sup_{a\ge0}|\hF(a)-F(a)|
\label{whpDiffBound2}
\end{align}
where the inequality holds because $\hF(a)<q$ for all $a<\ha$.

\paragraph{Case 1: $\beta\in[0,\infty)$.}
From the definition of $(\beta,\gamma,\zeta)$-clustered distributions, we have
\begin{align*}
F(a^*-\zeta)\le q-(\gamma\zeta)^{\beta+1}<q-\sqrt{\frac{\log(2/\delta)}{2n}}
\\ F(a^*+\zeta)\ge q+(\gamma\zeta)^{\beta+1}>q+\sqrt{\frac{\log(2/\delta)}{2n}}
\end{align*}
where the strict inequalities hold because $n>\frac{\log(2/\delta)}{2(\gamma\zeta)^{2\beta+2}}$.
Applying~\eqref{eqn:DKWoutcome}, we deduce that $\hF(a^*-\zeta)<q$ and $\hF(a^*+\zeta)>q$.
From the definition of $\ha=\inf\{a:\hF(a)\ge q\}$, we conclude that $\ha\ge a^*-\zeta$ and $\ha\le a^*+\zeta$ respectively,
% \Willcomment{The first conclusion requires the monotonicity of $\hF$, and technically we can even deduce $\ha>a^*-\zeta$ from right continuity.  But I say we ignore these details.}
allowing us to apply the definition of $(\beta,\gamma,\zeta)$-clustered distributions on $\ha$.

When $\ha\le a^*$, we derive from~\eqref{eqn:whpDiff} that
\begin{align*}
L(\ha)-L(a^*)
&=\int_{\ha}^{a^*}(q-F(z))dz\\
&\le(a^*-\ha)(q-F(\ha))\\
&\le\frac1\gamma(q-F(\ha))^{\frac1{\beta+1}}(q-F(\ha))\\
&=\frac1\gamma(q-F(\ha))^{\frac{\beta+2}{\beta+1}}\\
&\le\frac1\gamma\left(\frac{\log(2/\delta)}{2n}\right)^{\frac{\beta+2}{2\beta+2}},
\end{align*}
where the second inequality applies the definition of clustered distributions, and the last inequality is by \eqref{whpDiffBound1} and \eqref{eqn:DKWoutcome}.

On the other hand, when $\ha>a^*$, we derive from~\eqref{eqn:whpDiff} that
\begin{align*}
L(\ha)-L(a^*)
&=\int_{\ha}^{a^*}(q-F(z))dz\\
&\le\lim_{a\to\ha^-}(a-a^*)(F(a)-q)\\
&\le\lim_{a\to\ha^-}\frac1\gamma(F(a)-q)^{\frac1{\beta+1}}(F(a)-q)\\
&=\frac1\gamma\lim_{a\to\ha^-}(F(a)-q)^{\frac{\beta+2}{\beta+1}}\\ &\le\frac1\gamma\left(\frac{\log(2/\delta)}{2n}\right)^{\frac{\beta+2}{2\beta+2}},
\end{align*}
where the first inequality follows from properties of the Riemann integral, the second inequality applies the definition of clustered distributions, and the last inequality is by \eqref{whpDiffBound2} and \eqref{eqn:DKWoutcome}.

Therefore, we conclude that
$L(\ha)-L(a^*)\le\frac1\gamma\left(\frac{\log(2/\delta)}{2n}\right)^{\frac{\beta+2}{2\beta+2}}$
holds universally for all possible values of $\ha$ and $a^*$ when $\beta\in[0,\infty)$
and $n>\frac{\log(2/\delta)}{2(\gamma\zeta)^{2\beta+2}}$.

\paragraph{Case 2: $\beta=\infty$.}
% \Willcomment{Can you change to depend on $\mu(F)$, instead of an upper bound $\mu$ on the mean?}
By definition, the mean of the distribution $F$ can be written as
\begin{align}
\mu(F)&=\int_0^\infty(1-F(z))dz. \label{ineq:mean}
\end{align}
When $\ha\le a^*$, we derive
\begin{align*}
    \int_0^\infty(1-F(z))dz
    &\ge\int_{\ha}^{a^*}(1-F(z))dz\\
    &\ge\lim_{a\to a^{*-}}(a-\ha)(1-F(a))\\
    &\ge(a^*-\ha)(1-q),
\end{align*}
where the second inequality follows from properties of the Riemann integral, and the last inequality holds because $F(a)<q$ for all $a<a^*$. This implies $a^*-\ha \le \frac{\mu(F)}{1-q}$.
% \Zhuoxincomment{This derivation is almost the same as \eqref{eqn:astarUB}. Shall we replace it with: when $\ha\le a^*$, we have $a^*-\ha\le a^*\le \frac{\mu}{1-q}$, where the second inequality follows from \eqref{eqn:astarUB}?} \Willcomment{Nah, let's keep it self-contained}

Substituting into~\eqref{eqn:whpDiff}, we have
\begin{align*}
    L(\ha)-L(a^*)
    &=\int_{\ha}^{a^*}(q-F(z))dz\\
    &\le(a^*-\ha)(q-F(\ha))\\
    &\le\frac{\mu(F)}{1-q}\sqrt{\frac{\log(2/\delta)}{2n}}\\
    &\le\frac{2\mu(F)}{1-q}\sqrt{\frac{\log(2/\delta)}{2n}},
\end{align*}
where the second inequality applies \eqref{whpDiffBound1} and \eqref{eqn:DKWoutcome}.

On the other hand, when $\ha>a^*$, we similarly derive
\begin{align*}
    \int_0^\infty(1-F(z))dz
    &\ge\int_{a^*}^{\ha}(1-F(z))dz\\
    &\ge(\ha-a^*)\lim_{a\to\ha^-}(1-F(a)),
\end{align*}
where the second inequality is by properties of the Riemann integral.
Applying \eqref{ineq:mean}, we obtain
$(\ha-a^*)\lim_{a\to\ha^-}(1-F(a))\le\mu(F)$.
Meanwhile, we have
\begin{align*}
    \lim_{a\to\ha^-}F(a)
    &=\lim_{a\to\ha^-}(F(a)-\hF(a)+\hF(a))\\
    &\le\sup_{a\ge0}|\hF(a)-F(a)|+\lim_{a\to\ha^-}\hF(a)\\
    &\le\sqrt{\frac{\log(2/\delta)}{2n}}+q\\
    &\le\frac{1-q}{2}+q\\
    &=\frac{1+q}{2},
\end{align*}
where the second inequality follows from \eqref{eqn:DKWoutcome} and the fact that $\hF(a)<q$ for all $a<\ha$, and the third inequality is by the assumption that $n\ge\frac{2\log(2/\delta)}{(1-q)^2}$.
Substituting back into $(\ha-a^*)\lim_{a\to\ha^-}(1-F(a))\le\mu(F)$, we derive $\ha-a^*\le\frac{\mu(F)}{1-\frac{1+q}{2}}=\frac{2\mu(F)}{1-q}$.
Substituting the final derivation into~\eqref{eqn:whpDiff}, we get
\begin{align*}
    L(\ha)-L(a^*)
    &=\int_{\ha}^{a^*}(q-F(z))dz\\
    &\le(\ha-a^*)\lim_{a\to\ha^-}(F(a)-q)\\
    &\le\frac{2\mu(F)}{1-q}\sqrt{\frac{\log(2/\delta)}{2n}},
\end{align*}
where the first inequality follows from the properties of the Riemann integral, and the second inequality uses \eqref{whpDiffBound2} and \eqref{eqn:DKWoutcome}.

Therefore, we conclude that
$L(\ha)-L(a^*)\le\frac{2\mu(F)}{1-q}\sqrt{\frac{\log(2/\delta)}{2n}}$
holds when $\beta=\infty$
and $n\ge\frac{2\log(2/\delta)}{(1-q)^2}$.
\end{proof}

We now upper-bound the multiplicative regret $\frac{L(\ha)-L(a^*)}{L(a^*)}$ incurred by the SAA algorithm.
When $\beta=\infty$, a convergence rate of $O(1/\sqrt{n})$ can be established on the multiplicative regret without making any assumptions on the denominator $L(a^*)$ being lower-bounded.
However, to get a faster convergence rate when $\beta<\infty$, we also need to make the assumption that 
% For multiplicative regret and $\beta<\infty$, we need the further assumption that
$F(a^*-\zeta),F(a^*+\zeta)$ are bounded away from $0,1$ respectively, to prevent the denominator $L(a^*)$ from being too small.  This is captured in the new parameter $\tau$.

In contrast to \Cref{thm:hpAdd}, the regret upper bound for $\beta<\infty$ now depends additionally on parameters $\zeta$ and $\tau$, and the regret upper bound for $\beta=\infty$ now worsens when $q$ is close to 0 or 1 (whereas before it only worsened when $q$ is close to 1).  This worsening when $q$ is close to 0 or 1 has been shown to be necessary for multiplicative regret \citep{cheung2019sampling}.

\begin{theorem} \label{thm:hpMult}
Fix $q\in(0,1)$ and $\beta\in[0,\infty],\gamma\in(0,\infty),\zeta\in(0,(\min\{q,1-q\})^{\frac{1}{\beta+1}}/\gamma),\tau\in(0,\min\{q,1-q\}-(\gamma\zeta)^{\beta+1}]$.

If $\beta<\infty$, then whenever the number of samples satisfies $n>\frac{\log(2/\delta)}{2(\gamma\zeta)^{2\beta+2}}$, we have
\begin{align*}
\frac{L(\ha)-L(a^*)}{L(a^*)}
\le\frac{1}{\gamma\zeta\tau}\left(\frac{\log(2/\delta)}{2n}\right)^{\frac{\beta+2}{2\beta+2}}
=O\left(\left(\frac{\log(1/\delta)}{n}\right)^{\frac{\beta+2}{2\beta+2}}\right)
\end{align*}
with probability at least $1-\delta$, for any $\delta\in(0,1)$ and any $(\beta,\gamma,\zeta)$-clustered distribution satisfying $F(a^*-\zeta)\ge\tau, F(a^*+\zeta)\le1-\tau$.
% \Zhuoxincomment{Perhaps I’m missing something, but why didn’t we emphasize here that the distribution has a finite mean for $\beta<\infty$?}

If $\beta=\infty$, then whenever the number of samples satisfies $n>\frac{\log(2/\delta)}{2(\min\{q,1-q\})^2}$, we have
\begin{align*}
    \frac{L(\ha)-L(a^*)}{L(a^*)}
    \le \frac{2}{\min\{q,1-q\}\sqrt{\frac{2n}{\log(2/\delta)}}-1}
    =O\left(\left(\frac{\log(1/\delta)}{n}\right)^{\frac12}\right)
\end{align*}
with probability at least $1-\delta$, for any $\delta\in(0,1)$ and any distribution.
% (with finite mean).
\end{theorem}

The $\beta=\infty$ case was studied in \citet[Thm.~2.2]{levi2007provably}, who establish that $n\ge\frac{9}{\eps^2}\frac{\log(2/\delta)}{2(\min\{q,1-q\})^2}$ samples is sufficient to guarantee a multiplicative regret at most $\eps$, for $\eps\le1$.
In order to make our error bound of $\frac{2}{\min\{q,1-q\}\sqrt{\frac{2n}{\log(2/\delta)}}-1}$ at most $\eps$, we need $n\ge\frac{(2+\eps)^2}{\eps^2}\frac{\log(2/\delta)}{2(\min\{q,1-q\})^2}$, which always satisfies our condition of $n>\frac{\log(2/\delta)}{2(\min\{q,1-q\})^2}$.
Therefore, the $\beta=\infty$ case of our \Cref{thm:hpMult} can be viewed as an improvement over \citet[Thm.~2.2]{levi2007provably}, that holds for all $\eps>0$, and moreover shows that a smaller constant is sufficient for $\eps\le 1$ (because $\frac{(2+\eps)^2}{\eps^2}\le\frac{9}{\eps^2}$).
We note however that a better dependence on $\min\{q,1-q\}$ was established in \citet{levi2015data} for $\eps\le1$.

\begin{proof}[Proof of \Cref{thm:hpMult}]
For $\beta\in[0,\infty)$, we derive from~\eqref{eqn:Loss} that
\begin{align}
    L(a^*)&=\int_0^{a^*}(1-q)F(z)dz+\int_{a^*}^\infty q(1-F(z))dz \nonumber\\
    &\ge\int_{a^*-\zeta}^{a^*}(1-q)F(z)dz+\int_{a^*}^{a^*+\zeta} q(1-F(z))dz \nonumber\\
    &\ge\int_{a^*-\zeta}^{a^*}(1-q)F(a^*-\zeta)dz+\int_{a^*}^{a^*+\zeta} q(1-F(a^*+\zeta))dz \nonumber\\
    &\ge\int_{a^*-\zeta}^{a^*}(1-q)\tau dz+\int_{a^*}^{a^*+\zeta} q\tau dz \nonumber\\
    &=\zeta\tau \label{LBL(a*)},
\end{align}
where the last inequality follows from the assumptions that $F(a^*-\zeta)\ge\tau$ and $F(a^*+\zeta)\le1-\tau$.

By \Cref{thm:hpAdd}, we know that with probability at least $1-\delta$,
\begin{align*}
    L(\ha)-L(a^*)\le\frac1\gamma\left(\frac{\log(2/\delta)}{2n}\right)^{\frac{\beta+2}{2\beta+2}},
\end{align*}
under the assumption that $n>\frac{\log(2/\delta)}{2(\gamma\zeta)^{2\beta+2}}$.
Thus, with probability at least $1 - \delta$, we have
\begin{align*}
\frac{L(\ha)-L(a^*)}{L(a^*)}\le\frac{1}{\gamma\zeta\tau}\left(\frac{\log(2/\delta)}{2n}\right)^{\frac{\beta+2}{2\beta+2}}
\end{align*}
for any $n>\frac{\log(2/\delta)}{2(\gamma\zeta)^{2\beta+2}}$.

The proof for $\beta=\infty$ is deferred to \Cref{sec:betaInftyHP}, due to similarities with \citet{levi2007provably}.
\end{proof}

\section{Expectation Upper Bounds} \label{sec:expUB}

We first upper-bound the expected additive regret $\bE[L(\ha)]-L(a^*)$ incurred by the SAA algorithm. In contrast to \Cref{thm:hpAdd}, here our regret upper bound for $\beta<\infty$ depends on all three parameters $\beta,\gamma,\zeta$ and holds for all values of $n$.  The regret upper bound for $\beta=\infty$ still only has an inverse dependence on $1-q$ but not $q$. Like our additive regret result in \Cref{thm:hpAdd}, some parts of these bounds will depend on the mean $\mu(F)$ of the demand distribution.

\begin{theorem}\label{thm:expAdd}
Fix $q\in(0,1)$ and $\beta\in[0,\infty],\gamma\in(0,\infty),\zeta\in(0,(\min\{q,1-q\})^{\frac{1}{\beta+1}}/\gamma]$.

If $\beta<\infty$, then we have
\begin{align*}
\bE[L(\ha)]-L(a^*) \le\frac{2}{\gamma}\left(\frac{1}{\beta+1}+\frac{1}{\sqrt{e}}\right)\left(\frac{1}{2\sqrt{n}}\right)^{\frac{\beta+2}{\beta+1}}+\frac{\mu(F)(q+1)}{n(\gamma\zeta)^{\beta+1}}=O\left(n^{-\frac{\beta+2}{2\beta+2}}\right)
\end{align*}
for any $(\beta,\gamma,\zeta)$-clustered distribution and any number of samples $n$.

If $\beta=\infty$, then we have
\begin{align*}
\bE[L(\ha)]-L(a^*)
    &\le\left(\frac{1}{\sqrt{e}}+2\right)\frac{\mu(F)}{(1-q)\sqrt{n}}
    =O\left(n^{-\frac12}\right)
\end{align*}
for any distribution and any number of samples $n$.
\end{theorem}

For the $\beta=0$ and $\beta=\infty$ cases, respective upper bounds of $(n+\frac{\mu(F) q}{1-q})\exp[-2n(\gamma\zeta)^2]+(\frac{\mu(F)(1+\mu(F))}{1-q}+\frac{1}{2\gamma})\frac{1}{n}$ \citep[Prop.~2]{lin2022data} and $\frac{(1+\mu(F))^2}{(1-q)\sqrt{n}}$ \citep[p.~2009]{lin2022data} were previously known\footnote{
\citet{lin2022data} did not normalize the unit costs of understocking and overstocking to sum to 1. The bounds we compare with here are obtained by substituting $\rho=q$ and $b+h=1$ into their bounds.
}.
% \Willdelete{,
% both also under the assumption that the distribution has bounded
% mean}.  
Our upper bound for $\beta=0$ requires a less restrictive condition (based on clustered distributions) than the positive density condition in \citet{lin2022data}. Importantly, our regret bounds only have a linear dependence on the mean $\mu(F)$ of the distribution, whereas the known upper bounds suffered from a quadratic scaling with the mean.

\begin{proof}[Proof of \Cref{thm:expAdd}]
We first consider the case where $\beta=\infty$, and then the case where $\beta\in[0,\infty)$.

\paragraph{Case 1: $\beta=\infty$.}
Let $a'=\inf\{a:F(a)\ge q+\frac{1-q}{2\sqrt{n}}\}$. We know $a'\ge a^*$ from the definition of $a^*=\inf\{a:F(a)\ge q\}$. Therefore, we derive from \eqref{eqn:expDiff} that
\begin{align}
&\quad\bE[L(\ha)]-L(a^*) \nonumber\\
&=\int_0^{a^*}(q-F(z))\Pr[\hF(z)\ge q]dz+\int_{a^*}^\infty(F(z)-q)\Pr[\hF(z)<q]dz \nonumber\\
&=\int_0^{a^*}(q-F(z))\Pr[\hF(z)\ge q]dz+
\int_{a^*}^{a'}(F(z)-q)\Pr[\hF(z)<q]dz+
\int_{a'}^\infty(F(z)-q)\Pr[\hF(z)<q]dz. \label{expGeneral}
\end{align}

We note that if $z<a^*$, then $q-F(z)>0$ by definition of $a^*$, and we have
\begin{align*}
    \Pr[\hF(z)\ge q]
    =\Pr[\hF(z)-F(z)\ge q-F(z)]
    \le\exp\left(-2n(q-F(z))^2\right),
\end{align*}
where the inequality follows from Hoeffding's inequality \cite[Thm.~2]{hoeffding1963probability}. Otherwise if $z\ge a^*$, then $F(z)-q\ge0$ by definition of $a^*$, and we have
\begin{align*}
    \Pr[\hF(z)< q]
    =\Pr[F(z)-\hF(z)> F(z)-q]
    \le\exp\left(-2n(F(z)-q)^2\right),
\end{align*}
where the inequality again applies Hoeffding's inequality. So the first two terms in \eqref{expGeneral} sum up to
\begin{align}
    &\quad\int_0^{a^*}(q-F(z))\Pr[\hF(z)\ge q]dz+\int_{a^*}^{a'}(F(z)-q)\Pr[\hF(z)<q]dz \nonumber\\
    &\le\int_0^{a'}|q-F(z)|\exp\left(-2n|q-F(z)|^2\right)dz
    \nonumber\\
    &\le\frac{a'}{2\sqrt{en}},\label{expGenTerm12}
\end{align}
where the last inequality holds because the function $g(x)=xe^{-2nx^2}$ is at most $\frac{1}{2\sqrt{en}}$ for all $x\ge0$.
Meanwhile, we derive
\begin{align*}
    \int_0^\infty (1-F(z))dz
    &\ge\int_0^{a'} (1-F(z))dz\\
    &\ge \lim_{a\to a'^{-}} a(1-F(a))\\
    &\ge a'\left(1-q-\frac{1-q}{2\sqrt{n}}\right)\\
    &\ge \frac{a'(1-q)}{2},
\end{align*}
where the second inequality follows from properties of the Riemann integral, the third inequality holds because $F(a)<q+\frac{1-q}{2\sqrt{n}}$ for all $a<a'$, and the last inequality holds for all positive integer $n$. 
Applying \eqref{ineq:mean}, we deduce that
$a'\le\frac{2\mu(F)}{1-q}$. Substituting this into \eqref{expGenTerm12}, we have
\begin{align}
    \int_0^{a^*}(q-F(z))\Pr[\hF(z)\ge q]dz+
    \int_{a^*}^{a'}(F(z)-q)\Pr[\hF(z)<q]dz
    &\le\frac{\mu(F)}{(1-q)\sqrt{en}}. \label{expGenTerm12Outcome}
\end{align}

For the third term in \eqref{expGeneral}, we have
\begin{align}
    \int_{a'}^\infty(F(z)-q)\Pr[\hF(z)<q]dz
    &=\int_{a'}^\infty(F(z)-q)\Pr\left[\frac1n\sum_{i=1}^n\bI(Z_i\le z)<q\right]dz \nonumber\\
    &=\int_{a'}^\infty(F(z)-q)\Pr\left[\frac{1}{n}\Bin(n,F(z))< q\right]dz \nonumber\\
    &=\int_{a'}^{\inf\{a:F(a)=1\}}(1-F(z))dz\frac{(F(z)-q)\Pr[\frac{1}{n}\Bin(n,F(z))< q]}{1-F(z)} \nonumber\\
    &\le\mu(F)\cdot\sup_{F\in[q+\frac{1-q}{2\sqrt{n}},1)}\frac{(F-q)\Pr[\frac1n\Bin(n,1-F)\ge 1-q]}{1-F}, \label{expGenTerm3Mid}
\end{align}
% \Zhuoxincomment{I moved all $\mu(F)$ outside the supremum and added $\cdot$.}
where $\Bin(n,F(z))$ is a binomial random variable with parameters $n$ and $F(z)$, the second equality follows from the independence of samples, the third equality follows because $\Pr[\frac{1}{n}\Bin(n,F(z))< q]=0$ if $F(z)=1$, and the inequality uses $\int_{a'}^\infty(1-F(z))dz\le\int_0^\infty(1-F(z))dz=\mu(F)$.

% \Zhuoxinreplace{by the assumption that the mean of the distribution is at most 1.}{by the assumption that the mean of the distribution is at most $\mu$.}

Consider a random variable $X$ defined as $\frac1n\Bin(n,1-F)$. The expected value and variance of $X$ are given by $\bE[X]=1-F$ and $\mathrm{Var}(X)=\frac{F(1-F)}{n}$ respectively. By Chebyshev's inequality \citep[see e.g.][p.~423]{shalev-shwartz2014understanding}, we obtain that for all $F\in[q+\frac{1-q}{2\sqrt{n}},1)$,
\begin{align}
    \Pr\left[\frac1n\Bin(n,1-F)\ge 1-q\right]
    &=\Pr[X\ge 1-q]\nonumber \\
    &=\Pr[X-(1-F)\ge F-q]\nonumber \\
    &\le\Pr[|X-(1-F)|\ge F-q]\nonumber \\
    &\le\frac{F(1-F)}{n(F-q)^2}.
    \label{Chebyshev_1-F}
\end{align}
Plugging it into \eqref{expGenTerm3Mid}, we have 
\begin{align}
    \int_{a'}^\infty(F(z)-q)\Pr[\hF(z)<q]dz
    \le\mu(F)\cdot\sup_{F\in[q+\frac{1-q}{2\sqrt{n}},1)}\frac{F}{n(F-q)}
    \le\frac{2\mu(F)}{(1-q)\sqrt{n}}. \label{expGenTerm3Outcome}
\end{align}

Combining \eqref{expGenTerm12Outcome} and \eqref{expGenTerm3Outcome}, we have
\begin{align*}
\bE[L(\ha)]-L(a^*)\le\left(\frac{1}{\sqrt{e}}+2\right)\frac{\mu(F)}{(1-q)\sqrt{n}}.   
\end{align*}

\paragraph{Case 2: $\beta\in[0,\infty)$.} 
% \Zhuoxincomment{I've changed the upper bound on mean to $\mu$ in this proof.} \Willcomment{Same thing---what do you think about just changing to be dependent on the mean $\mu(F)$ instead of an upper bound?}
We decompose $\bE[L(\ha)]-L(a^*)$ into three separate parts as follows.
By \eqref{eqn:expDiff},
\begin{align}
    \bE[L(\ha)]-L(a^*)
    &=\int_0^{a^*}(q-F(z))\Pr[\hF(z)\ge q]dz+\int_{a^*}^\infty(F(z)-q)\Pr[\hF(z)<q]dz \nonumber\\
    &=\int_0^{a^*-\zeta}(q-F(z))\Pr[\hF(z)\ge q]dz+\int_{a^*-\zeta}^{a^*}(q-F(z))\Pr[\hF(z)\ge q]dz\nonumber\\
    &\quad+\int_{a^*}^{a^*+\zeta}(F(z)-q)\Pr[\hF(z)<q]dz+\int_{a^*+\zeta}^\infty(F(z)-q)\Pr[\hF(z)<q]dz\nonumber\\
    &\le\int_0^{a^*-\zeta}(q-F(z))\Pr[\hF(z)\ge q]dz \label{expTerm1}\\
    &\quad+\int_{a^*-\zeta}^{a^*+\zeta}|q-F(z)|\exp[-2n|q-F(z)|^2]dz \label{expTerm2}\\
    &\quad+\int_{a^*+\zeta}^\infty(F(z)-q)\Pr[\hF(z)<q]dz, \label{expTerm3}
\end{align}
where the inequality is by Hoeffding's inequality.
We then analyze \eqref{expTerm1}, \eqref{expTerm2}, and \eqref{expTerm3} separately.

For \eqref{expTerm1}, similar with the analysis of the third term in \eqref{expGeneral} for the case where $\beta=\infty$, we derive
\begin{align}
    \int_0^{a^*-\zeta}(q-F(z))\Pr[\hF(z)\ge q]dz
    &=\int_0^{a^*-\zeta}(q-F(z))\Pr\left[\frac{1}{n}\Bin(n,F(z))\ge q\right]dz \nonumber\\
    &=\int_0^{a^*-\zeta}(1-F(z))dz\frac{(q-F(z))\Pr[\frac{1}{n}\Bin(n,F(z))\ge q]}{1-F(z)} \nonumber\\
    &\le\mu(F)\cdot\sup_{F\in[0,q-(\gamma\zeta)^{\beta+1}]}\frac{(q-F)\Pr[\frac1n\Bin(n,F)\ge q]}{1-F} \nonumber\\
    &\le\mu(F)\cdot\sup_{F\in[0,q-(\gamma\zeta)^{\beta+1}]}\frac{F}{n(q-F)} \nonumber\\
    &\le\frac{\mu(F) q}{n(\gamma\zeta)^{\beta+1}},\label{expTerm1outcome}
\end{align}
where the first inequality uses $\int_0^{a^*-\zeta}(1-F(z))dz\le\int_0^\infty(1-F(z))dz=\mu(F)$ 
and $F(a^*-\zeta)\le q-(\gamma\zeta)^{\beta+1}$ (by definition of clustered distributions). 
The second inequality is by Chebyshev's inequality.
Specifically, let $X=\frac1n\Bin(n,F)$ be a random variable. Then,
\begin{align}
\Pr\left[\frac1n\Bin(n,F)\ge q\right]
&=\Pr[X\ge q]\nonumber \\
&=\Pr[X-F\ge q-F]\nonumber \\
&\le\Pr[|X-F|\ge q-F]\nonumber \\
&\le\frac{F(1-F)}{n(q-F)^2}, \label{Chebyshev_F}
\end{align}
where the last inequality follows from Chebyshev’s inequality, using the fact that $\bE[X]=F$ and $\mathrm{Var}(X)=\frac{F(1-F)}{n}$.

Similarly, for \eqref{expTerm3} we have
\begin{align}
    \int_{a^*+\zeta}^\infty(F(z)-q)\Pr[\hF(z)<q]dz
    &=\int_{a^*+\zeta}^\infty(F(z)-q)\Pr\left[\frac{1}{n}\Bin(n,F(z))< q\right]dz \nonumber\\
    &=\int_{a^*+\zeta}^{\inf\{a:F(a)=1\}}(1-F(z))dz\frac{(F(z)-q)\Pr[\frac{1}{n}\Bin(n,F(z))< q]}{1-F(z)} \nonumber\\
    &\le\mu(F)\cdot\sup_{F\in[q+(\gamma\zeta)^{\beta+1},1)}\frac{(F-q)\Pr[\frac1n\Bin(n,1-F)\ge 1-q]}{1-F}  \nonumber\\
    &\le\mu(F)\cdot\sup_{F\in[q+(\gamma\zeta)^{\beta+1},1)}\frac{F}{n(F-q)}  \nonumber\\
    &\le\frac{\mu(F)}{n(\gamma\zeta)^{\beta+1}} \label{expTerm3outcome},
\end{align}
where the second equality follows because $\Pr[\frac{1}{n}\Bin(n,F(z))< q]=0$ if $F(z)=1$, the first inequality holds because $\int_{a^*+\zeta}^\infty(1-F(z))dz\le\int_0^\infty(1-F(z))dz=\mu(F)$ 
and $F(a^*+\zeta)\ge q+(\gamma\zeta)^{\beta+1}$ by definition of clustered distributions, and the second inequality follows from Chebyshev's inequality,
by the same derivation as in \eqref{Chebyshev_1-F}.

To analyze \eqref{expTerm2}, we need to consider two cases. When $\frac{1}{2\sqrt{n}}\ge(\gamma\zeta)^{\beta+1}$, we know that $\zeta\le\frac1\gamma\left(\frac{1}{2\sqrt{n}}\right)^{\frac{1}{\beta+1}}$. Because the function $g(x)=xe^{-2nx^2}$ is at most $\frac{1}{2\sqrt{en}}$ for all $x\ge0$, we obtain
\begin{align}
    \int_{a^*-\zeta}^{a^*+\zeta}|q-F(z)|\exp[-2n|q-F(z)|^2]dz
    &\le\int_{a^*-\zeta}^{a^*+\zeta}\frac{1}{2\sqrt{en}}dz\\
    &=\frac{\zeta}{\sqrt{en}}\\
    &\le\frac{2}{\gamma\sqrt{e}}\left(\frac{1}{2\sqrt{n}}\right)^{\frac{\beta+2}{\beta+1}}. \label{expCase1}
\end{align}
On the other hand, for the case where $\frac{1}{2\sqrt{n}}<(\gamma\zeta)^{\beta+1}$, we know that $\frac{1}{\gamma}(\frac{1}{2\sqrt{n}})^{\frac{1}{\beta+1}}<\zeta$. Therefore, we can decompose \eqref{expTerm2} into the following three terms:
\begin{align}
    &\quad\int_{a^*-\zeta}^{a^*+\zeta}|q-F(z)|\exp[-2n|q-F(z)|^2]dz \nonumber\\
    &=\int_{a^*-\zeta}^{a^*-\frac{1}{\gamma}(\frac{1}{2\sqrt{n}})^{\frac{1}{\beta+1}}}|q-F(z)|\exp[-2n|q-F(z)|^2]dz \label{expTerm4}\\
    &\quad+\int_{a^*-\frac{1}{\gamma}(\frac{1}{2\sqrt{n}})^{\frac{1}{\beta+1}}}^{a^*+\frac{1}{\gamma}(\frac{1}{2\sqrt{n}})^{\frac{1}{\beta+1}}}|q-F(z)|\exp[-2n|q-F(z)|^2]dz \label{expTerm5}\\
    &\quad+\int_{a^*+\frac{1}{\gamma}(\frac{1}{2\sqrt{n}})^{\frac{1}{\beta+1}}}^{a^*+\zeta}|q-F(z)|\exp[-2n|q-F(z)|^2]dz \label{expTerm6}.
\end{align}

When $z\in\left[a^*-\zeta,a^*-\frac{1}{\gamma}\left(\frac{1}{2\sqrt{n}}\right)^{\frac{1}{\beta+1}}\right]$, we have
\begin{align*}
    |q-F(z)|
    \ge\left(\gamma|z-a^*|\right)^{\beta+1}
    \ge\left(\gamma\left|a^*-\left(a^*-\frac{1}{\gamma}\left(\frac{1}{2\sqrt{n}}\right)^{\frac{1}{\beta+1}}\right)\right|\right)^{\beta+1}
    =\frac{1}{2\sqrt{n}},
\end{align*}
where the first inequality follows from definition of clustered distributions. Meanwhile, because the function $g(x)=xe^{-2nx^2}$ is monotonically decreasing on the interval $\left[\frac{1}{2\sqrt{n}},\infty\right)$, we obtain
\begin{align*}
    &\quad\int_{a^*-\zeta}^{a^*-\frac{1}{\gamma}(\frac{1}{2\sqrt{n}})^{\frac{1}{\beta+1}}}|q-F(z)|\exp[-2n|q-F(z)|^2]dz\\
    &\le\int_{a^*-\zeta}^{a^*-\frac{1}{\gamma}(\frac{1}{2\sqrt{n}})^{\frac{1}{\beta+1}}}\left(\gamma|z-a^*|\right)^{\beta+1}\exp[-2n\left(\gamma|z-a^*|\right)^{2(\beta+1)}]dz.
\end{align*}
Similarly, we derive that
\begin{align*}
    &\quad\int_{a^*+\frac{1}{\gamma}(\frac{1}{2\sqrt{n}})^{\frac{1}{\beta+1}}}^{a^*+\zeta}|q-F(z)|\exp[-2n|q-F(z)|^2]dz\\
    &\le\int_{a^*+\frac{1}{\gamma}(\frac{1}{2\sqrt{n}})^{\frac{1}{\beta+1}}}^{a^*+\zeta}\left(\gamma|z-a^*|\right)^{\beta+1}\exp[-2n\left(\gamma|z-a^*|\right)^{2(\beta+1)}]dz.
\end{align*}
Therefore, we can sum \eqref{expTerm4} and \eqref{expTerm6} to get
\begin{align*}
    &\quad\left(\int_{a^*-\zeta}^{a^*-\frac{1}{\gamma}(\frac{1}{2\sqrt{n}})^{\frac{1}{\beta+1}}}+\int_{a^*+\frac{1}{\gamma}(\frac{1}{2\sqrt{n}})^{\frac{1}{\beta+1}}}^{a^*+\zeta}\right)|q-F(z)|\exp[-2n|q-F(z)|^2]dz\\
    &\le\left(\int_{a^*-\zeta}^{a^*-\frac{1}{\gamma}(\frac{1}{2\sqrt{n}})^{\frac{1}{\beta+1}}}+\int_{a^*+\frac{1}{\gamma}(\frac{1}{2\sqrt{n}})^{\frac{1}{\beta+1}}}^{a^*+\zeta}\right)
    \left(\gamma|z-a^*|\right)^{\beta+1}\exp[-2n\left(\gamma|z-a^*|\right)^{2(\beta+1)}]dz.
\end{align*}
To simplify the integral, we let $x$ denote $|z-a^*|$, which yields
\begin{align}
    &\quad\left(\int_{a^*-\zeta}^{a^*-\frac{1}{\gamma}(\frac{1}{2\sqrt{n}})^{\frac{1}{\beta+1}}}+\int_{a^*+\frac{1}{\gamma}(\frac{1}{2\sqrt{n}})^{\frac{1}{\beta+1}}}^{a^*+\zeta}\right)|q-F(z)|\exp[-2n|q-F(z)|^2]dz \nonumber\\
    &\le2\int_{\frac{1}{\gamma}(\frac{1}{2\sqrt{n}})^{\frac{1}{\beta+1}}}^{\zeta}(\gamma x)^{\beta+1}\exp[-2n(\gamma x)^{2\beta+2}]dx \nonumber\\
    &=2\int_{\frac{1}{\gamma}(\frac{1}{2\sqrt{n}})^{\frac{1}{\beta+1}}}^{\zeta}
    \frac{(\gamma x)^{2\beta+1}}{(\gamma x)^\beta}\exp[-2n(\gamma x)^{2\beta+2}]dx \nonumber\\
    &\le\frac{2}{\left(\frac{1}{2\sqrt{n}}\right)^{\frac{\beta}{\beta+1}}}\cdot\frac{\exp[-2n(\gamma x)^{2\beta+2}]}{2n\gamma(2\beta+2)}\bigg|_{\zeta}^{\frac{1}{\gamma}(\frac{1}{2\sqrt{n}})^{\frac{1}{\beta+1}}} \nonumber\\
    &\le\frac{2}{\gamma(\beta+1)}\left(\frac{1}{2\sqrt{n}}\right)^{\frac{\beta+2}{\beta+1}}, \label{expTerm46}
\end{align}
where the last inequality holds because $\exp[-2n(\gamma x)^{2\beta+2}]|_{\zeta}^{\frac{1}{\gamma}(\frac{1}{2\sqrt{n}})^{\frac{1}{\beta+1}}}\le\exp[-2n(\gamma x)^{2\beta+2}]|_\zeta^0\le1$.

For \eqref{expTerm5}, because we have $g(x)=xe^{-2nx^2}\le\frac{1}{2\sqrt{en}}$ for all $x\ge0$, it follows that
\begin{align}
    \int_{a^*-\frac{1}{\gamma}(\frac{1}{2\sqrt{n}})^{\frac{1}{\beta+1}}}^{a^*+\frac{1}{\gamma}(\frac{1}{2\sqrt{n}})^{\frac{1}{\beta+1}}}|q-F(z)|\exp[-2n|q-F(z)|^2]dz
    &\le\int_{a^*-\frac{1}{\gamma}(\frac{1}{2\sqrt{n}})^{\frac{1}{\beta+1}}}^{a^*+\frac{1}{\gamma}(\frac{1}{2\sqrt{n}})^{\frac{1}{\beta+1}}}\frac{1}{2\sqrt{en}}dz \nonumber\\
    &=\frac{2}{\gamma\sqrt{e}}\left(\frac{1}{2\sqrt{n}}\right)^{\frac{\beta+2}{\beta+1}}. \label{expTerm5outcome}
\end{align}

Combining the results \eqref{expTerm46} and \eqref{expTerm5outcome}, we know that under the case where $\frac{1}{2\sqrt{n}}<(\gamma\zeta)^{\beta+1}$,
\begin{align*}
    \int_{a^*-\zeta}^{a^*+\zeta}|q-F(z)|\exp[-2n|q-F(z)|^2]dz
    \le\frac{2}{\gamma}\left(\frac{1}{\beta+1}+\frac{1}{\sqrt{e}}\right)\left(\frac{1}{2\sqrt{n}}\right)^{\frac{\beta+2}{\beta+1}}.
\end{align*}
Note that this result is strictly greater than the result $\frac{2}{\gamma\sqrt{e}}\left(\frac{1}{2\sqrt{n}}\right)^{\frac{\beta+2}{\beta+1}}$ derived in \eqref{expCase1} for the case where $\frac{1}{2\sqrt{n}}\ge(\gamma\zeta)^{\beta+1}$, so for any number of samples $n$, we have
\begin{align}
    \int_{a^*-\zeta}^{a^*+\zeta}|q-F(z)|\exp[-2n|q-F(z)|^2]dz
    \le\frac{2}{\gamma}\left(\frac{1}{\beta+1}+\frac{1}{\sqrt{e}}\right)\left(\frac{1}{2\sqrt{n}}\right)^{\frac{\beta+2}{\beta+1}}. \label{expTerm2outcome}
\end{align}

Finally, by combining the results \eqref{expTerm1outcome}, \eqref{expTerm3outcome}, and \eqref{expTerm2outcome}, we conclude that
\begin{align*}
    \bE[L(\ha)]-L(a^*) \le\frac{2}{\gamma}\left(\frac{1}{\beta+1}+\frac{1}{\sqrt{e}}\right)\left(\frac{1}{2\sqrt{n}}\right)^{\frac{\beta+2}{\beta+1}}+\frac{\mu(F)(q+1)}{n(\gamma\zeta)^{\beta+1}}
\end{align*}
when $\beta\in[0,\infty)$.
\end{proof}

We now upper-bound the expected multiplicative regret incurred by the SAA algorithm.  For multiplicative regret and $\beta<\infty$, we need the further assumption that $F(a^*-\zeta),F(a^*+\zeta)$ are bounded away from $0,1$ respectively, to prevent the denominator $L(a^*)$ from becoming too small.

\begin{theorem}\label{thm:expMult}
Fix $q\in(0,1)$ and $\beta\in[0,\infty], \gamma\in(0,\infty), \zeta\in(0,(\min\{q,1-q\})^{\frac{1}{\beta+1}}/\gamma), \tau\in(0,\min\{q,1-q\}-(\gamma\zeta)^{\beta+1}]$.

If $\beta<\infty$, then we have
\begin{align*}
\frac{\bE[L(\ha)]-L(a^*)}{L(a^*)}
\le\max\left\{\frac{1}{n(\gamma\zeta)^{\beta+1}\min\left\{q,1-q\right\}},\frac{2}{\gamma\zeta\tau}\left(\frac{1}{\beta+1}+\frac{1}{\sqrt{e}}\right)\left(\frac{1}{2\sqrt{n}}\right)^{\frac{\beta+2}{\beta+1}}\right\}=O\left(n^{-\frac{\beta+2}{2\beta+2}}\right)
\end{align*}
for any $(\beta,\gamma,\zeta)$-clustered distribution satisfying $F(a^*-\zeta)\ge\tau, F(a^*+\zeta)\le1-\tau$ and any number of samples $n$.

If $\beta=\infty$, then we have
\begin{align}\label{eqn:expMultGen}
\sup_{F:\mu(F)<\infty}\frac{\bE[L(\ha)]-L(a^*)}{L(a^*)}
&=\max\left\{\sup_{F\in(0,q)}\frac{q-F}{(1-q)F}\Pr\left[\frac1n\Bin(n,F)\ge q\right],\sup_{F\in[q,1)}\frac{F-q}{q(1-F)}\Pr\left[\frac1n\Bin(n,F) < q\right]\right\}
\end{align}
for any number of samples $n$, where $\Bin(n,F)$ is a binomial random variable with parameters $n$ and $F$.
\end{theorem}
% \Zhuoxincomment{Before, $\mu$ was used as the upper bound on the mean, but here it denotes the mean itself. Would this cause confusion for readers?}  \Willcomment{If you agree with my changes before this, then this should be resolved?}

The $\beta=\infty$ case was studied in \citet[Thm.~2]{besbes2023big}, who characterized the exact value of $\sup_{F:\mu(F)<\infty}\frac{\bE[L(\ha)]-L(a^*)}{L(a^*)}$ (instead of merely providing an upper bound), showing it to equal the expression in~\eqref{eqn:expMultGen}.
This expression is then shown to be $O(n^{-\frac12})$.
We derive the same expression using a shorter proof that bypasses their machinery, although their machinery has other benefits such as deriving the minimax-optimal policy (which is not SAA).
We note that an exact analysis of the worst-case expected additive regret $\sup_{F:\mu(F)<\infty}(\bE[L(\ha)]-L(a^*))$ is also possible, even in a contextual setting \citep{besbes2023contextual}, but our simplification does not appear to work there.

\begin{proof}[Proof of \Cref{thm:expMult}]
For $\beta\in[0,\infty)$, we begin by using the same decomposition of $\bE[L(\ha)]-L(a^*)$ as in the proof of \Cref{thm:expAdd}. By \eqref{eqn:expDiff},
\begin{align*}
    &\quad\bE[L(\ha)]-L(a^*)\\
    &\le\int_0^{a^*-\zeta}(q-F(z))\Pr[\hF(z)\ge q]dz 
    +\int_{a^*-\zeta}^{a^*+\zeta}|q-F(z)|\exp[-2n|q-F(z)|^2]dz
    +\int_{a^*+\zeta}^\infty(F(z)-q)\Pr[\hF(z)<q]dz\\
    &\le\int_0^{a^*-\zeta}(q-F(z))\Pr[\hF(z)\ge q]dz+\frac{2}{\gamma}\left(\frac{1}{\beta+1}+\frac{1}{\sqrt{e}}\right)\left(\frac{1}{2\sqrt{n}}\right)^{\frac{\beta+2}{\beta+1}}+\int_{a^*+\zeta}^\infty(F(z)-q)\Pr[\hF(z)<q]dz,
\end{align*}
where the last inequality follows from \eqref{expTerm2outcome}.

We similarly decompose $L(a^*)$ into three terms as follows. By \eqref{eqn:Loss},
\begin{align*}
L(a^*) 
&=\int_0^{a^*} (1-q)F(z) dz+\int_{a^*}^\infty q(1-F(z))dz\\
&=\int_0^{a^*-\zeta} (1-q)F(z) dz+\left(\int_{a^*-\zeta}^{a^*} (1-q)F(z) dz+\int_{a^*}^{a^*+\zeta} q(1-F(z))dz\right)+\int_{a^*+\zeta}^\infty q(1-F(z))dz\\
&\ge\int_0^{a^*-\zeta} (1-q)F(z) dz+\tau\zeta+\int_{a^*+\zeta}^\infty q(1-F(z))dz,
\end{align*}
where the last inequality applies \eqref{LBL(a*)} given the assumption that $F(a^*-\zeta)\ge\tau$, $F(a^*+\zeta)\le1-\tau$.

Therefore, we have
\begin{align}
    &\quad\frac{\bE[L(\ha)]-L(a^*)}{L(a^*)} \nonumber\\
    &\le\frac{\int_0^{a^*-\zeta}(q-F(z))\Pr[\hF(z)\ge q]dz+\frac{2}{\gamma}\left(\frac{1}{\beta+1}+\frac{1}{\sqrt{e}}\right)\left(\frac{1}{2\sqrt{n}}\right)^{\frac{\beta+2}{\beta+1}}+\int_{a^*+\zeta}^\infty(F(z)-q)\Pr[\hF(z)<q]dz}{\int_0^{a^*-\zeta} (1-q)F(z) dz+\tau\zeta+\int_{a^*+\zeta}^\infty q(1-F(z))dz} \nonumber\\
    &\le\max\left\{\frac{\int_0^{a^*-\zeta}(q-F(z))\Pr[\hF(z)\ge q]dz}{\int_0^{a^*-\zeta} (1-q)F(z) dz},\frac{2}{\gamma\zeta\tau}\left(\frac{1}{\beta+1}+\frac{1}{\sqrt{e}}\right)\left(\frac{1}{2\sqrt{n}}\right)^{\frac{\beta+2}{\beta+1}},\frac{\int_{a^*+\zeta}^\infty(F(z)-q)\Pr[\hF(z)<q]dz}{\int_{a^*+\zeta}^\infty q(1-F(z))dz}\right\} \nonumber\\
    &\le\max\bigg\{
    \sup_{F\in(0,q-(\gamma\zeta)^{\beta+1}]}\frac{(q-F)\Pr\left[\frac1n\Bin(n,F)\ge q\right]}{(1-q)F}, \nonumber\\
    &\quad\quad\quad\quad~\frac{2}{\gamma\zeta\tau}\left(\frac{1}{\beta+1}+\frac{1}{\sqrt{e}}\right)\left(\frac{1}{2\sqrt{n}}\right)^{\frac{\beta+2}{\beta+1}}, \nonumber\\
    &\quad\quad\quad\quad~\sup_{F\in[q+(\gamma\zeta)^{\beta+1},1)}\frac{(F-q)\Pr\left[\frac1n\Bin(n,F) < q\right]}{q(1-F)}
    \bigg\}, \label{expMultDecompose}
\end{align}
where the last inequality uses $F(a^*-\zeta)\le q-(\gamma\zeta)^{\beta+1}$ and $F(a^*+\zeta)\ge q+(\gamma\zeta)^{\beta+1}$ from the definition of clustered distributions.

Next we analyze the maximum of the first and third terms in \eqref{expMultDecompose}. We derive
\begin{align*}
    &\quad\max\bigg\{
    \sup_{F\in(0,q-(\gamma\zeta)^{\beta+1}]}\frac{(q-F)\Pr\left[\frac1n\Bin(n,F)\ge q\right]}{(1-q)F},
    \sup_{F\in[q+(\gamma\zeta)^{\beta+1},1)}\frac{(F-q)\Pr\left[\frac1n\Bin(n,F) < q\right]}{q(1-F)}
    \bigg\}\\
    &\le\max\left\{\sup_{F\in(0,q-(\gamma\zeta)^{\beta+1}]}\frac{1-F}{n(1-q)(q-F)},\sup_{F\in[q+(\gamma\zeta)^{\beta+1},1)}\frac{F}{nq(F-q)}\right\}\\
    &\le\max\left\{\frac{1}{n(1-q)(\gamma\zeta)^{\beta+1}},\frac{1}{nq(\gamma\zeta)^{\beta+1}}\right\}\\
    &=\frac{1}{n(\gamma\zeta)^{\beta+1}\min\left\{q,1-q\right\}},
\end{align*}
where the first inequality follows from two applications of Chebyshev’s inequality: the first one applies to the first term, as in \eqref{Chebyshev_F}, and the second one applies to the second term after a transformation $\Pr\left[\frac1n\Bin(n,F) < q\right]=\Pr[\frac1n\Bin(n,1-F)\ge 1-q]$, as in \eqref{Chebyshev_1-F}.

Substituting this into \eqref{expMultDecompose}, we have
\begin{align*}
    \frac{\bE[L(\ha)]-L(a^*)}{L(a^*)}
    \le\max\left\{\frac{1}{n(\gamma\zeta)^{\beta+1}\min\left\{q,1-q\right\}},\frac{2}{\gamma\zeta\tau}\left(\frac{1}{\beta+1}+\frac{1}{\sqrt{e}}\right)\left(\frac{1}{2\sqrt{n}}\right)^{\frac{\beta+2}{\beta+1}}\right\}.
\end{align*}

The proof for $\beta=\infty$ is deferred to \Cref{sec:betaInftyExp}, because it is simplifying an existing result from \citet{besbes2023big}.
\end{proof}

\section{Additive Lower Bound} \label{sec:lowerBds}

We now lower-bound the additive regret of any (possibly randomized) data-driven algorithm for Newsvendor, showing it to be $\Omega(n^{-\frac{\beta+2}{2\beta+2}})$ with probability at least 1/3.  This implies that the expected additive regret is also $\Omega(n^{-\frac{\beta+2}{2\beta+2}})$.  The lower bound for multiplicative regret is similar, with the main challenge being to modify the distributions to satisfy $F(a^*-\zeta)\ge\tau,F(a^*+\zeta)\le 1-\tau$, so we defer it to \Cref{sec:additionalLB}.

% \Zhuoxincomment{Should we modify this theorem so that the lower bound involves $\mu$? 
% It seems we could change $H$ to $\frac{\mu}{\max\{\gamma,1\}}\left(\frac{C}{\sqrt{n}}\right)^{\frac{1}{\beta+1}}$, which would give a lower bound of 
% $\frac{\mu}{8\max\{\gamma,1\}}\left(\frac{q(1-q)}{3\sqrt{n}}\right)^{\frac{\beta+2}{\beta+1}}$. 
% But this looks strange because (1) our high-probability additive regret for $\beta<\infty$ does not depend on $\mu$, and 
% (2) the $n^{-\frac{\beta+2}{2\beta+2}}$ term in our expected additive regret for $\beta<\infty$ also does not involve $\mu$. 
% The lower bounds in \cite[Thm.~1]{lin2022data} and \cite[Thm.~2]{lyu2024closing} are also independent of $\mu$.}
% \Willcomment{Yeah, I don't think any modification is needed.  For lower bounds, generally you want the distribution to be as "simple" as possible, and here our distribution is bounded which is the most you can ask for.}

\begin{theorem} \label{thm:lowAdd}
Fix $q\in(0,1)$ and $\beta\in[0,\infty],\gamma\in(0,\infty),\zeta\in(0,(\min\{q,1-q\})^{\frac{1}{\beta+1}}/\gamma]$.
Any learning algorithm based on $n$ samples makes a decision with additive regret at least
\begin{align*}
\frac{1}{8\max\{\gamma,1\}}\left(\frac{q(1-q)}{3\sqrt{n}}\right)^{\frac{\beta+2}{\beta+1}}=\Omega\left(n^{-\frac{\beta+2}{2\beta+2}}\right)
\end{align*}
with probability at least 1/3 on some $(\beta,\gamma,\zeta)$-clustered distribution that takes values in [0,1].
Therefore, the expected additive regret is at least
\begin{align*}
    \frac{1}{24\max\{\gamma,1\}}\left(\frac{q(1-q)}{3\sqrt{n}}\right)^{\frac{\beta+2}{\beta+1}}=\Omega\left(n^{-\frac{\beta+2}{2\beta+2}}\right).
\end{align*}
\end{theorem}

\begin{proof}[Proof of \Cref{thm:lowAdd}]
% Fix $q\in(0,1)$ and $\beta\in[0,\infty),\gamma\in(0,\infty),\zeta\in(0,(\min\{q,1-q\})^{\frac{1}{\beta+1}}/\gamma]$.
Let $C=\frac{q(1-q)}{3}$, $H=\frac{1}{\max\{\gamma,1\}}\left(\frac{C}{\sqrt{n}}\right)^{\frac{1}{\beta+1}}$. Consider two distributions $P$ and $Q$, whose respective CDF functions $F_P$ and $F_Q$ are:
\begin{align*}
    F_P(z)&=
    \begin{cases}
        0, &z\in(-\infty, 0)\\
        q+z\frac{C}{H\sqrt{n}}, &z\in[0,H)\\
        1,&z\in[H,\infty);
    \end{cases}\\
    F_Q(z)&=
    \begin{cases}
         0, &z\in(-\infty, 0)\\
        q+z\frac{C}{H\sqrt{n}}-\frac{C}{\sqrt{n}}, &z\in[0,H)\\
        1,&z\in[H,\infty).
    \end{cases}
\end{align*}
We let $L_P(a)$ and $L_Q(a)$ denote the respective expected loss functions under true distributions $P$ and $Q$, and from the CDF functions, it can be observed that the respective optimal decisions are $a_P^*=0$ and $a_Q^*=H$. We now show that any learning algorithm based on $n$ samples will incur an additive regret at least $\frac{1}{8\max\{\gamma,1\}}\left(\frac{q(1-q)}{3\sqrt{n}}\right)^{\frac{\beta+2}{\beta+1}}$ with probability at least 1/3, on distribution $P$ or $Q$.

\paragraph{Establishing validity of distributions.}
First we show that both $P$ and $Q$ are $(\beta,\gamma,\zeta)$-clustered distributions. 
Note that the constraint $\zeta\in(0,(\min\{q,1-q\})^{\frac{1}{\beta+1}}/\gamma]$ ensures that any $z\in[a^*-\zeta,a^*+\zeta]$ with $F(z)=0$ or $F(z)=1$ would satisfy \eqref{eqn:clustered}; therefore it suffices to verify \eqref{eqn:clustered} on $z\in[0,H)$ for both $P$ and $Q$. For distribution $P$, which has $a^*=0$, we have
\begin{align*}
    |F_P(z)-q|
    =z\frac{C}{H\sqrt{n}}
    =\frac{z}{H}(\max\{\gamma,1\} H)^{\beta+1}
    =z\max\{\gamma,1\}^{\beta+1}H^{\beta}
    >(\gamma z)^{\beta+1}
    =(\gamma|z-0|)^{\beta+1}
\end{align*}
for all $z\in[0,H)$, where the second equality follows from $\frac{C}{\sqrt{n}}=(\max\{\gamma,1\} H)^{\beta+1}$ and the inequality applies $H>z$. Therefore $P$ is a $(\beta,\gamma,\zeta)$-clustered distribution. 
It can be verified by symmetry that $Q$ is also a $(\beta,\gamma,\zeta)$-clustered distribution.

In addition, it can be elementarily observed that
\begin{align*}
\lim_{z\to H^-} F_P(z)=q+\frac{q(1-q)}{3\sqrt{n}} &\le q+\frac{1-q}3=1-\frac23 (1-q) < 1
\\ F_Q(0)=q-\frac{q(1-q)}{3\sqrt{n}} &\ge q-\frac q3=\frac23 q > 0
\end{align*}
which ensures the monotonicity of the CDF's for $P$ and $Q$.

Finally, it is easy to see that $H\le1$, and hence both distributions $P$ and $Q$ take values in $[0,1]$.

\paragraph{Upper-bounding the probabilistic distance between $P$ and $Q$.}
We analyze the squared Hellinger distance between distributions $P$ and $Q$, denoted as $\mathrm{H}^2(P,Q)$.  Because $P$ and $Q$ only differ in terms of their point masses on $0$ and $H$, standard formulas
% \footnote{\Willedit{Some sources define the squared Hellinger distance without the "1/2" in front.}}
for Hellinger distance yield
\begin{align}
\mathrm{H}^2(P,Q)
% &=\frac12\sum_{z\in\{0,H\}}\left(\sqrt{p(z)}-\sqrt{q(z)}\right)^2 \label{eqn:hellDef}\\
&=\frac12\left(\left(\sqrt{q}-\sqrt{q-\frac{C}{\sqrt{n}}}\right)^2+\left(\sqrt{1-q-\frac{C}{\sqrt{n}}}-\sqrt{1-q}\right)^2\right) \label{eqn:hellDef}\\
&=\frac12\left(q+q-\frac{C}{\sqrt{n}}-2\sqrt{q\left(q-\frac{C}{\sqrt{n}}\right)}+1-q+1-q-\frac{C}{\sqrt{n}}-2\sqrt{(1-q)\left(1-q-\frac{C}{\sqrt{n}}\right)}\right) \label{additiveHellinger}\\
&=\frac12\left(-\frac{2C}{\sqrt{n}}+2q-2q\sqrt{1-\frac{C}{q\sqrt{n}}}+2(1-q)-2(1-q)\sqrt{1-\frac{C}{(1-q)\sqrt{n}}}\right) \nonumber\\
&\le\frac12\left(-\frac{2C}{\sqrt{n}}+2q\left(\frac{C}{2q\sqrt{n}}+\frac{C^2}{2q^2n}\right)+2(1-q)\left(\frac{C}{2(1-q)\sqrt{n}}+\frac{C^2}{2(1-q)^2n}\right)\right) \nonumber\\
&=\frac12\left(\frac{C^2}{qn}+\frac{C^2}{(1-q)n}\right) \nonumber\\
&=\frac{C^2}{2nq(1-q)}, \label{hellinger}
\end{align}
where the inequality follows from applying $1-\sqrt{1-x}\le\frac x2+\frac{x^2}{2}$, $\forall x\in[0,1]$.
We note that we are substituting in $x=\frac{C}{q\sqrt{n}}$ and $x=\frac{C}{(1-q)\sqrt{n}}$, which are at most $1$ because $C=q(1-q)/3$.

Let $P^n$ denote the distribution for the $n$ samples observed by the algorithm under distribution $P$, and let $Q^n$ denote the corresponding distribution under $Q$. Let $\mathrm{TV}(P^n,Q^n)$ denote the total variation distance between $P^n$ and $Q^n$. By a relationship\footnote{Some sources such as \citet[Thm~2.2]{tsybakov2009introduction} use a tighter upper bound of $\sqrt{2\mathrm{H}^2(P^n,Q^n)(1-\mathrm{H}^2(P^n,Q^n)/2)}$ on $\mathrm{TV}(P^n,Q^n)$ (noting that their definition of $\mathrm{H}^2(P,Q)$ also differs, by not having the coefficient "1/2" in~\eqref{eqn:hellDef}).  The weaker upper bound of $\sqrt{2\mathrm{H}^2(P^n,Q^n)}$ as used in \citet{guo2021generalizing} will suffice for our purposes.} between the total variation distance and Hellinger distance,
% \Willcomment{I realized that these are tighter bounds that what is used in \citet[Thm.~2.2]{tsybakov2009introduction}.  Hopefully this doesn't confuse anyone..}\Zhuoxincomment{Actually I think the error probability derived in \cite{tsybakov2009introduction} is greater than ours? Following the definition of squared Hellinger distance in \cite{tsybakov2009introduction}
% , which is twice as much as ours, their error probability is given by $\frac{1-\mathrm{H}\sqrt{1-\frac{\mathrm{H}^2}{4}}}{2}$. This is greater than our $\frac{1-\mathrm{H}}{2}$ because $\sqrt{1-\frac{\mathrm{H}^2}{4}}\le1$.},
we have
    $\mathrm{TV}(P^n,Q^n)\le\sqrt{2\mathrm{H}^2(P^n,Q^n)}$,
which is at most $\sqrt{2n\mathrm{H}^2(P,Q)}$ according to the additivity of the Hellinger distance.
% \cite[see e.g.][Lem.~3]{guo2021generalizing}.
By applying \eqref{hellinger}, we obtain
\begin{align*}
    \mathrm{TV}(P^n,Q^n)
    \le\frac{C}{\sqrt{q(1-q)}}
    =\frac{\sqrt{q(1-q)}}{3}
    &\le\frac13.
\end{align*}

\paragraph{Lower-bounding the expected regret of any algorithm.}
Fix any (randomized) algorithm for data-driven Newsvendor, and consider the sample paths of its execution on the distributions $P$ and $Q$ side-by-side.  The sample paths can be coupled so that the algorithm makes the same decision for $P$ and $Q$ on an event $E$ of measure $1-\mathrm{TV}(P^n,Q^n)\ge 2/3$, by definition of total variation distance.
Letting $A_P,A_Q$ be the random variables for the decisions of the algorithm on distributions $P,Q$ respectively,
we have that $A_P$ and $A_Q$ are identically distributed conditional on $E$.
Therefore, either $\Pr[A_P\ge\frac H2|E]=\Pr[A_Q\ge\frac H2|E]\ge1/2$ or $\Pr[A_P\le\frac H2|E]=\Pr[A_Q\le\frac H2|E]\ge 1/2$.

First consider the case where $\Pr[A_P\ge\frac H2|E]=\Pr[A_Q\ge\frac H2|E]\ge1/2$.
Note that if $A_P\ge\frac H2$, then we can derive from~\eqref{eqn:whpDiff} that under the true distribution $P$,
\begin{align*}
L_P(A_P) - L_P(a^*_P)
&=\int_0^{A_P}(F_P(z)-q)dz\\
&\ge\int_0^{\frac H2}(F_P(z)-q)dz\\
&=\int_0^{\frac H2} z\frac{C}{H\sqrt{n}} dz\\
&=\frac{CH}{8\sqrt{n}}\\
&=\frac{1}{8\max\{\gamma,1\}}\left(\frac{q(1-q)}{3\sqrt{n}}\right)^{\frac{\beta+2}{\beta+1}}.
\end{align*}
Therefore, we would have
\begin{align*}
\Pr\left[L_P(A_P)-L_P(a^*_P)\ge\frac{1}{8\max\{\gamma,1\}}\left(\frac{q(1-q)}{3\sqrt{n}}\right)^{\frac{\beta+2}{\beta+1}}\right]
&\ge\Pr\left[A_P\ge\frac H2\right]\\
&\ge\Pr\left[A_P\ge\frac H2\bigg|E\right]\Pr[E]\\
&\ge\frac12 \cdot\frac 23=\frac13.
\end{align*}

Now consider the other case where $\Pr[A_P\le\frac H2|E]=\Pr[A_Q\le\frac H2|E]\ge 1/2$.
If $A_Q\le\frac H2$, then we can similarly derive from~\eqref{eqn:whpDiff} that under the true distribution $Q$,
\begin{align*}
L_Q(A_Q) - L_Q(a^*_Q)
&=\int_{A_Q}^H(q-F_Q(z))dz\\
&\ge\int_{\frac H2}^H(q-F_Q(z))dz\\
&=\int_{\frac H2}^H\left(\frac{C}{\sqrt{n}}- z\frac{C}{H\sqrt{n}}\right) dz\\
&=\frac{CH}{8\sqrt{n}}\\
&=\frac{1}{8\max\{\gamma,1\}}\left(\frac{q(1-q)}{3\sqrt{n}}\right)^{\frac{\beta+2}{\beta+1}}.
\end{align*}
The proof then finishes analogous to the first case.
\end{proof}

\section{Simulations}
\label{sec:simulation}

In this section, we conduct simulations using several commonly-used demand distributions to illustrate how our theory characterizes the regrets of SAA decisions in data-driven Newsvendor. These simulations serve not only to validate our theory, but also to demonstrate how our framework can predict (based on the empirical distribution) which distributions are likely to incur the most regret.
% in real-world applications.

We numerically compute the expected additive regrets for several distributions under different values of $q$ and number of samples $n$. Additionally, we provide the 95th percentile of additive regret in \Cref{apx:simulation_supp}, which represents the high-probability additive regret with $\delta=0.05$.
By comparing these simulation results with the relative order of $\beta$ across distributions, we show that our theory largely captures the comparison of regrets across distributions as the number of samples grows. In particular, it helps explain the \enquote{crossover points} in the regret curves, a phenomenon that previous theories relying on lower-bounding the PDF \citep{besbes2013implications,lin2022data} fail to capture.

\paragraph{Setup.} The distributions we use are taken from \citet[Table 4]{besbes2023big}. We provide the details of these distributions and plot their CDFs in \Cref{fig:cdf}, where the dashed lines represent $q=0.4$ and $q=0.9$ as considered in the simulations.
The number of samples, $n$, ranges from 1 to 200 for both $q=0.4$ and $q=0.9$, with results computed every 5 points.
To approximate the expected additive regret under a fixed $n$, we randomly generate $n$ samples from the underlying distribution to compute $\ha$, calculate the exact value of $L(\ha)-L(a^*)$ by \eqref{eqn:whpDiff}, and then average over $10,000$ repetitions. 

\begin{figure}[]
    \centering
    \includegraphics[width=0.5\linewidth]{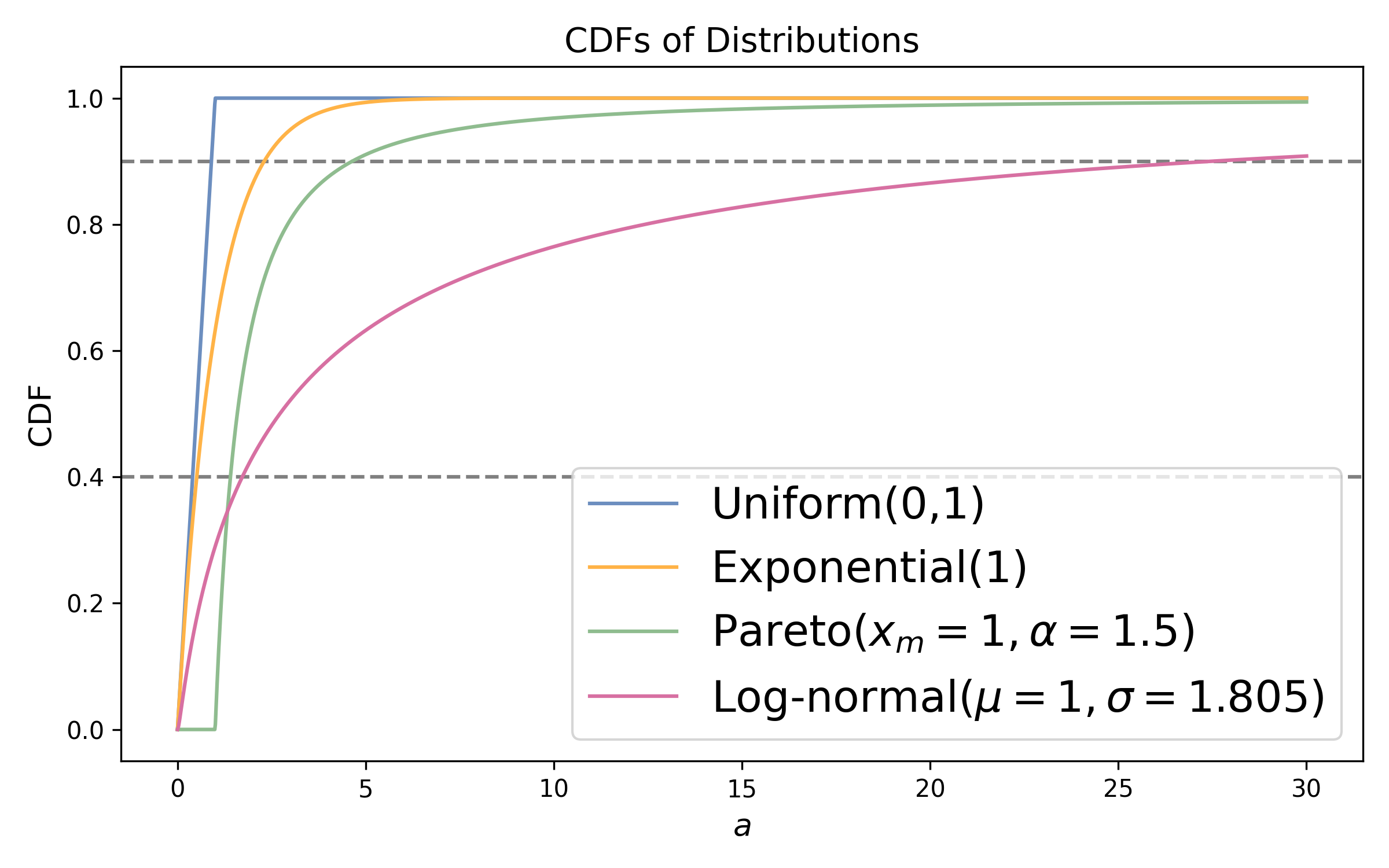}
    \caption{The CDFs of the distributions with $q=0.4$ and $q=0.9$. }
    \label{fig:cdf}
\end{figure}

Here we do not specify values of $\beta$ for each distribution, as it depends on the values of $\gamma$ and $\zeta$. Instead, we use the following as a proxy for $\beta$:
\begin{align} \label{eqn:DeltaEps}
\Delta(\eps):=\max\{F^{-1}(\min\{q+\eps,1\})-a^*,a^*-F^{-1}(\max\{q-\eps,0\})\}.
\end{align}
To explain why, let $\eps=|F(\ha)-q|$.  This means that $F(\ha)$ equals either $q-\eps$ or $q+\eps$, which implies that $\Delta(\eps)$ is an upper bound on $|\ha-a^*|$, noting that $\ha=F^{-1}(F(\ha))$ for the continuous distributions being considered.  Our definition of clustered distributions~\eqref{eqn:clustered} is satisfied if $\Delta(\eps)\le \frac1\gamma \eps^{\frac{1}{\beta+1}}$, which requires a larger $\beta$ for larger values of $\Delta(\eps)$ (assuming a fixed value of $\gamma$).  This explains why here we use $\Delta(\eps)$ as a proxy for $\beta$.
In the latter part of \Cref{apx:full_spectrum_egs}, we do provide formulas for $\beta$ in terms of  $\gamma$ and $\zeta$ and use them to show some concrete values for $\beta$.
% \Zhuoxinedit{For completeness, we report the corresponding values of $\beta$ for these distributions (under fixed $q$, $n$ and $\gamma$) in \Cref{apx:full_spectrum_egs}.}

Of course, comparing $\Delta(\eps)$ across distributions requires setting a value of $\eps$, just like comparing $\beta$ requires setting a value of $\zeta$.  Both parameters $\eps,\zeta$ can be interpreted as the "reasonable range of error" under a fixed value of $n$, in the quantile and decision spaces respectively.
Instead of trying to define an exact conversion from $n$ to $\eps$, we simply note that $\eps$ would shrink\footnote{The error in the quantile space, $\eps$, would shrink roughly at the rate of $1/\sqrt{n}$ (by the DKW inequality; see the beginning of our proof of \Cref{thm:hpAdd}).
% One benefit of considering error in the quantile space instead of decision space is that it is generally less dependent on the distribution.
} as $n$ grows, and look at how the curves $\Delta(\eps)$ of different distributions intersect as $\eps$ shrinks.
% $\eps$ represents the deviation of $F(\ha)$ from $F(a^*)$, which imposes an upper bound of $|\ha-a^*|$ according to the definition of clustered distributions.
% Consequently, $\Delta(\eps)$ reflects the corresponding values of $\beta$, given a common $\gamma$ across all distributions.
% Intuitively, as the number of samples increases, the data-driven solution $\ha$ is more likely to concentrate around $a^*$. To illustrate this, we plot the x-axis in reverse order of $\epsilon$, with the range of $\epsilon$ for $q=0.4$ and $q=0.9$ spanning from $1/50$ to 1.

\begin{figure}[t!]
    \centering
    \begin{subfigure}{0.48\textwidth}
        \centering
        \includegraphics[width=\linewidth]{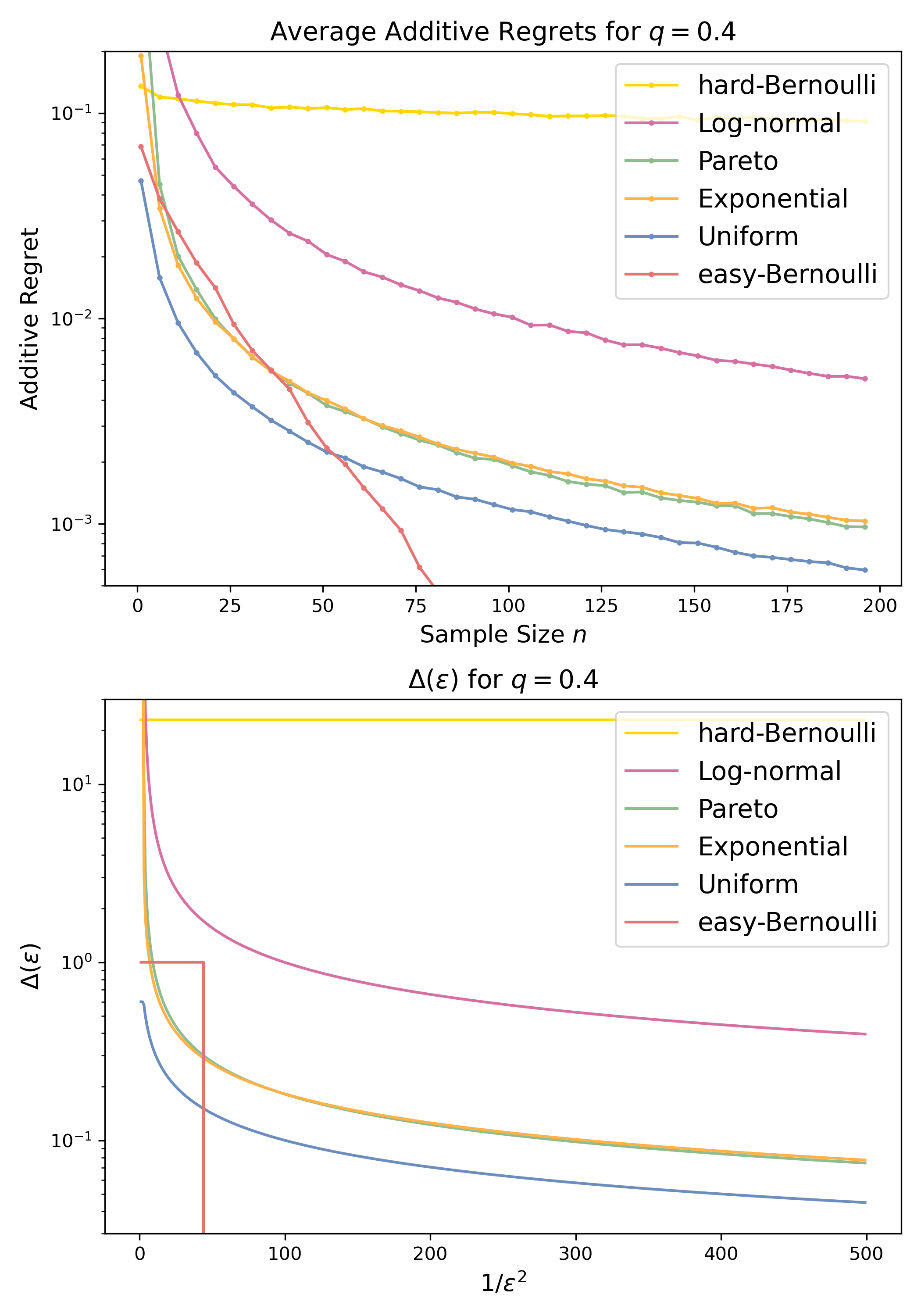} 
        \caption{$q=0.4$}
        \label{fig:addMeanA}
    \end{subfigure}
    \hfill
    \begin{subfigure}{0.48\textwidth}
        \centering
        \includegraphics[width=\linewidth]{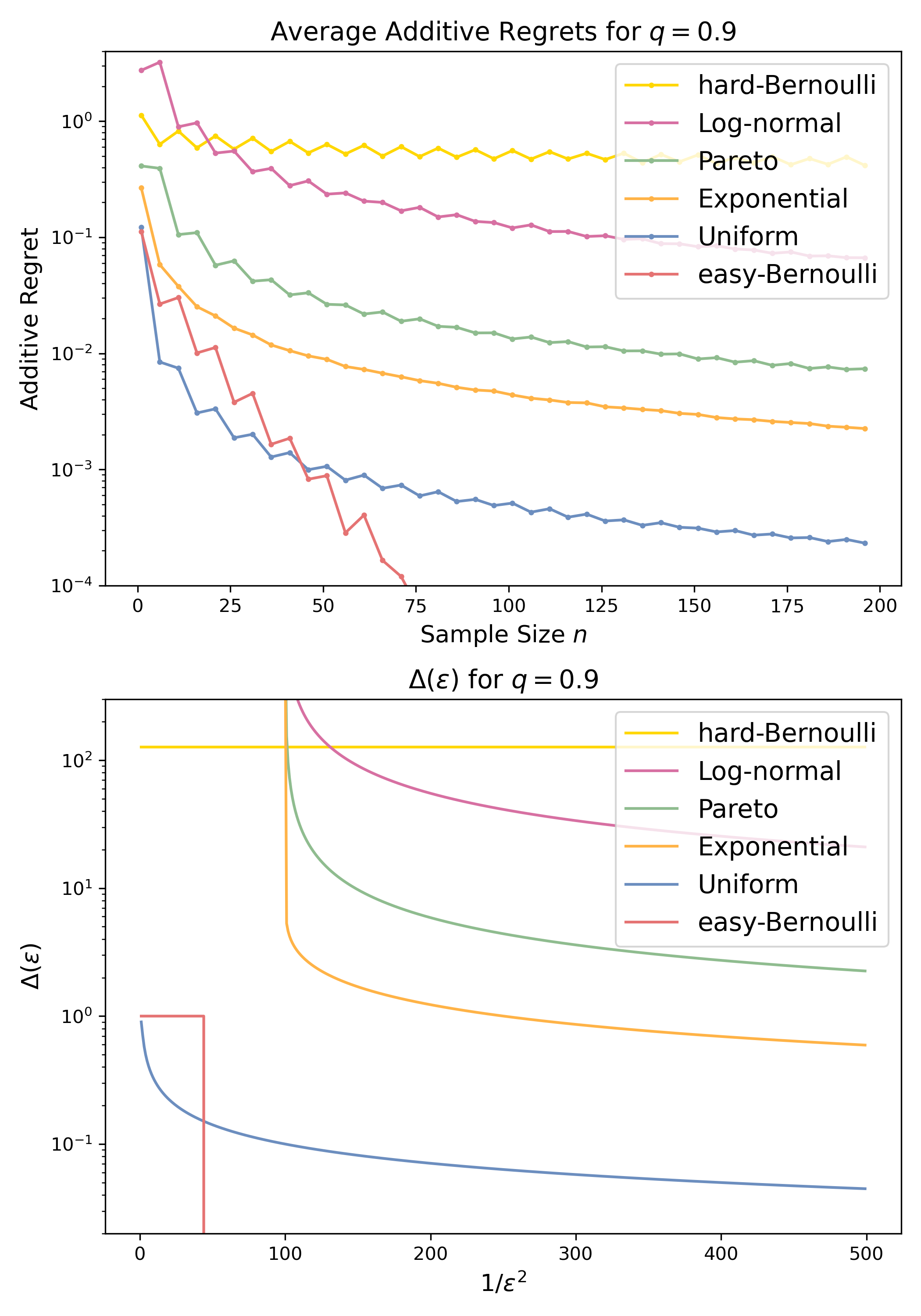}
        \caption{$q=0.9$}
        \label{fig:addMeanB}
    \end{subfigure}
\caption{Average additive regrets (top) and the values of $\Delta(\eps)$ (bottom) for the distributions under $q=0.4$ and $q=0.9$. 
Note that all vertical axes are plotted on a logarithmic scale. Also note that the horizontal axes for the $\Delta(\eps)$ plots are $1/\eps^2$, with $\eps$ decreasing as one moves to the right, reflecting the scaling that $\eps$ is roughly $1/\sqrt{n}$.}
\label{fig:addMean}
\end{figure}

\paragraph{Results.}
Importantly, the intersections in the $\Delta(\eps)$ curves \textit{are consistent with} the intersections in the numerical regret curves of the different distributions as $n$ grows, as shown in \Cref{fig:addMean}.
To elaborate, $\Delta(\eps)$ reflects the learning difficulty for a distribution (under a quantile $q$) when the error in the quantile space is $\eps$.  \Cref{fig:addMeanB} ($q=0.9$) shows that if the $\Delta(\eps)$ curves follow a consistent order at every value of $\eps$, then the corresponding numerical regret curves follow the same order at every value of $n$.  That is, the order of regret from lowest to highest is Uniform $<$ Exponential $<$ Pareto $<$ Log-normal at every value of $n$, which is explained by the $\Delta(\eps)$ curves being ordered Uniform $<$ Exponential $<$ Pareto $<$ Log-normal at every value of $\eps$.  Meanwhile, \Cref{fig:addMeanA} ($q=0.4$) has the $\Delta(\eps)$ curves for the Exponential, Pareto distributions cross at some value of $n$.  Consistently, the corresponding regret curves also cross at some value of $n$, with Exponential being easier to learn under large $\eps$ (small $n$) but harder to learn under small $\eps$ (large $n$).
% In each pair of plots, the two individual plots are highly consistent with each other.
% Our theory accurately predicts the overall ordering of additive regrets, which corresponds to the order of $\beta$.
% For instance, when $q=0.9$, the $\Delta(\eps)$ plot shows that $\beta$ follows the order Bernoulli $<$ Uniform $<$ Exponential $<$ Pareto $<$ Log-normal for most values of $n$, and the order of additive regrets exactly mirrors this. Similarly, when $q=0.4$, the order of $\beta$ also captures the overall ranking of regrets well.
% To highlight the crossing point between the Exponential and Pareto distributions for $q=0.4$ (which is discussed in the following paragraph), we focus on smaller values of $n$ in \Cref{fig:addMean}. This causes the regrets of Exponential and Pareto to appear close after the crossing point. However, for larger values of $n$, Pareto’s regret is consistently smaller than Exponential’s, remaining consistent with the ranking indicated by $\beta$.

\Cref{fig:addMean} also includes two Bernoulli distributions, to demonstrate more nuanced crossing points.
For both $q=0.4$ and $q=0.9$, we include an "easy-to-learn" Bernoulli distribution whose probability is far away from $1-q$, and a "hard-to-learn" Bernoulli distribution whose probability is close to $1-q$
% \Willedit{Both Bernoulli distributions are weighted to have mean comparable to the other distributions}
(see \Cref{apx:simulation_supp} for details).
% \Zhuoxincomment{In the current plot, the easy-Bernoulli distribution is simply the standard Bernoulli taking values in 0 and 1. I constructed the hard-Bernoulli as a weighted Bernoulli with a higher mean because it was difficult to find a standard Bernoulli that performs worse than the Log-normal. I think it unnecessary to scale the easy-Bernoulli to have a mean of 14, as the simulations already demonstrate that our theory allows for comparisons across distributions with different means (e.g., the Uniform has mean 0.5, while the Log-normal has mean 14). I also generated a version with a scaled easy-Bernoulli, but the resulting plot doesn’t look as nice (I can show it during the meeting).}
We again observe consistent
% Beyond the overall ordering, our theory also provides precise predictions for reversals in additive regrets by identifying the crossing points of $\beta$.
% A notable example occurs when $q=0.4$, where the Exponential and Pareto distributions have a crossing point in $\beta$, leading to a reversal of their additive regrets around $n=25$. Although the difference in regrets before and after the reversal is relatively small, the notion of $(\beta,\gamma,\zeta)$-clustered distribution successfully anticipates this behavior.
% Additionally, we observe similar
crossing points in \Cref{fig:addMean}, such as between the easy-Bernoulli and Uniform distributions when $q=0.4$, and between the Log-normal and hard-Bernoulli distributions for both $q=0.4$ and $q=0.9$.
We note that it is possible for $\Delta(\eps)$ curves to cross multiple times, such as between the easy-Bernoulli and Exponential distributions when $q=0.4$, explaining why easy-Bernoulli has lower average numerical regret than Exponential when $n$ is small or big but higher numerical regret for an intermediate range of $n$.
% The pairs of plots are consistent in reflecting these crossing points.
% Furthermore, our theory is able to capture more complex cases.
% For example, when $q=0.4$, the bottom plot shows two crossing points between the Bernoulli and Exponential distributions, which directly correspond to the behavior of the additive regrets in the upper plot. This illustrates the ability of our theory to handle multiple crossing points and further highlights its effectiveness in characterizing the performance of SAA on data-driven Newsvendor problems.
We note that similar patterns are observed for the 95th percentile regrets, as shown in \Cref{apx:simulation_supp}.

\paragraph{Discussion and limitations.}
% While our theory accurately predicts most crossing points, there are cases where the regrets exhibit crossing points that are not reflected in the $\Delta(\epsilon)$ plot. For instance, when $q=0.9$, the regrets of the Bernoulli and Uniform distributions cross at $n=1$. Notably, these crossings occur near the edges of the plot, where $n$ is small, specifically close to 1, and $\epsilon$ is greater than or close to $\min\{q, 1-q\}$. This discrepancy may be due to the small sample size, which introduces greater randomness in these extreme cases. These edge cases, however, are quite rare and do not significantly affect the overall validity of our theory.
Our theory does not explain all crossing points observed in the numerical regret curves: when $q=0.9$, the regret curves of easy-Bernoulli and Uniform cross twice while their $\Delta(\eps)$ curves only cross once; on the other hand, the regret curves of Exponential and hard-Bernoulli do not cross even though their $\Delta(\eps)$ curves do.
% Our theory does not explain all crossing points, e.g.\ between the easy-Bernoulli and Uniform distributions, and between the Exponential and hard-Bernoulli distributions, when $q=0.9$.
Indeed, the actual numerical regret depends on the entire distribution and how this affects the understocking/overstocking costs, not just the inverse CDF values at $q\pm\eps$ that form the basis of our $\Delta(\eps)$ curves and our notion of $(\beta,\gamma,\zeta)$-clustered distributions.
That being said, our theory better captures actual numerical regrets than previous notions based on lower-bounding the PDF \citep{besbes2013implications,lin2022data}.  To elaborate, instead of having our $\Delta(\eps)$ plots based on the CDF, they could draw a similar plot of the minimum PDF value over a range $[a^*-\zeta,a^*+\zeta]$ for progressively shrinking $\zeta$.
% \Willcomment{Remind me: Did you try this, and it does not work?}
% \Zhuoxincomment{I didn't plot it, but theoretically, the minimum PDF with a shrinking $\zeta$ can explain the distributions here, because the CDFs are all concave. An example of a continuous distribution that cannot be explained by the minimum PDF is my previous distribution with two crossing points, which is convex to the right of $a^*$ and concave to the left.}
% We note that the simulation results lead to an important observation: despite being primarily designed for the large-sample regime, our theorems prove to be effective even for relatively small values of $n$, as seen in the range of $n$ considered in our simulations (1 to 200 for $q=0.4$, 1 to 500 for $q=0.9$).
However, if the minimum PDF value occurs at $a^*$, then their plot would be constant in $\zeta$, and hence be unable to explain crossover points in the regrets (see \Cref{apx:simulation_supp} for a full example).

Another key implication of our simulation results is that the property of being $(\beta,\gamma,\zeta)$-clustered is not only intrinsic to the underlying distribution, but also related to the number of samples $n$ we consider. Theoretically, any continuous distribution with a density at least $\gamma$ on a small neighborhood around $a^*$ is $(0,\gamma,\zeta)$-clustered for some sufficiently small $\zeta$. However, the simulations show that the value of $\zeta$ of interest is a variable that is decreasing in $n$, making it unreasonable to take arbitrarily small $\zeta$ for obtaining a smaller $\beta$. This point highlights the necessity of considering the entire spectrum of $\beta$ between 0 and $\infty$, even if could technically be $(\beta,\gamma,\zeta)$-clustered with $\beta=0$.

Finally, we acknowledge that our notion of clustered distributions cannot fully explain numerical multiplicative regrets. This is because the ordering of multiplicative regrets is strongly influenced by the ordering of $L(a^*)$, which depends on the entire CDF and is inherently arbitrary. Our theory emphasizes that the accumulation of the CDF in the small interval around $a^*$ plays a significant role in regret. In contrast, $L(a^*)$ is highly sensitive to the long tail of the distribution, which is largely unrelated to our theory and its focus on the local behavior near $a^*$.

\section{Conclusion}

We provide a survey of results and (simplified) proofs for data-driven Newsvendor, including both upper and lower bound analyses, varying along several dimensions: additive vs.\ multiplicative regret, high-probability vs.\ expectation bounds, and different distribution classes.
We introduce a notion of \textit{clustered distributions} based on the CDF, which shows the entire spectrum of convergence rates between $1/\sqrt{n}$ and $1/n$ is possible, and is also a useful predictor of empirical regret in simulations.
We hope this can be a useful reference for future scholars of data-driven Newsvendor, which can be broadly viewed as a foundation for data-driven decision making in Operations Research.

\paragraph{Acknowledgments.} The authors thank Meichun Lin for identifying an earlier error in a comparison to \citet{lin2022data}, Omar Besbes and Omar Mouchtaki for useful discussions related to the paper, and the anonymous referees and associate editor of \textit{Operations Research} as well as area editor Dennis Zhang for their comprehensive comments that improved the paper.

% \textbf{Zhuoxin Chen} is an undergraduate student majoring in Mathematics and Physics and Industrial Engineering at Tsinghua University.  She is interested in data-driven decision-making, with applications in supply chain and revenue management.  She was the winner of the 2025 INFORMS Undergraduate Operations Research Prize for this work.

% \textbf{Will Ma} is the Roderick H. Cushman Associate Professor in the Decision, Risk, and Operations Division of Columbia Business School. His research is on designing, analyzing, and deploying algorithms for Operations Research problems under uncertainty, specializing in e-commerce inventory \& fulfillment, assortment optimization, and online resource allocation.

\bibliographystyle{abbrvnat}
\bibliography{bibliography}

\clearpage
\appendix
\crefalias{section}{appendix} % <--- 加上这一句关键代码

\section{Examples of Clustered Distributions}
\label{apx:full_spectrum_egs}
\paragraph{Full spectrum of minimum possible $\beta$.}
We provide a theoretical construction of a distribution whose minimum possible $\beta$ can take any value in $[0,\infty]$. For a given $q$ and $\beta$, the simplest construction is to enforce equality in \eqref{eqn:clustered}, where $a^*,\gamma$, and $\zeta$ can be set to any desired values, as long as they satisfy $\zeta\le\frac1\gamma(\min\{q,1-q\})^{\frac{1}{\beta+1}}$. 
We prove by contradiction that it is impossible to achieve a smaller $\beta$ by adjusting $\gamma$ and $\zeta$ for this type of distribution. For a $(\beta, \gamma, \zeta)$-clustered distribution that satisfies equality in \eqref{eqn:clustered}, we have
\begin{align}
|a-a^*|&=\frac1\gamma|F(a)-q|^{\frac{1}{\beta+1}} &\forall a\in[a^*-\zeta,a^*+\zeta].
\label{eqn:smallestBeta}
\end{align}
Now, suppose this distribution is also a $(\beta',\gamma',\zeta')$-clustered distribution for some $\beta'<\beta$. Substituting \eqref{eqn:smallestBeta} into \eqref{eqn:clustered}, we obtain
\begin{align*}
\frac{\gamma'^{\beta'+1}}{\gamma^{\beta+1}}&\le|a-a^*|^{\beta-\beta'} &\forall a\in[a^*-\min\{\zeta,\zeta'\},a^*)\cup(a^*,a^*+\min\{\zeta,\zeta'\}].
\end{align*}
When $a\to a^*$, since $\beta-\beta'>0$, the right-hand side tends to 0, while the left-hand side remains a positive constant. This leads to a contradiction, and thus no smaller $\beta$ exists.

\begin{figure}[t!]
    \centering
    \includegraphics[width=0.4\linewidth]{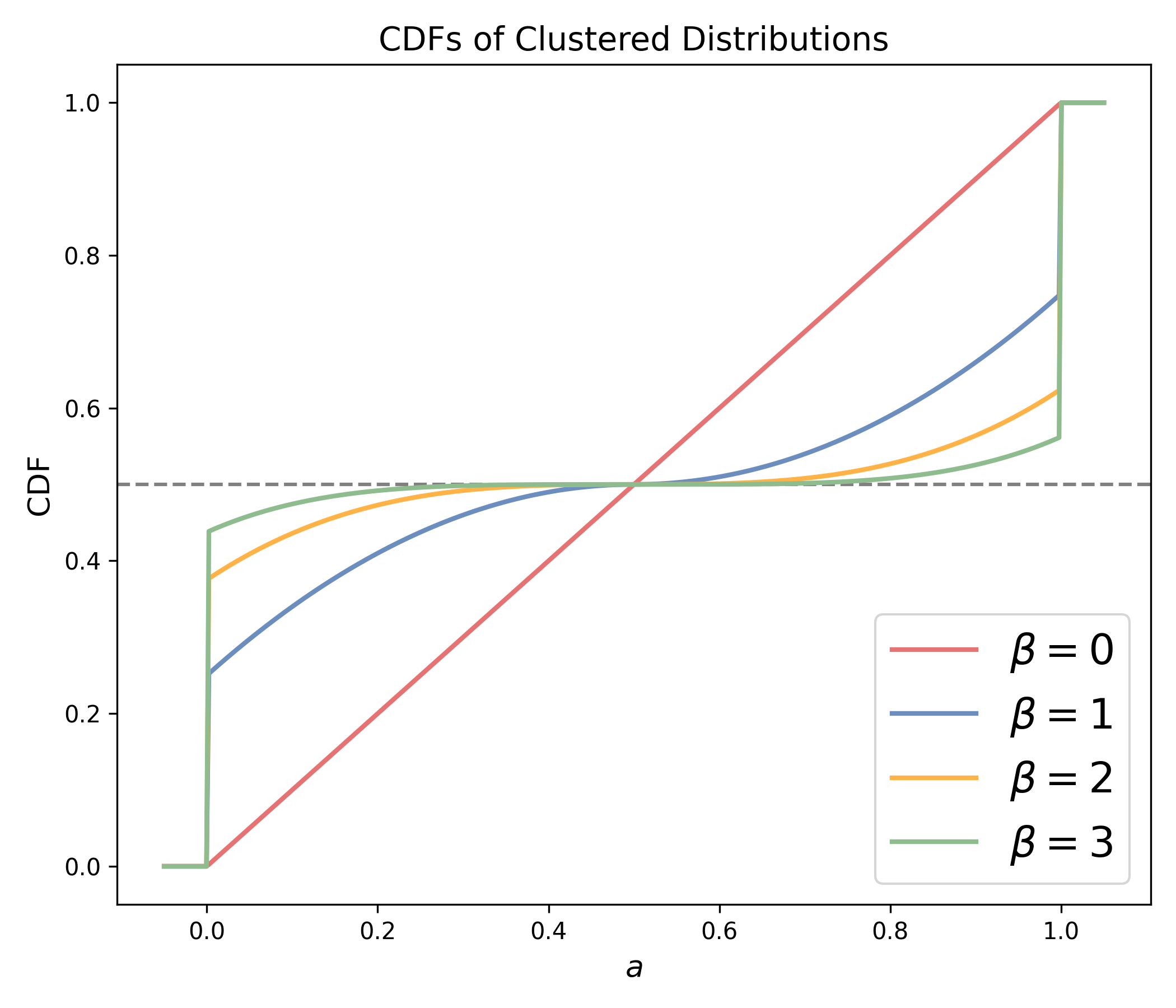}
    \caption{CDFs of the constructed distribution for four specific values of the minimum possible $\beta$ (0, 1, 2, and 3), where the parameters are set as $q = 0.5$, $a^* = 0.5$, $\gamma = 1$, and $\zeta = 0.5$.}
    \label{fig:clusteredCDF}
\end{figure}

We provide an example of this construction in \Cref{fig:clusteredCDF}, where $q=0.5,a^*=0.5, \gamma=1,$ and $\zeta=0.5$. With these parameters, the minimum possible value of $\beta$ can take any value within the range $[0, \infty]$. In \Cref{fig:clusteredCDF}, we show the CDFs of the constructed distribution for four specific values of the minimum possible $\beta$: 0, 1, 2, and 3.

\paragraph{Computing $\beta,\gamma,\zeta$ for specific distributions.}

In this part, we show how to determine the exact values of $\beta,\gamma,\zeta$
% \Willedit{(as well as $\tau$ which is used for the multiplicative regret bounds)}
for specified distributions and given values of $q$ and $n$, using the distributions in \Cref{sec:simulation} as examples. We note that for most distributions, the minimum possible value of $\beta$ can be uniquely determined only after $\gamma$ and $\zeta$ are fixed.
Therefore, in the results below, we fix $\gamma$ and choose $\zeta$ based on the number of samples $n$ first.

To determine an appropriate value of $\zeta$, the key observation is that the term $|F(a)-q|$ in the definition of clustered distributions \eqref{eqn:clustered} converges at the rate $O(1/\sqrt{n})$ by the DKW inequality. Hence, for a given $n$, a reasonable approximation is to take the smallest $\zeta$ such that either
$F(a^*+\zeta)-q\ge\min\{\frac{1}{\sqrt{n}},1-q\}$ 
or $q-F(a^*-\zeta)\ge\min\{\frac{1}{\sqrt{n}},q\}$.
This leads to 
\begin{align*}
\zeta=\max\left\{
F^{-1}\left(\min\left\{q+\frac{1}{\sqrt{n}},1\right\}\right)-a^*,
a^*-F^{-1}\left(\max\left\{q-\frac{1}{\sqrt{n}},0\right\}\right)
\right\}.
\end{align*}
We note that this expression can be obtained directly by substituting $\eps=1/\sqrt{n}$ into \eqref{eqn:DeltaEps}.

With fixed $\gamma$ and $\zeta$, we can determine the minimum possible value of $\beta$ that satisfies the definition of clustered distributions \eqref{eqn:clustered}. 
In \Cref{tab:formula}, we present the resulting analytical expressions for $\zeta$ and $\beta$ for several distribution families, written explicitly in terms of the corresponding distribution parameters. These results cover all distribution families considered in \Cref{sec:simulation}, except for the Log-normal distribution. Log-normal CDFs do not admit closed-form elementary expressions, which makes it difficult to derive general analytical expressions for the minimum possible $\beta$. Nevertheless, for fixed values of $q$, $n$, $\gamma$, and the parameters of a Log-normal distribution, we compute the corresponding values of $\zeta$ and $\beta$ numerically, and report them in \Cref{tab:value} (see the next paragraph).
% \Zhuoxincomment{Should I add some proofs for this table? For instance, I can show how the expressions for $\beta$ are obtained for Uniform, Exponential, and Pareto, using concavity of their CDFs. For Log-normal, deriving $\zeta$ is straightforward, but I am not sure how to obtain $\beta$ because its CDF is neither convex nor concave.} \Willcomment{No proofs needed since it is analytic derivations}

\begin{table}[t!]
    \centering
    \begin{tabular}{|c|c|c|}
    \hline
Distribution & $\zeta$ & $\beta$ ($\gamma\zeta\le1$)
\tablefootnote{The definition of clustered distributions indicates $\zeta\le\frac1\gamma(\min\{q,1-q\})^{\frac{1}{\beta+1}}$, which means that there is no solution for $\beta$ when $\gamma\zeta>1$. Therefore in this table we assume that the fixed $\gamma$ satisfies $\gamma\zeta\le1$.
% \Zhuoxincomment{The condition for $\beta=\infty$ is very restrictive in this table because we fix $\gamma$ and $\zeta$ first. } \Willcomment{Right, you are saying that it is very hard for $\beta$ to be set to $\infty$ in Table 2.}
} \\ \hline
%Bernoulli
\makecell{Bernoulli $c\cdot\mathrm{Ber}(p)$\\  $F(a)=\begin{cases}
    0,&a<0\\1-p,&a\in[0,c)\\1,&a\ge c
\end{cases}$}
& \makecell[l]{
$
\begin{cases}
0, &\text{if }\frac{1}{\sqrt{n}}<|1-p-q|\\
&\text{or } \frac{1}{\sqrt{n}}=1-p-q\\
c, &\text{otherwise}
\end{cases}
$}
& \makecell[l]{\hspace*{0.0em}
$\begin{cases}
\infty, &\text{if }p+q=1\\
&\text{or }\gamma\zeta=1\\
\frac{\ln|1-p-q|}{\ln\gamma+\ln c}-1, &\text{if }p+q\ne1\\
&\text{and }\gamma\zeta\in(\frac{1}{\sqrt{n}},1)\\
0, &\text{otherwise}
\end{cases}
$} 
\\ \hline
% Uniform
% ~ & ~ & \multirow{3}{*}{Merged} & ~ \\ \cline{1-2} \cline{4-4}
\makecell{Uniform$(b_1,b_2)$ \\ $F(a)=\frac{a-b_1}{b_2-b_1}$\\$\forall a\in[b_1,b_2]$} 
& 
% \makecell{
% $\begin{cases}
% \frac{b_2-b_1}{\sqrt{n}},~\text{if }\frac{1}{\sqrt{n}}\le\max\{q,1-q\}\\
% (b_2-b_1)\max\{q,1-q\}, ~\text{otherwise}
% \end{cases}$} 
\makecell[l]{
$
\begin{cases}
\frac{b_2-b_1}{\sqrt{n}}, & \text{if } \frac{1}{\sqrt{n}}\le \max\{q,1-q\} \\ % [2mm]
\makebox[0pt][l]{$(b_2-b_1)\max\{q,1-q\}$,} \\
& \text{otherwise}
\end{cases}
$
}
& \multirow[l]{2}{*}{
\makecell[l]{\\ \\
$\begin{cases}
\infty, &\text{if }\gamma\zeta=1\\
-\frac{\ln n}{2(\ln\gamma+\ln\zeta)}-1,&\text{if }\gamma\zeta\in(\frac{1}{\sqrt{n}},1)\\
0, &\text{otherwise}
\end{cases}$
}}
\\ \cline{1-2}
% Exponential
\makecell{Exponential$(\lambda)$ \\ $F(a)=1-e^{-\lambda a}$ \\ $\forall a>0$} 
& \makecell[l]{
$\begin{cases}
\frac1\lambda\ln\frac{1-q}{1-q-\frac{1}{\sqrt{n}}},
&\text{if }\frac{1}{\sqrt{n}}<1-q\\
\infty, &\text{otherwise}
\end{cases}$}
& \\ \cline{1-2}
% Pareto
\makecell{Pareto$(x_m,\alpha)$\\$F(a)=1-\left(\frac{x_m}{a}\right)^\alpha$\\$\forall a\ge x_m$} 
& \makecell[l]{
$\begin{cases}
\makebox[0pt][l]{$\frac{x_m}{(1-q-\frac{1}{\sqrt{n}})^{\frac1\alpha}}-\frac{x_m}{(1-q)^{\frac{1}{\alpha }}}$,}\\
&\text{if }\frac{1}{\sqrt{n}}<1-q\\
\infty, &\text{otherwise}
\end{cases}$
} 
& \\ \hline
\end{tabular}
\caption{
Expressions for $\zeta$ and $\beta$ for selected distribution families considered in \Cref{sec:simulation}, under fixed values of $q$, $n$, and $\gamma$. The expressions are written in terms of the parameters of each distribution family. 
}
\label{tab:formula}
\end{table}

Based on \Cref{tab:formula}, we compute the exact values of $\zeta$ and $\beta$ for all distributions considered in \Cref{sec:simulation}. The results are reported in \Cref{tab:value}. 
Consistent with the setup in \Cref{sec:simulation}, we report values for $q=0.4$ and $q=0.9$.
For each value of $q$, we present results for both a small number of samples ($n=11$) and a large number of samples ($n=196$). The values of $n$ are chosen to remain within the range 1 to 200 used in \Cref{sec:simulation} and correspond to points where we computed regrets (every 5 integers, starting with 1).
The fixed value of $\gamma$ is smaller when $q=0.9$, which reflects the smaller slope at $a^*$ for the Exponential, Pareto, and Log-normal distributions.

\begin{table}[t!]
\centering
\begin{tabular}{|c|*{8}{>{\centering\arraybackslash}p{0.8cm}|}}
\hline
\multirow{3}{*}{Distribution} & \multicolumn{4}{c|}{$q=0.4~(\gamma=1)$} & \multicolumn{4}{c|}{$q=0.9~ (\gamma=0.1)$} \\
\cline{2-9}
& \multicolumn{2}{c|}{$n=11$} & \multicolumn{2}{c|}{$n=196$} & \multicolumn{2}{c|}{$n=11$} & \multicolumn{2}{c|}{$n=196$} \\
\cline{2-9}
& $\zeta$ & $\beta$ & $\zeta$ & $\beta$ & $\zeta$ & $\beta$ & $\zeta$ & $\beta$ \\
\hline
Easy-Bernoulli & 1 & $\infty$ & 0 & 0 & 1 & 0 & 0 & 0 \\
\hline
Uniform$(0,1)$ & 0.30 & 0 & 0.07 & 0 & 0.30 & 0 & 0.07 & 0 \\
\hline
Exponential$(1)$ & 0.70 & 2.34 & 0.13 & 0.28 & $\infty$ & N/A & 1.25 & 0.27 \\
\hline
Pareto$(1,1.5)$ & 0.83 & 5.57 & 0.12 & 0.26 & $\infty$ & N/A & 6.06 & 4.27 \\
\hline
Log-normal$(1,1.805^2)$ & 5.34 & N/A & 0.67 & 5.53 & $\infty$ & N/A & 56.75 & N/A \\
\hline
Hard-Bernoulli & 23 & N/A & 23 & N/A & 127 & N/A & 127 & N/A \\
\hline
\end{tabular}
\caption{
Values of $\zeta$ and $\beta$ for the distributions used in \Cref{sec:simulation}, under fixed values of $q$, $n$, and $\gamma$. \enquote{N/A} indicates that no feasible $\beta$ exists because $\gamma\zeta > 1$. Log-normal$(1,1.805^2)$ denotes the Log-normal distribution with $\mu=1$ and $\sigma=1.805$, consistent with \Cref{fig:cdf}.  The parameters for easy-Bernoulli and hard-Bernoulli differ across values of $q$, and these choices are detailed in \Cref{apx:simulation_supp}.
}
\label{tab:value}
\end{table}

Comparing with \Cref{fig:addMean}, we observe that the overall ordering of $\beta$ in \Cref{tab:value} is consistent with the empirical regret orderings under different combinations of $q$ and $n$ (recall that smaller values of $\beta$ correspond to smaller regret). For instance, by comparing the values of $\beta$ for the Exponential and Pareto distributions when $q=0.4$, we can anticipate a crossover point in the regret curves as $n$ increases.
In particular, for $n=11$, Exponential has a smaller value of $\beta$ than Pareto in \Cref{tab:value}, and indeed has a smaller empirical regret in \Cref{fig:addMeanA}; meanwhile, for $n=196$, Exponential has a larger $\beta$ than Pareto in \Cref{tab:value}, consistent with the larger empirical regret in  \Cref{fig:addMeanA}.

In certain cases, there is no solution for $\beta$; an example is the hard-Bernoulli distribution under all combinations of $q$ and $n$ listed in \Cref{tab:value}. This occurs because we use the same value of $\gamma$ for all distributions, while the corresponding values of $\zeta$ vary significantly across distributions. Since the condition $\gamma\zeta<1$ is required by the definition of clustered distributions, it may be impossible to choose a single value of $\gamma$ that ensures that all distributions in \Cref{tab:value} admit comparable values of $\beta$. For example, when $q=0.4$ and $n=11$, choosing $\gamma<1/23$ can ensure that the hard-Bernoulli distribution admits a valid $\beta$, but will force the values of $\beta$ for the easy-Bernoulli, Uniform, Exponential, and Pareto distributions to be $0$. This is one reason why \Cref{fig:addMean} plots a proxy for $\beta$ rather than $\beta$ itself.

When $\beta$ is not available for some distributions, the values of $\zeta$ can be used instead as an indicator of regret orderings. For example, when $q=0.9$ and $n=196$, the ordering of $\zeta$ aligns with the regret ordering shown in \Cref{fig:addMean}. However, when $n$ is very small, $\zeta$ can go to infinity for distributions such as the Exponential, Pareto, and Log-normal distributions, whose CDFs never attain $1$. This behavior is also reflected in the $\Delta(\eps)$ plot for $q=0.9$ in \Cref{fig:addMean}, where $\Delta(\eps)$ becomes finite only after $\eps$ exceeds a certain threshold.

Finally, we note that our choice of $\zeta$ based on $1/\sqrt{n}$ is asymptotic. 
The purpose of \Cref{tab:value} is to provide a rough indication of how $\zeta$ and $\beta$ change as $n$ increases, rather than to give a precise characterization or to identify the exact value of $n$ at which regret curves cross.

\section{Continuous Lower Bound for $\beta=0$} \label{apx:continuous}

We provide a self-contained lower bound of $\Omega(1/n)$ using two continuous distributions, which contrasts other lower bounds (see \citet{lyu2024closing}) that use a continuum of continuous distributions.
Our two distributions are obtained by modifying those from \Cref{thm:lowAdd} (which had point masses on $0$ and $H$) to be continuous in the case where $\beta=0$.

\begin{theorem} \label{thm:lowAdd2}
Fix $q\in(0,1)$ and $\gamma\in(0,\infty)$.
Any learning algorithm based on $n$ samples makes a decision with additive regret at least
\begin{align*}
\frac{q^2(1-q)^2}{72\max\{\gamma,1\}n}=\Omega\left(\frac1n\right)
\end{align*}
with probability at least 1/3 on some continuous distribution with density at least $\gamma$ over an interval in [0,1]. 
Therefore, the expected additive regret is at least
\begin{align*}
\frac{q^2(1-q)^2}{216\max\{\gamma,1\}n}=\Omega\left(\frac1n\right).
\end{align*}
\end{theorem}

\begin{proof}[Proof of \Cref{thm:lowAdd2}]
Let $C=\frac{q(1-q)}{3}$, $H=\frac{C}{\max\{\gamma,1\}\sqrt{n}}$. 
Fix $\eta\in\left(0,\min\left\{\frac{\min\{q,1-q\}-\frac{C}{\sqrt{n}}}{\gamma},\frac13q\right\}\right]$.
Consider two distributions $P$ and $Q$, whose respective CDF functions $F_P$ and $F_Q$ are:
\begin{align*}
    F_P(z)&=
    \begin{cases}
        0, &z\in(-\infty, 0)\\
        \frac{q}{\eta}z, &z\in[0,\eta)\\
        q+\frac{C}{H\sqrt{n}}(z-\eta), &z\in[\eta,H+\eta)\\
        q+\frac{C}{\sqrt{n}}+\frac{1-q-\frac{C}{\sqrt{n}}}{\eta}(z-H-\eta), &z\in[H+\eta,H+2\eta)\\
        1,&z\in[H+2\eta,\infty);
    \end{cases}\\
    F_Q(z)&=
    \begin{cases}
        0, &z\in(-\infty, 0)\\
        \frac{q-\frac{C}{\sqrt{n}}}{\eta}z, &z\in[0,\eta)\\
        q-\frac{C}{\sqrt{n}}+\frac{C}{H\sqrt{n}}(z-\eta), &z\in[\eta,H+\eta)\\
        q+\frac{1-q}{\eta}(z-H-\eta), &z\in[H+\eta,H+2\eta)\\
        1,&z\in[H+2\eta,\infty).
    \end{cases}
\end{align*}
From the CDF functions, it can be observed that $F_P(z)$, $F_Q(z)$ are continuous, and that the respective optimal decisions are $a_P^*=\eta$ and $a_Q^*=H+\eta$. We now show that any learning algorithm based on $n$ samples will incur an additive regret at least $\frac{q^2(1-q)^2}{72\max\{\gamma,1\}n}$ with probability at least 1/3, on distribution $P$ or $Q$.

It is easy to see that $P$ and $Q$ are continuous distributions with a positive density over interval $[0,H+2\eta]$, where $H+2\eta\le C+ \frac23q \le 1.$  To see that this density is at least $\gamma$, we need to check that all of the slopes
$$
\frac q\eta, \frac{C}{H\sqrt{n}}, \frac{1-q-\frac{C}{\sqrt{n}}}{\eta}, \frac{q-\frac{C}{\sqrt{n}}}{\eta}, \frac{1-q}{\eta}
$$
are at least $\gamma$.  We first derive that $\frac{C}{H\sqrt{n}}=\max\{\gamma,1\}\ge\gamma$.  We next derive that
$$
\frac q\eta\ge\frac{q-\frac{C}{\sqrt{n}}}{\eta}\ge\frac{q-\frac{C}{\sqrt{n}}}{\left(\frac{\min\{q,1-q\}-\frac{C}{\sqrt{n}}}{\gamma}\right)}\ge\gamma.
$$
Similarly, we derive that
$$
\frac{1-q}{\eta} \ge \frac{1-q-\frac{C}{\sqrt{n}}}{\eta} \ge \frac{1-q-\frac{C}{\sqrt{n}}}{\left(\frac{\min\{q,1-q\}-\frac{C}{\sqrt{n}}}{\gamma}\right)}\ge\gamma
$$
which completes the verification that $P$ and $Q$ have density at least $\gamma$ over an interval in [0,1].

We next analyze the squared Hellinger distance between $P$ and $Q$. Because the PDF's of $P$ and $Q$ only differ on $[0,\eta)$ and $[H+\eta,H+2\eta)$, standard formulas for Hellinger distance yield
\begin{align*}
\mathrm{H}^2(P,Q)
&=\frac12\int_0^\eta\left(\sqrt{\frac{q}{\eta}}-\sqrt{\frac{q-\frac{C}{\sqrt{n}}}{\eta}}\right)^2dz+\frac12\int_{H+\eta}^{H+2\eta}\left(\sqrt{\frac{1-q}{\eta}}-\sqrt{\frac{1-q-\frac{C}{\sqrt{n}}}{\eta}}\right)^2dz \\
&=\frac12\left(q+q-\frac{C}{\sqrt{n}}-2\sqrt{q\left(q-\frac{C}{\sqrt{n}}\right)}+1-q+1-q-\frac{C}{\sqrt{n}}-2\sqrt{(1-q)\left(1-q-\frac{C}{\sqrt{n}}\right)}\right). 
\end{align*}
Note that this is the same as \eqref{additiveHellinger}. Therefore, following the analysis in the proof of \Cref{thm:lowAdd}, we conclude that $\mathrm{TV}(P^n,Q^n)\le\frac13$.

Fix any (randomized) algorithm for data-driven Newsvendor, and consider the sample paths of its execution on the distributions $P$ and $Q$ side-by-side.  The sample paths can be coupled so that the algorithm makes the same decision for $P$ and $Q$ on an event $E$ of measure $1-\mathrm{TV}(P^n,Q^n)\ge 2/3$, by definition of total variation distance.
Letting $A_P,A_Q$ be the random variables for the decisions of the algorithm on distributions $P,Q$ respectively,
we have that $A_P$ and $A_Q$ are identically distributed conditional on $E$.
Therefore, either $\Pr[A_P\ge\frac H2+\eta|E]=\Pr[A_Q\ge\frac H2+\eta|E]\ge1/2$ or $\Pr[A_P\le\frac H2+\eta|E]=\Pr[A_Q\le\frac H2+\eta|E]\ge 1/2$.

First consider the case where $\Pr[A_P\ge\frac H2+\eta|E]=\Pr[A_Q\ge\frac H2+\eta|E]\ge1/2$.
Note that if $A_P\ge\frac H2+\eta$, then we can derive from~\eqref{eqn:whpDiff} that under the true distribution $P$,
\begin{align*}
L_P(A_P) - L_P(a^*_P)
&=\int_\eta^{A_P} \frac{C}{H\sqrt{n}}(z-\eta) dz\\
&\ge\int_\eta^{\frac H2+\eta} \frac{C}{H\sqrt{n}}(z-\eta) dz\\
&=\frac{C^2}{8\max\{\gamma,1\}n}\\
&=\frac{q^2(1-q)^2}{72\max\{\gamma,1\}n}.
\end{align*}
 \ \newline
Therefore, we would have
\begin{align*}
\Pr\left[L_P(A_P)-L_P(a^*_P)\ge\frac{q^2(1-q)^2}{72\max\{\gamma,1\}n}\right]
&\ge\Pr\left[A_P\ge\frac H2+\eta\right]\\
&\ge\Pr\left[A_P\ge\frac H2+\eta\bigg|E\right]\Pr[E]\\
&\ge\frac12 \cdot\frac 23=\frac13.
\end{align*}

Now consider the other case where $\Pr[A_P\le\frac H2+\eta|E]=\Pr[A_Q\le\frac H2+\eta|E]\ge 1/2$.
If $A_Q\le\frac H2+\eta$, then we can similarly derive from~\eqref{eqn:whpDiff} that under the true distribution $Q$,
\begin{align*}
L_Q(A_Q) - L_Q(a^*_Q)
% \ge\int_{\frac H2}^H(q-F_Q(z))dz
&=\int_{A_Q}^{H+\eta}(q-F_Q(z))dz\\
&\ge\int_{\frac H2+\eta}^{H+\eta}\left(\frac{C}{\sqrt{n}}- \frac{C}{H\sqrt{n}}(z-\eta)\right) dz\\
&=\frac{C^2}{8\max\{\gamma,1\}n}\\
&=\frac{q^2(1-q)^2}{72\max\{\gamma,1\}n}.
\end{align*}
The proof then finishes analogous to the first case.
\end{proof}

\section{Relationship Between the IFR Property and Clustered Distributions} \label{apx:IFR_proof}
We show in this section that any continuous distribution with the IFR property is a $(0,\gamma,\zeta)$-clustered distributions for some $\gamma$ and $\zeta$. (We note that the assumption of $F$ being a continuous distribution is important, because discrete IFR distributions would still require $\beta=\infty$ in our condition.)

To see this, we first show that any CDF $F$ with the IFR property always has left and right derivatives at every point $a\in\{a:0<F(a)<1\}$, which is the interval where $a^*$ may lie. We let $G(a)=\ln(1-F(a))$, and note that $G(a)$ is concave when $F$ is IFR. 
Since $F$ is continuous, the set $\{a:0<F(a)<1\}$ is an open interval, which forms the domain of $G$. It follows that $G$ has left and right derivatives at every point within this interval, as concave functions are differentiable from both sides at all interior points of their domain. Indeed, the right derivative of $F$ at $a^*$
\begin{align}
F'_+(a^*)&=\lim_{a\to a^{*+}}\frac{F(a)-F(a^*)}{a-a^*}\nonumber \\
&=\lim_{a\to a^{*+}}\frac{(1-e^{G(a)})-(1-e^{G(a^*)})}{a-a^*}\nonumber\\
&=\lim_{a\to a^{*+}}\frac{(1-e^{G(a)})-(1-e^{G(a^*)})}{G(a)-G(a^*)}\cdot\lim_{a\to a^{*+}}\frac{G(a)-G(a^*)}{a-a^*} \label{derivative}
\end{align}
exists because the first term in \eqref{derivative} is $-e^{G(a^*)}$ by continuity of $G$ and the second term in \eqref{derivative} is the right derivative of $G$ at $a^*$. Similarly, the left derivative $F'_-(a^*)$ also exists.

We next show that if both $F'_+(a^*)$ and $F'_-(a^*)$ are positive, then $F$ is $(0,\gamma,\zeta)$-clustered for sufficiently small $\gamma$ and $\zeta$.
Considering the right side, if $F'_+(a^*)=\lim_{a\to a^{*+}}\frac{F(a)-F(a^*)}{a-a^*}=C>0$, where $C$ is a positive constant, then by definition of limit, for any $\epsilon>0$, there exists $\zeta>0$ such that $\left|\frac{F(a)-F(a^*)}{a-a^*}-C\right|<\epsilon$ holds for all $a\in(a^*,a^*+\zeta)$. Setting $\epsilon=C/2$, we obtain $F(a)-F(a^*)>\frac{C}{2}(a-a^*)$. Therefore, we can choose $\gamma\le\frac C2$ and conclude that $F$ is a $(0,\gamma,\zeta)$-clustered distribution on the right side of $a^*$. The proof for the left side follows similarly. 

Finally, we show that any CDF $F$ that is IFR must have positive $F'_+(a^*)$ and $F'_-(a^*)$. If any of $F_+'(a^*),F_-'(a^*)$ is 0, 
then $0\in\partial (-G)(a^*)$, where $\partial (-G)(a^*)$ is the subdifferential of the convex function $-G$ at $a^*$. This implies $a^*$ is a minimizer of $-G$, and hence a maximizer of $G$.

% \Zhuoxincomment{Subdifferential is defined for convex functions, but I am afraid that superdifferential is not a standard concept. I can probably change the notation to $\partial(-G)(a^*)$.} \Willcomment{Yeah, I would say it's better to use the notation $\partial(-G)(a^*)$}

% \Zhuoxindelete{If $F'(a^*)$ is 0, then $G'(a^*)=\frac{-F'(a^*)}{1-F(a^*)}=0$, which implies that $a^*$ is a maximizer of $G$ by concavity. }
Since $\ln x$ is monotonic, it follows that $a^*$ is a minimizer of $F$, which is impossible unless $a^*=\min\{a:F(a)>0\}$ and $ F(a^*)\ge q$. However, this counterexample is excluded in \cite{zhang_sampling-based_2025} because they assume that $F$ is continuous. This completes the proof that any distribution with the IFR property must be $(0,\gamma,\zeta)$-clustered for some $\gamma$ and $\zeta$.

\citet[Remark 1]{zhang_sampling-based_2025} also compare their result with \citet[Thm.~4]{levi2015data}, who establish the sample complexity for a biased SAA algorithm under the assumption that the PDF of the demand distribution is log-concave (along with some additional assumptions). These assumptions imply that $F'(a^*)>0$, and therefore, the regret of SAA must be $O(1/n)$. 

\section{Projected SAA} \label{sec:projSAA}

If we did know $\mu(F)$ or more generally an upper bound $\mu$ on the mean of the distribution, then one could use a \textit{projected SAA} algorithm instead, where the SAA decision $\ha$ is projected to lie in $[0,\frac{\mu}{1-q}]$.  This interval is guaranteed to contain the optimal decision, because
\begin{align*}
\mu \ge \int_0^\infty (1-F(z))dz
\ge \int_0^{a^*} (1-F(z))dz
\ge \lim_{a\to a^{*-}}a(1-F(a))
\ge a^*(1-q).
% , \label{eqn:astarUB}
\end{align*}
% where the third inequality follows from properties of the Riemann integral, and the last inequality uses the fact that $F(a)<q$ for all $a<a^*$.
% \Willcomment{This would simplify the following analyses...}
This would simplify the analyses of high-probability and expected additive regrets for $\beta=\infty$.
% \Zhuoxincomment{I am not sure whether the proof of expected additive regret for $\beta<\infty$ can be simplified.}

\paragraph{High-probability additive regret for $\beta=\infty$.} The proof can be simplified because we now have $|\ha-a^*|\le\frac{\mu}{1-q}$. Consequently, the assumption $n\ge\frac{2\log(2/\delta)}{(1-q)^2}$ in \Cref{thm:hpAdd} is no longer needed.

When $\ha\le a^*$, we derive from \eqref{eqn:whpDiff} that
\begin{align*}
    L(\ha)-L(a^*)
    &=\int_{\ha}^{a^*}(q-F(z))dz\\
    &\le(a^*-\ha)(q-F(\ha))\\
    &\le\frac{\mu}{1-q}\sqrt{\frac{\log(2/\delta)}{2n}},
\end{align*}
where the second inequality applies \eqref{whpDiffBound1} and \eqref{eqn:DKWoutcome}.
On the other hand, when $\ha>a^*$, we similarly derive
\begin{align*}
    L(\ha)-L(a^*)
    &=\int_{\ha}^{a^*}(q-F(z))dz\\
    &\le(\ha-a^*)\lim_{a\to\ha^-}(F(a)-q)\\
    &\le\frac{\mu}{1-q}\sqrt{\frac{\log(2/\delta)}{2n}},
\end{align*}
where the first inequality follows from the properties of the Riemann integral, and the second inequality uses \eqref{whpDiffBound2} and \eqref{eqn:DKWoutcome}.

Therefore, we conclude that
$L(\ha)-L(a^*)\le\frac{\mu}{1-q}\sqrt{\frac{\log(2/\delta)}{2n}}$
holds for $\beta=\infty$ under the projected SAA for any $n$. We note that the proof for $\beta\in[0,\infty)$ cannot be simplified, because a lower bound on $n$ is still required to guarantee $\ha\in[a^*-\zeta,a^*+\zeta]$.

\paragraph{Expected additive regret for $\beta=\infty$.} Since for projected SAA, the upper bound of the high-probability additive regret for $\beta=\infty$ holds for any $n$, we are able to apply the method in \citet[Lem.~E-5]{besbes2023big} to convert the high-probability bound into an expected bound.

Specifically, from the bound derived above, $L(\ha)-L(a^*)\le\frac{\mu}{1-q}\sqrt{\frac{\log(2/\delta)}{2n}}$, we can equivalently express it as
\begin{align*}
\Pr\left[L(\ha)-L(a^*)>\eps
\right]\le2\exp\left(-2n\frac{(1-q)^2}{\mu^2}\eps^2\right).
\end{align*}
Therefore, we have
\begin{align*}
\bE[L(\ha)-L(a^*)]
&=\int_0^\infty \Pr\left[L(\ha)-L(a^*)>\eps
\right]d\eps\\
&\le \frac{1}{\sqrt{n}}+2\int_{\frac{1}{\sqrt{n}}}^\infty\exp\left(-2n\frac{(1-q)^2}{\mu^2}\eps^2\right)d\eps\\
&\le \frac{1}{\sqrt{n}}+2\int_{\frac{1}{\sqrt{n}}}^\infty\exp\left(-2\sqrt{n}\frac{(1-q)^2}{\mu^2}\eps\right)d\eps\\
&= \frac{1}{\sqrt{n}}+\frac{\mu^2}{\sqrt{n}(1-q)^2}\left.\exp\left(-2\sqrt{n}\frac{(1-q)^2}{\mu^2}\eps\right)\right|_{\infty}^{\eps=\frac{1}{\sqrt{n}}}\\
&=\frac{1}{\sqrt{n}}\left(1+\frac{\mu^2}{(1-q)^2}\exp\left(-2\frac{(1-q)^2}{\mu^2}\right)\right),
\end{align*}
where the first inequality follows from the fact that $\Pr\left[L(\ha)-L(a^*)>\eps
\right]\le1$ on the interval $[0,\frac{1}{\sqrt{n}}]$, and the second inequality holds because $\eps\ge\frac{1}{\sqrt{n}}$ on the interval $[\frac{1}{\sqrt{n}},\infty)$.
We note that it is sufficient to integrate up to $\frac{\mu}{1-q}$,
because $L(\ha)-L(a^*)\le |\ha-a^*|\max_{a}|F(a)-q|\le\frac{\mu}{1-q}$. Here we integrate up to $\infty$ to demonstrate the generality of this method and to avoid addressing whether $\frac{1}{\sqrt{n}}\le\frac{\mu}{1-q}$.

In addition, the original proof of the expected additive regret for $\beta=\infty$ in \Cref{thm:expAdd} can also be simplified.
We derive from \eqref{eqn:expDiff} that
\begin{align*}
&\quad\bE[L(\ha)]-L(a^*) \nonumber\\
&=\int_0^{a^*}(q-F(z))\Pr[\hF(z)\ge q]dz+\int_{a^*}^{\frac{\mu}{1-q}}(F(z)-q)\Pr[\hF(z)<q]dz\\
&\le\int_0^{\frac{\mu}{1-q}}|q-F(z)|\exp\left(-2n|q-F(z)|^2\right)dz\\
&\le\int_0^{\frac{\mu}{1-q}}\frac{1}{2\sqrt{en}}dz\\
&=\frac{\mu}{2(1-q)\sqrt{en}},
\end{align*}
where the first inequality applies Hoeffding's inequality, and the second inequality follows from the fact that the function $g(x)=xe^{-2nx^2}$ is at most $\frac{1}{2\sqrt{en}}$ for all $x\ge0$. This proof is simpler because we no longer need to deal with the tail term $\int_{\frac{\mu}{1-q}}^{\infty}(F(z)-q)\Pr[\hF(z)<q]dz$.

\section{Supplement to Simulations}
\label{apx:simulation_supp}

\paragraph{Details about distributions.}
The details of the Bernoulli distributions in \Cref{fig:addMean} are as follows.
We use $\Ber(0.45)$ and $\Ber(0.25)$ as the easy-Bernoulli distributions for $q = 0.4$ and $q = 0.9$, respectively, and $23 \cdot \Ber(0.59)$ and $127 \cdot \Ber(0.11)$ as the corresponding hard-Bernoulli distributions.
Here, $c\cdot\Ber(p)$ denotes a scaled Bernoulli distribution taking values in ${0, c}$ instead of ${0,1}$, with $c$ chosen to keep the mean below 14. This ensures the Bernoulli distributions have similar means as the other distributions considered in the simulations.

\paragraph{95th percentile regrets.}
Meanwhile, in this \namecref{apx:simulation_supp} we also present the 95th percentile of additive regrets across the distributions in \Cref{fig:addPercentile} to validate our high-probability bounds. By comparing these results with the plots of $\Delta(\eps)$ at the bottom of \Cref{fig:addMean}, it can be observed that the order of high-probability additive regrets closely follows the order of $\beta$, with our theory generally capturing the crossover points.

\begin{figure}[]
    \centering
    \begin{subfigure}{0.48\textwidth}
        \centering
        \includegraphics[width=\linewidth]{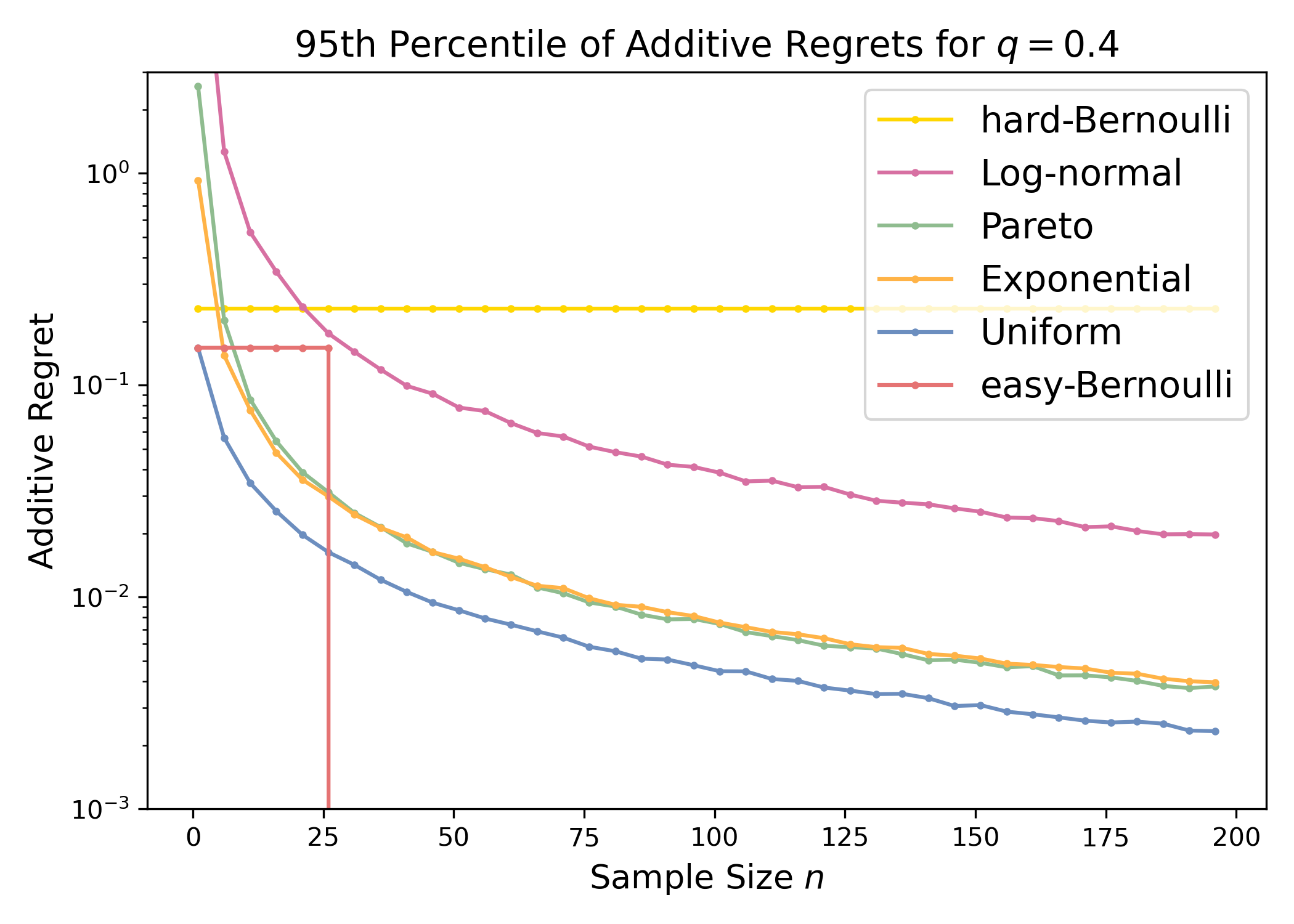} 
        \caption{$q=0.4$}
    \end{subfigure}
    \hfill
    \begin{subfigure}{0.48\textwidth}
        \centering
        \includegraphics[width=\linewidth]{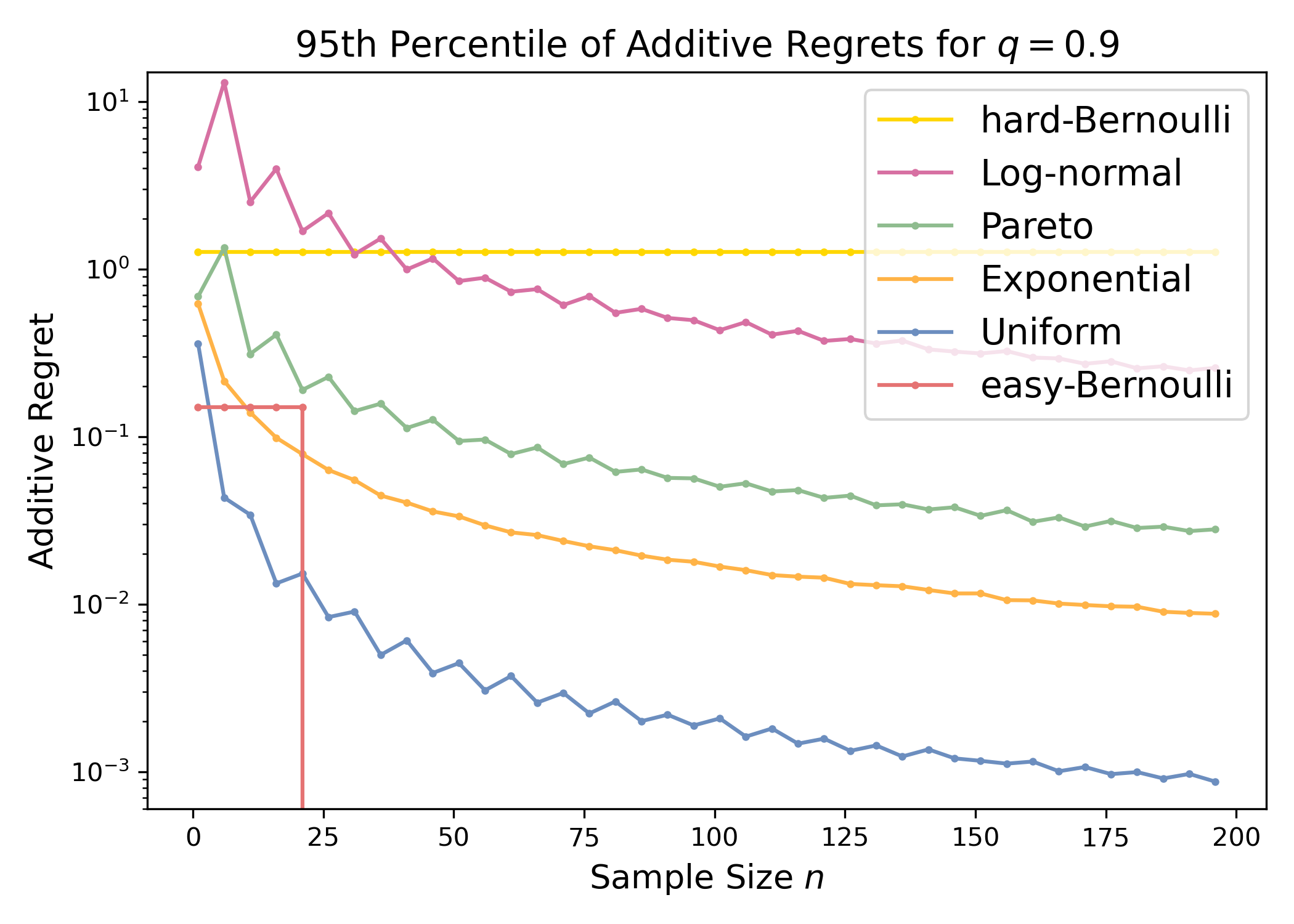}
        \caption{$q=0.9$}
    \end{subfigure}
\caption{95th percentile of additive regrets for the distributions under $q=0.4$ and $q=0.9$.}
\label{fig:addPercentile}
\end{figure}

We note some differences compared to the previous plots of average additive regrets (\Cref{fig:addMean}).
Firstly, for both $q=0.4$ and $q=0.9$, the 95th percentile of the easy-Bernoulli's additive regret exhibits a sharp decline, intersecting with the regret curves of other distributions at the same value of $n$. This behavior contrasts with \Cref{fig:addMean}, where the crossing points between the easy-Bernoulli and other distributions do not occur simultaneously. The difference arises because, for the 95th percentile, the regret can only take values in either $L(0)-L(a^*)$ or $L(1)-L(a^*)$, where $a^*=0$ for $q=0.4$ and $a^*=1$ for $q=0.9$, while in the previous case we plotted the average of additive regrets. Consequently, the sharp decline in the 95th percentile indicates that the probability of SAA making a mistake on the easy-Bernoulli distribution drops below $0.05$.
Similar sharp declines can also be observed for the hard-Bernoulli distributions when the number of samples $n$ is sufficiently large.
Secondly, there are some additional crossing points that are not captured by the $\Delta(\eps)$ plots. For example, when $q=0.9$, the easy-Bernoulli and Exponential distributions have two crossing points, while neither their additive regret curves nor their $\Delta(\eps)$ curves cross in \Cref{fig:addMean}. However, the majority of the behavior aligns well with the $\Delta(\eps)$ plots.

\paragraph{Minimum PDF fails.}
\begin{figure}[t!]
    \centering
    \begin{subfigure}{0.48\textwidth}
        \centering
        \includegraphics[height=5.64cm]{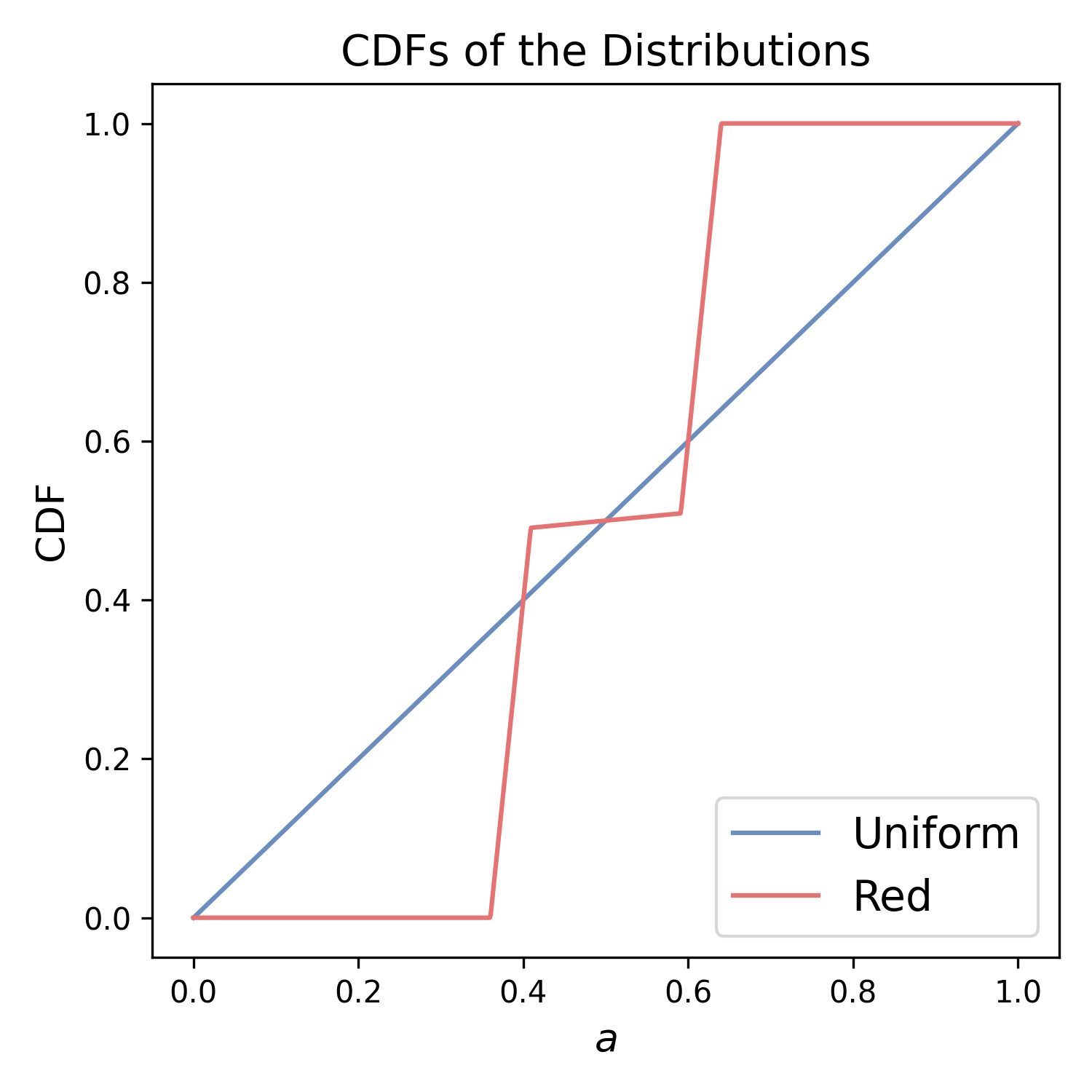} 
        \caption{CDFs of the two distributions}
        \label{fig:min_pdf_cdf}
    \end{subfigure}
    \hfill
    \begin{subfigure}{0.48\textwidth}
        \centering
        \includegraphics[width=\linewidth]{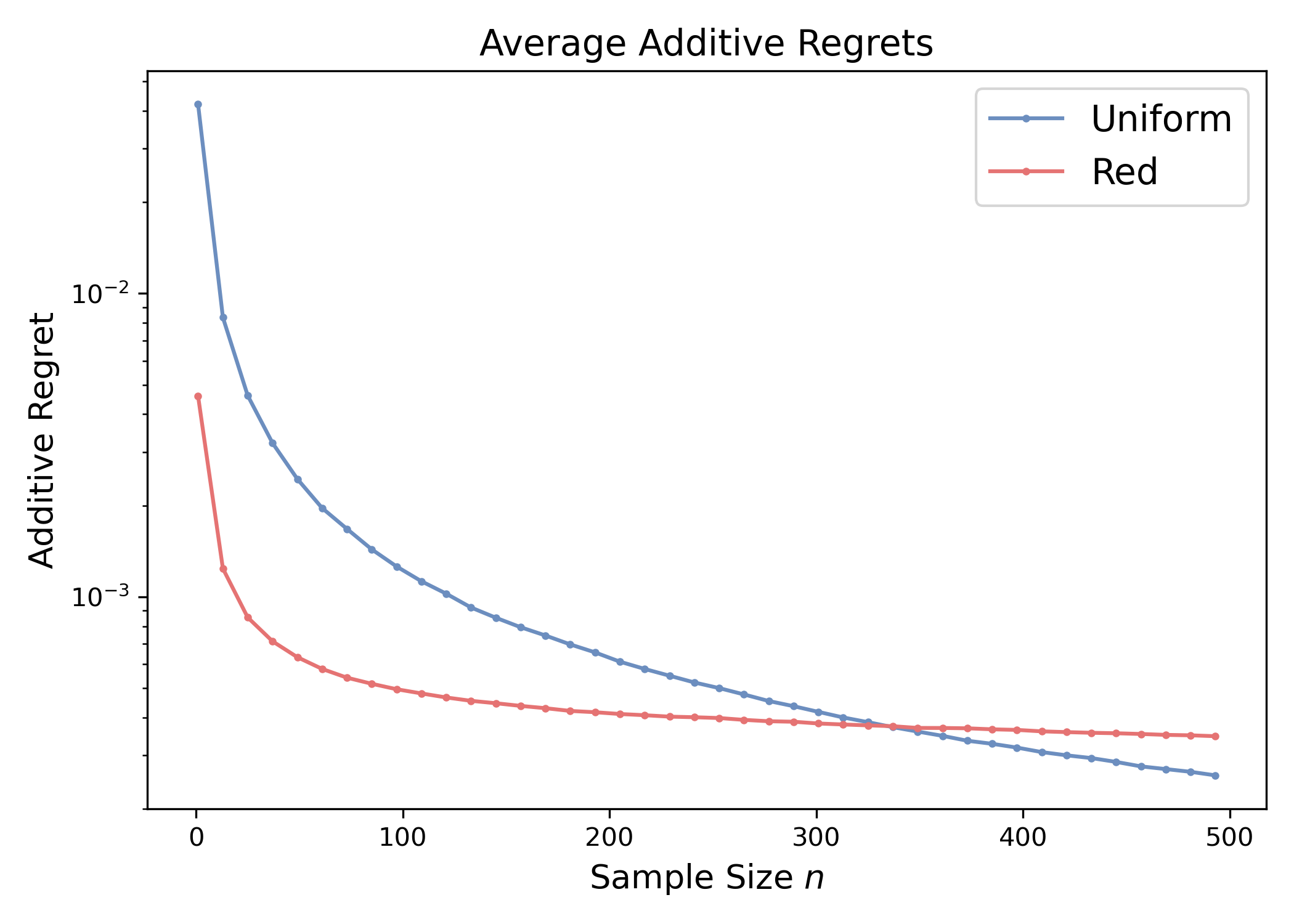}
        \caption{Average additive regrets for the two distributions}
        \label{fig:min_pdf_regret}
    \end{subfigure}
    \hfill
    \begin{subfigure}{0.48\textwidth}
        \centering
        \includegraphics[width=\linewidth]{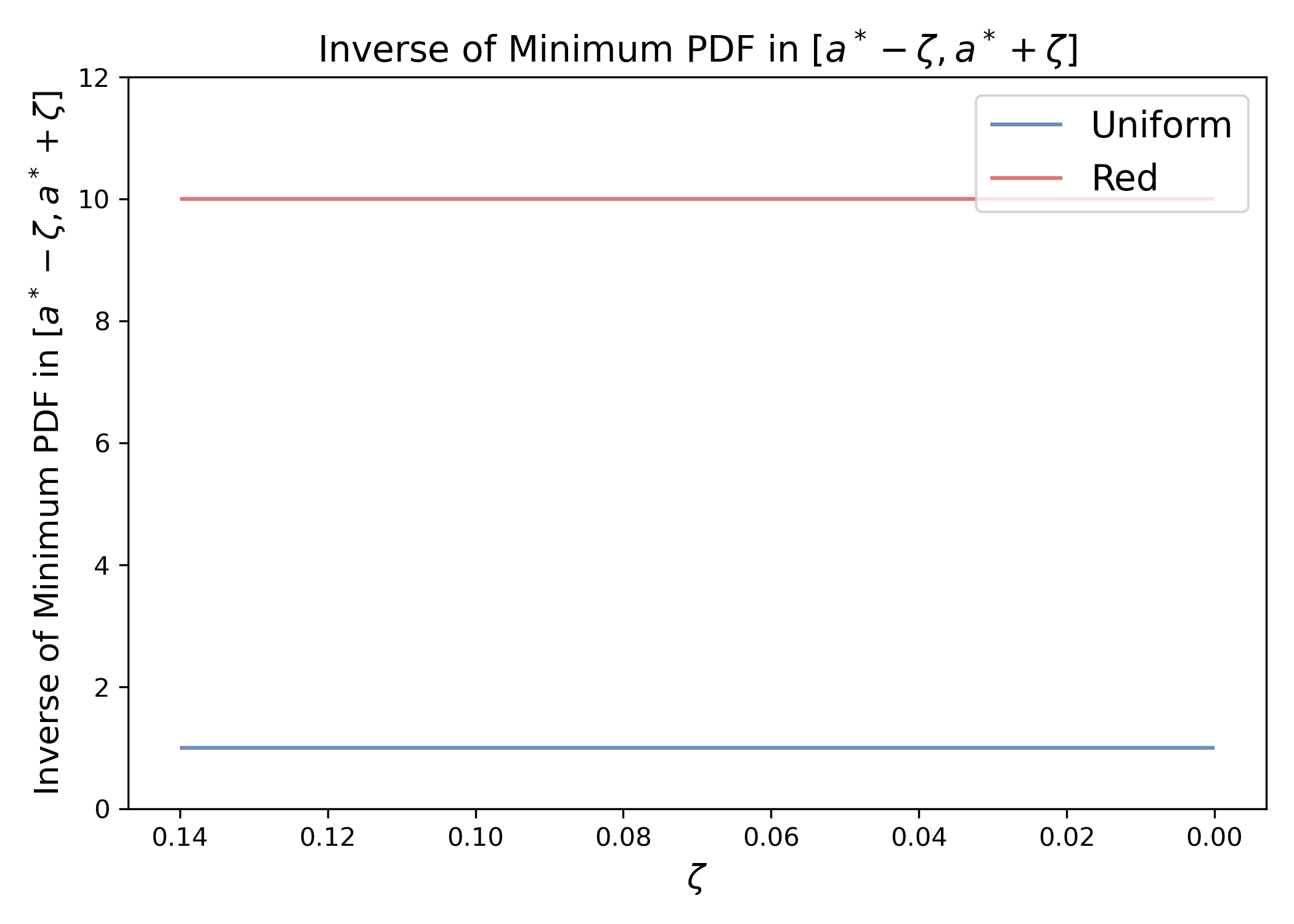}
        \caption{Inverse of minimum PDF in $[a^*-\zeta,a^*+\zeta]$}
        \label{fig:min_pdf}
    \end{subfigure}
    \hfill
    \begin{subfigure}{0.48\textwidth}
        \centering
        \includegraphics[width=\linewidth]{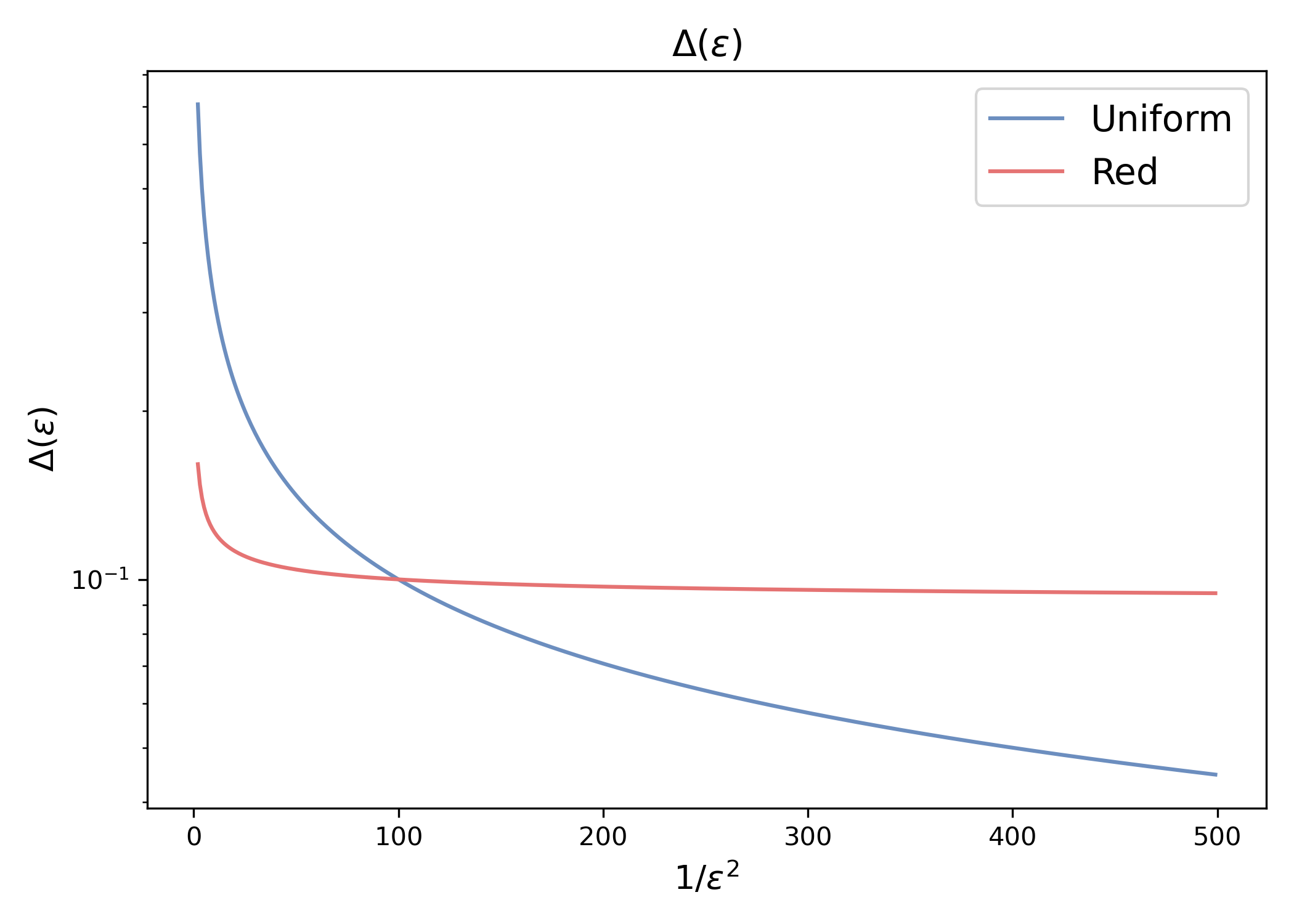}
        \caption{$\Delta(\eps)$ for the two distributions}
        \label{fig:min_pdf_delta}
    \end{subfigure}  
\caption{An example where the theory based on minimum PDF around $a^*$ fails to explain the crossing points in regret curves.
% \Zhuoxincomment{I don't know how to name the red distribution. In \Cref{fig:min_pdf} I only plot $\zeta\in[0,0.14]$ because the minimum PDF of the red distribution is 0 for larger $\zeta$.}
}
\label{fig:minimumPDF}
\end{figure}

We provide an example where the minimum PDF fails to capture the crossover points in the regret curves, in \Cref{fig:minimumPDF}. Specifically, we construct a "Red" distribution whose minimum PDF is attained at $a^*=0.5$, and compare it with the Uniform$(0,1)$ distribution under $q=0.5$. A theory based on minimum PDF would suggest that a lower PDF around $a^*$ should imply higher regret. Therefore, as shown in \Cref{fig:min_pdf}, where we plot the inverse of the minimum PDF around $a^*$ for both distributions with vanishing $\zeta$, the Red distribution is expected to consistently incur higher regret than the Uniform distribution (because the minimum PDF value never changes with $\zeta$).

However, simulation results in \Cref{fig:min_pdf_regret} show that the Red distribution actually incurs lower regret than the Uniform distribution when $n$ is small. This is because when $F(\ha)$ is far away from $q$, the SAA action $\ha$ is actually closer to $a^*$ under the Red CDF than the Uniform CDF.
Only when $n$ becomes sufficiently large does the regret ordering align with the prediction from the minimum PDF.
Meanwhile, our notion of clustered distributions accurately captures this crossover behavior, in \Cref{fig:min_pdf_delta}.

\paragraph{Distribution and variance of regrets with different $n$.}
We plot the empirical distribution of additive regrets and report their mean, 95th percentile, and variance in \Cref{fig:regDistribution}. To elaborate, we focus on the continuous demand distributions used in \Cref{fig:addMean}, under the same values $q=0.4$ and $q=0.9$. For each demand distribution and each value of $q$, we generate regret distributions for $n=11$ and $n=196$, representing small and large numbers of samples respectively and matching the values of $n$ considered in \Cref{tab:value}. Each regret distribution is produced using 10,000 repetitions.

\begin{figure}[t!]
    \centering
    \begin{subfigure}{\textwidth}
        \centering
        \includegraphics[width=\linewidth]{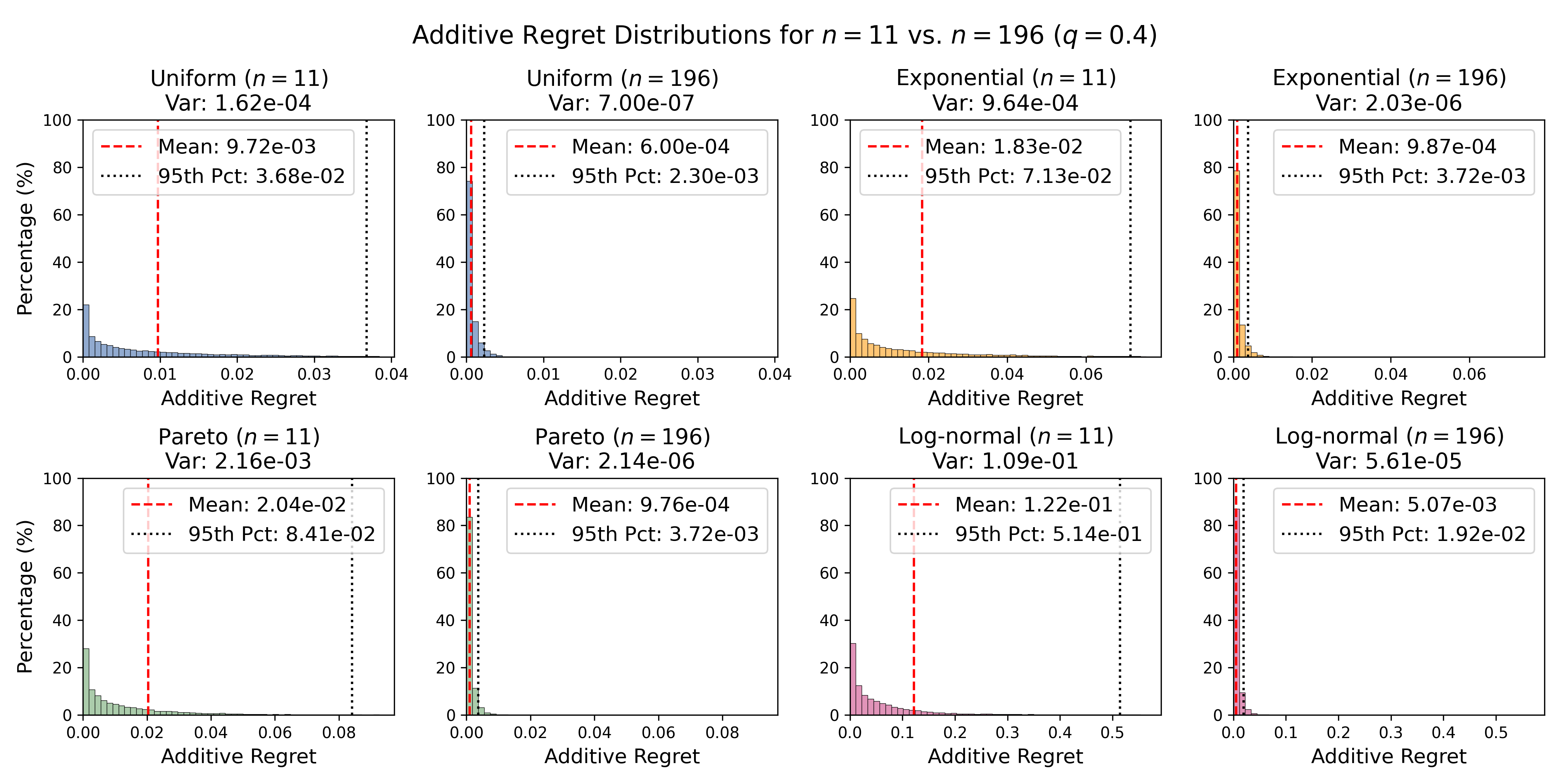} 
        \caption{$q=0.4$}
        \label{fig:regDistributionA}
    \end{subfigure}
    \vspace{0.4cm}
    \begin{subfigure}{\textwidth}
        \centering
        \includegraphics[width=\linewidth]{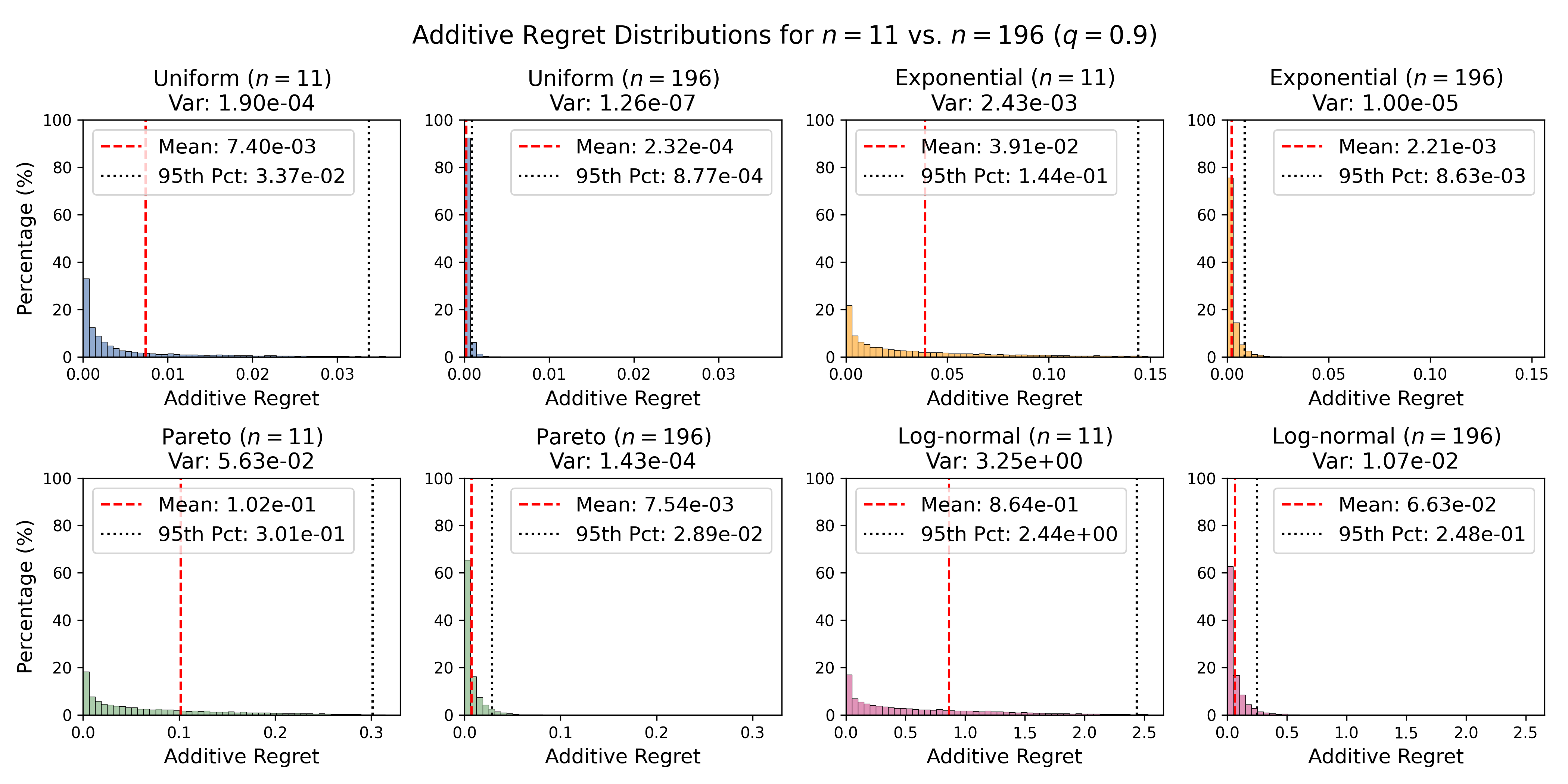}
        \caption{$q=0.9$}
        \label{fig:regDistributionB}
    \end{subfigure}
\caption{Empirical distributions, means, 95th percentiles, and variances of additive regrets under different values of $q$, $n$, and demand distributions.
% \Willedit{Note that when $n$ is 11 instead of 196, both the variance and the separation between mean and 95th-percentile regret are dramatically higher.}
}
\label{fig:regDistribution}
\end{figure}

For every regret distribution in \Cref{fig:regDistribution}, we compute the mean, the 95th percentile (which represents the high-probability regret for $\delta=0.05$), and the variance. Within each pair of subplots associated with small and large $n$, we fix the range of the x-axis to make the change in the gap between the mean and the 95th percentile visually comparable. 
Across all distributions and both $q$ values, the variance and the mean-to-95th-percentile gap shrink markedly as $n$ increases, indicating that comparable expectation and high-probability regrets (as our bounds imply) can only be hoped for when $n$ is sufficiently large. These observations provide empirical motivation for imposing a lower bound on $n$ in our high-probability bounds (\Cref{thm:hpAdd,thm:hpMult}).

\section{$\beta=\infty$ Cases of Multiplicative Regret}

\subsection{High-probability Upper Bound} \label{sec:betaInftyHP}

We prove \Cref{thm:hpMult} for $\beta=\infty$.
To do so, we first analyze the case where $\ha\le a^*$. We derive from~\eqref{eqn:Loss} that
\begin{align*}
    L(a^*)
    &\ge\int_{\ha}^{a^*} (1-q)F(z) dz\\
    &\ge(1-q)F(\ha)(a^*-\ha)\\
    &\ge(1-q)\left(q-\sup_{a\ge0}|\hF(a)-F(a)|\right)(a^*-\ha),
\end{align*}
where the last inequality follows from \eqref{whpDiffBound1}. By~\eqref{eqn:DKWoutcome} and the assumption that $n>\frac{\log(2/\delta)}{2(\min\{q,1-q\})^2}$, we have $q>\sqrt{\frac{\log(2/\delta)}{2n}}\ge\sup_{a\ge0}|\hF(a)-F(a)|$.
This enables us to derive from~\eqref{eqn:whpDiff} that
\begin{align}
    \frac{L(\ha)-L(a^*)}{L(a^*)}
    &=\frac{\int_{\ha}^{a^*} (q-F(z)) dz}{L(a^*)} \nonumber\\
    &\le\frac{(a^*-\ha)(q-F(\ha))}{(1-q)\left(q-\sup_{a\ge0}|\hF(a)-F(a)|\right)(a^*-\ha)} \nonumber\\
    &\le\frac{\sup_{a\ge0}|\hF(a)-F(a)|}{q(1-q)-(1-q)\sup_{a\ge0}|\hF(a)-F(a)|}, \label{hpMultBound1}
\end{align}
where the second inequality applies \eqref{whpDiffBound1}.

For the case where $\ha>a^*$, by \eqref{eqn:Loss} and properties of the Riemann integral, we have
\begin{align*}
    L(a^*)
    &\ge\int_{a^*}^{\ha} q(1-F(z)) dz\\
    &\ge\lim_{a\to\ha^-}q(1-F(a))(a-a^*)\\
    &\ge q\left(1-q-\sup_{a\ge0}|\hF(a)-F(a)|\right)(\ha-a^*),
\end{align*}
where the last inequality follows from \eqref{whpDiffBound2}.
By~\eqref{eqn:DKWoutcome} and the assumption that $n>\frac{\log(2/\delta)}{2(\min\{q,1-q\})^2}$, we have $1-q>\sqrt{\frac{\log(2/\delta)}{2n}}\ge\sup_{a\ge0}|\hF(a)-F(a)|$.
Therefore, we derive from~\eqref{eqn:whpDiff} that
\begin{align}
    \frac{L(\ha)-L(a^*)}{L(a^*)}
    &=\frac{\int_{a^*}^{\ha} (F(z)-q) dz}{L(a^*)} \nonumber\\
    &\le\frac{\lim_{a\to\ha^-}(a-a^*)(F(a)-q)}{q\left(1-q-\sup_{a\ge0}|\hF(a)-F(a)|\right)(\ha-a^*)}\nonumber\\
    &\le\frac{\sup_{a\ge0}|\hF(a)-F(a)|}{q(1-q)-q\sup_{a\ge0}|\hF(a)-F(a)|}, \label{hpMultBound2}
\end{align}
where the first inequality is by properties of the Riemann integral, and the last inequality uses~\eqref{whpDiffBound2}.

Combining \eqref{hpMultBound1} and \eqref{hpMultBound2}, we conclude that
\begin{align*}
    \frac{L(\ha)-L(a^*)}{L(a^*)}
    &\le\frac{\sup_{a\ge0}|\hF(a)-F(a)|}{q(1-q)-\max\{q,1-q\}\sup_{a\ge0}|\hF(a)-F(a)|}\\
    &=\frac{\sup_{a\ge0}|\hF(a)-F(a)|}{\max\{q,1-q\}\left(\min\{q,1-q\} - \sup_{a\ge0}|\hF(a)-F(a)|\right)}\\
    &\le\frac{2\sup_{a\ge0}|\hF(a)-F(a)|}{\min\{q,1-q\} - \sup_{a\ge0}|\hF(a)-F(a)|}
\end{align*}
holds for both cases 
under the assumption that $n>\frac{\log(2/\delta)}{2(\min\{q,1-q\})^2}$. Applying \eqref{eqn:DKWoutcome}, we have that with probability at least $1-\delta$,
\begin{align*}
    \frac{L(\ha)-L(a^*)}{L(a^*)}
    \le\frac{2\sqrt{\frac{\log(2/\delta)}{2n}}}{\min\{q,1-q\}-\sqrt{\frac{\log(2/\delta)}{2n}}}
    =\frac{2}{\min\{q,1-q\}\sqrt{\frac{2n}{\log(2/\delta)}}-1}.
\end{align*}

\subsection{Expectation Upper Bound} \label{sec:betaInftyExp}

We see from \eqref{eqn:expDiff} and \eqref{eqn:Loss} that
\begin{align*}
\bE[L(\ha)]-L(a^*) &=\int_0^{a^*}(q-F(z))\Pr[\hF(z)\ge q]dz+\int_{a^*}^\infty(F(z)-q)\Pr[\hF(z)<q]dz
\\ L(a^*) &=\int_0^{a^*} (1-q)F(z) dz+\int_{a^*}^\infty q(1-F(z))dz.
\end{align*}
Hence for any distribution with finite mean,
\begin{align*}
\frac{\bE[L(\ha)]-L(a^*)}{L(a^*)}
&=\frac{\int_0^{a^*}(q-F(z))\Pr[\hF(z)\ge q]dz+\int_{a^*}^\infty(F(z)-q)\Pr[\hF(z)<q]dz}{\int_0^{a^*} (1-q)F(z) dz+\int_{a^*}^\infty q(1-F(z))dz}\\
&\le\max\left\{\frac{\int_0^{a^*}(q-F(z))\Pr[\hF(z)\ge q]dz}{\int_0^{a^*} (1-q)F(z) dz},\frac{\int_{a^*}^\infty(F(z)-q)\Pr[\hF(z)<q]dz}{\int_{a^*}^\infty q(1-F(z))dz}\right\}\\
&=\max\left\{\sup_{F\in(0,q)}\frac{q-F}{(1-q)F}\Pr\left[\frac1n\Bin(n,F)\ge q\right],\sup_{F\in[q,1)}\frac{F-q}{q(1-F)}\Pr\left[\frac1n\Bin(n,F) < q\right]\right\}.
\end{align*}
This completes the upper bound on $\sup_{F:\mu(F)<\infty}\frac{\bE[L(\ha)]-L(a^*)}{L(a^*)}$.

Next we show that this bound is tight. By symmetry, we assume the maximum is achieved at some $F\in(0,q)$. Consider a Bernoulli distribution which takes the value 0 with probability $F$. Then we know that $a^*=1$, and the CDF of this distribution is
\begin{align*}
    F(z)=\begin{cases}
        0, &z<0\\
        F, &z\in[0,1)\\
        1, &z\ge1.
    \end{cases}
\end{align*}
So for this Bernoulli distribution, we derive from \eqref{eqn:expDiff} and \eqref{eqn:Loss} that
\begin{align*}
    \bE[L(\ha)]-L(a^*) 
    %&=\int_0^{a^*}(q-F(z))\Pr[\hF(z)\ge q]dz+\int_{a^*}^\infty(F(z)-q)\Pr[\hF(z)<q]dz\\
    &=\int_0^1 (q-F)\Pr[\hF(z)\ge q]dz
    =(q-F)\Pr\left[\frac1n\Bin(n,F)\ge q\right]\\
    L(a^*) 
    %&=\int_0^{a^*} (1-q)F(z) dz+\int_{a^*}^\infty q(1-F(z))dz\\
    &=\int_0^1(1-q)F dz
    =(1-q)F.
\end{align*}
This implies
\begin{align*}
    \frac{\bE[L(\ha)]-L(a^*)}{L(a^*)}=\frac{q-F}{(1-q)F}\Pr\left[\frac1n\Bin(n,F)\ge q\right],
\end{align*}
which shows that \eqref{eqn:expMultGen} is tight and that the supremum in $\sup_{F:\mu(F)<\infty}\frac{\bE[L(\ha)]-L(a^*)}{L(a^*)}$ can be achieved by Bernoulli distributions.

\section{Multiplicative Lower Bound} \label{sec:additionalLB}

We now lower-bound the multiplicative regret of any data-driven algorithm, showing it to be $\Omega(n^{-\frac{\beta+2}{2\beta+2}})$ with probability at least 1/3, which implies also a lower bound of $\Omega(n^{-\frac{\beta+2}{2\beta+2}})$ on the expected multiplicative regret.

\begin{theorem} \label{thm:lowMult}
Fix $q\in(0,1)$ and $\beta\in[0,\infty],\gamma\in(0,\infty),\zeta\in(0,(\min\{q,1-q\})^{\frac{1}{\beta+1}}/\gamma],\tau\in[0,\min\{q,1-q\}-(\gamma\zeta)^{\beta+1}]$.
Any learning algorithm based on $n$ samples makes a decision with multiplicative regret at least
\begin{align*}
    \frac{1}{16\gamma\zeta\tau+8q(1-q)}\left(\frac{(q-\tau)(1-q-\tau)}{3\sqrt{n}}\right)^{\frac{\beta+2}{\beta+1}}
    =\Omega\left(n^{-\frac{\beta+2}{2\beta+2}}\right)
\end{align*}
with probability at least 1/3 on some $(\beta,\gamma,\zeta)$-clustered distribution satisfying $F(a^*-\zeta)\ge\tau$ and $F(a^*+\zeta)\le1-\tau$.
Therefore, the expected multiplicative regret is at least
\begin{align*}
    \frac{1}{48\gamma\zeta\tau+24q(1-q)}\left(\frac{(q-\tau)(1-q-\tau)}{3\sqrt{n}}\right)^{\frac{\beta+2}{\beta+1}}
    =\Omega\left(n^{-\frac{\beta+2}{2\beta+2}}\right).
\end{align*}
\end{theorem}

\begin{proof}[Proof of \Cref{thm:lowMult}]
Let $C=\frac{(q-\tau)(1-q-\tau)}{3}$, $H=\frac1\gamma\left(\frac{C}{\sqrt{n}}\right)^{\frac{1}{\beta+1}}$.
Consider two distributions $P$ and $Q$, whose respective CDF functions $F_P$ and $F_Q$ are:
\begin{align*}
    F_P(z)&=
    \begin{cases}
        0, &z\in(-\infty, 0)\\
        \tau, &z\in[0,2\zeta)\\
        q+\frac{C(z-2\zeta)}{H\sqrt{n}}, &z\in[2\zeta,2\zeta+H)\\
        1-\tau, &z\in[2\zeta+H,4\zeta+H)\\
        1,&z\in[4\zeta+H,\infty);
    \end{cases}\\
    F_Q(z)&=
    \begin{cases}
        0, &z\in(-\infty, 0)\\
        \tau, &z\in[0,2\zeta)\\
        q+\frac{C(z-2\zeta-H)}{H\sqrt{n}}, &z\in[2\zeta,2\zeta+H)\\
        1-\tau, &z\in[2\zeta+H,4\zeta+H)\\
        1,&z\in[4\zeta+H,\infty).
    \end{cases}
\end{align*}
We let $L_P(a)$ and $L_Q(a)$ denote the respective expected loss functions under true distributions $P$ and $Q$, and from the CDF functions, it can be observed that the respective optimal decisions are $a_P^*=2\zeta$ and $a_Q^*=2\zeta+H$. We now show that any learning algorithm with $n$ samples will incur a multiplicative regret at least $\frac{1}{16\gamma\zeta\tau+8q(1-q)}\left(\frac{(q-\tau)(1-q-\tau)}{3\sqrt{n}}\right)^{\frac{\beta+2}{\beta+1}}$ with probability at least 1/3, on distribution $P$ or $Q$.

\paragraph{Establishing validity of distributions.}
First we show that both $P$ and $Q$ are $(\beta,\gamma,\zeta)$-clustered distributions. 
For distribution $P$, which has $a^*=2\zeta$, it suffices to verify \eqref{eqn:clustered} on $z\in[\zeta,3\zeta]$. We split the interval into segments $[\zeta,2\zeta)$ and $[2\zeta,3\zeta]$.
When $z$ is in the first segment, $F_P(z)=\tau$, so
\begin{align*}
    |F_P(z)-q|=q-\tau
    \ge(\gamma\zeta)^{\beta+1}
    \ge(\gamma |z-2\zeta|)^{\beta+1},
\end{align*}
where the first inequality follows from $\tau\in(0,\min\{q,1-q\}-(\gamma\zeta)^{\beta+1}]$, verifying \eqref{eqn:clustered}.
When $z$ is in the second segment, for the case where $\zeta<H$, it suffices to verify \eqref{eqn:clustered} on $z\in[2\zeta,2\zeta+H)$. We have
\begin{align*}
    |F_P(z)-q|
    =\frac{C(z-2\zeta)}{H\sqrt{n}}
    =\gamma^{\beta+1}H^\beta(z-2\zeta)
    >(\gamma|z-2\zeta|)^{\beta+1},
\end{align*}
where the second equality applies $\frac{C}{\sqrt{n}}=(\gamma H)^{\beta+1}$ and the inequality follows from $H>z-2\zeta$, verifying \eqref{eqn:clustered}.
On the other hand, for the case where $\zeta\ge H$, it remains to verify \eqref{eqn:clustered} on $z\in[2\zeta+H,3\zeta]$. We have $F_P(z)=1-\tau$, so
\begin{align*}
    |F_P(z)-q|=1-\tau-q
    \ge(\gamma\zeta)^{\beta+1}
    \ge(\gamma|z-2\zeta|)^{\beta+1},
\end{align*}
where the first inequality follows from $\tau\in(0,\min\{q,1-q\}-(\gamma\zeta)^{\beta+1}]$, again verifying~\eqref{eqn:clustered}. 
Therefore $P$ is a $(\beta,\gamma,\zeta)$-clustered distribution. 
It can be verified by symmetry that $Q$ is also a $(\beta,\gamma,\zeta)$-clustered distribution.

In addition, because $C=(q-\tau)(1-q-\tau)/3$, we obtain using the fact $\tau<q<1-\tau$ that
\begin{align*}
    \lim_{z\to (2\zeta+H)^-}F_P(z)&=q+\frac{C}{\sqrt{n}}=q+\frac{(q-\tau)(1-q-\tau)}{3\sqrt{n}}<q+\frac{1-q-\tau}{3}<1-\tau\le1\\
    F_Q(2\zeta)&=q-\frac{C}{\sqrt{n}}=q-\frac{(q-\tau)(1-q-\tau)}{3\sqrt{n}}>q-\frac{q-\tau}{3}>\tau\ge0
\end{align*}
which ensures the monotonicity of the CDF's for $P$ and $Q$.

Finally, we have
\begin{gather*}
    F_P(\zeta)=\tau\le F_Q(\zeta+H)\\
    F_P(3\zeta)\le1-\tau=F_Q(3\zeta+H),
\end{gather*}
which ensures  $F(a^*-\zeta)\ge\tau$ and $F(a^*+\zeta)\le1-\tau$ for both $P$ and $Q$.

\paragraph{Upper-bounding the probabilistic distance between $P$ and $Q$.} 
We analyze the squared Hellinger distance between distributions $P$ and $Q$.  Because $P$ and $Q$ only differ in terms of their point masses on $2\zeta$ and $2\zeta+H$, standard formulas for Hellinger distance yield
\begin{align*}
    &\quad \mathrm{H}^2(P,Q)\\
    &=\frac12\left(\left(\sqrt{q-\tau}-\sqrt{q-\tau-\frac{C}{\sqrt{n}}}\right)^2+\left(\sqrt{1-q-\tau-\frac{C}{\sqrt{n}}}-\sqrt{1-q-\tau}\right)^2\right)\\
    % &=\frac12\left(q-\tau+q-\tau-\frac{C}{\sqrt{n}}-2\sqrt{q\left(q-\frac{C}{\sqrt{n}}\right)}+1-q-\tau+1-q-\tau-\frac{C}{\sqrt{n}}-2\sqrt{(1-q-\tau)\left(1-q-\tau-\frac{C}{\sqrt{n}}\right)}\right)\\
    &=\frac12\left(-\frac{2C}{\sqrt{n}}+2(q-\tau)-2(q-\tau)\sqrt{1-\frac{C}{(q-\tau)\sqrt{n}}}+2(1-q-\tau)-2(1-q-\tau)\sqrt{1-\frac{C}{(1-q-\tau)\sqrt{n}}}\right)\\
    &\le\frac12\left(-\frac{2C}{\sqrt{n}}+2(q-\tau)\left(\frac{C}{2(q-\tau)\sqrt{n}}+\frac{C^2}{2(q-\tau)^2n}\right)+2(1-q-\tau)\left(\frac{C}{2(1-q-\tau)\sqrt{n}}+\frac{C^2}{2(1-q-\tau)^2n}\right)\right)\\
    &=\frac12\left(\frac{C^2}{(q-\tau)n}+\frac{C^2}{(1-q-\tau)n}\right),
\end{align*}
where the inequality follows from $1-\sqrt{1-x}\le\frac x2+\frac{x^2}{2}$, $\forall x\in[0,1]$.
We note that we are substituting in $x=\frac{C}{(q-\tau)\sqrt{n}}$ and $x=\frac{C}{(1-q-\tau)\sqrt{n}}$, which are at most $1$ because $C=(q-\tau)(1-q-\tau)/3$.

Similar with the analysis in the proof of \Cref{thm:lowAdd}, we have
\begin{align*}
    \mathrm{TV}(P^n,Q^n)
    &\le\sqrt{2\mathrm{H}^2(P^n,Q^n)}\\
    &\le\sqrt{2n\mathrm{H}^2(P,Q)}\\
    &\le C\sqrt{\frac{1}{q-\tau}+\frac{1}{1-q-\tau}}\\
    &=\frac{\sqrt{(q-\tau)(1-q-\tau)(1-2\tau)}}{3}\\
    &\le\frac13.
\end{align*}

\paragraph{Lower-bounding the expected regret of any algorithm.}
Fix any (randomized) algorithm for data-driven Newsvendor, and consider the sample paths of its execution on the distributions $P$ and $Q$ side-by-side.  The sample paths can be coupled so that the algorithm makes the same decision for $P$ and $Q$ on an event $E$ of measure $1-\mathrm{TV}(P^n,Q^n)\ge 2/3$, by definition of total variation distance.
Letting $A_P,A_Q$ be the random variables for the decisions of the algorithm on distributions $P,Q$ respectively,
we have that $A_P$ and $A_Q$ are identically distributed conditional on $E$.
Therefore, either 
$\Pr[A_P\ge2\zeta+\frac H2|E]=\Pr[A_Q\ge2\zeta+\frac H2|E]\ge1/2$
or 
$\Pr[A_P\le2\zeta+\frac H2|E]=\Pr[A_Q\le2\zeta+\frac H2|E]\ge 1/2$.

First consider the case where $\Pr[A_P\ge2\zeta+\frac H2|E]=\Pr[A_Q\ge2\zeta+\frac H2|E]\ge1/2$. By \eqref{eqn:Loss}, we have
\begin{align*}
    L_P(a_P^*)
    &=\int_0^{2\zeta} (1-q)\tau dz+\int_{2\zeta}^{2\zeta+H} q\left(1-q-\frac{C(z-2\zeta)}{H\sqrt{n}}\right)dz+\int_{2\zeta+H}^{4\zeta+H} q\tau dz\\
    &=2\zeta\tau+q(1-q)H-\frac{qCH}{2\sqrt{n}}\\
    &<2\zeta\tau+\frac{q(1-q)}{\gamma},
\end{align*}
where the inequality follows from $H=\frac1\gamma\left(\frac{C}{\sqrt{n}}\right)^{\frac{1}{\beta+1}}\le\frac1\gamma$ and $\frac{qCH}{2\sqrt{n}}>0$.
Note that if $A_P\ge2\zeta+\frac H2$, then we can derive from~\eqref{eqn:whpDiff} that under the true distribution $P$,
\begin{align*}
L_P(A_P) - L_P(a^*_P)
&=\int_{2\zeta}^{A_P}(F_P(z)-q)dz\\
&\ge\int_{2\zeta}^{2\zeta+\frac H2}(F_P(z)-q)dz\\
&=\int_{2\zeta}^{2\zeta+\frac H2} \frac{C(z-2\zeta)}{H\sqrt{n}}dz\\
&=\frac{CH}{8\sqrt{n}}\\
&=\frac{1}{8\gamma}\left(\frac{(q-\tau)(1-q-\tau)}{3\sqrt{n}}\right)^{\frac{\beta+2}{\beta+1}}.
\end{align*}
Therefore, we would have
\begin{align*}
\frac{L_P(A_P)-L_P(a^*_P)}{L_P(a^*_P)}
>\frac{1}{16\gamma\zeta\tau+8q(1-q)}\left(\frac{(q-\tau)(1-q-\tau)}{3\sqrt{n}}\right)^{\frac{\beta+2}{\beta+1}}
\end{align*}
with probability at least 
$\Pr\left[A_P\ge2\zeta+\frac H2\right]
\ge\Pr\left[A_P\ge2\zeta+\frac H2|E\right]\Pr[E]
\ge\frac12 \cdot\frac 23
=\frac13$.

Now consider the other case where $\Pr[A_P\le2\zeta+\frac H2|E]=\Pr[A_Q\le2\zeta+\frac H2|E]\ge 1/2$.
By \eqref{eqn:Loss}, we have
\begin{align*}
    L_Q(a_Q^*)
    &=\int_0^{2\zeta}(1-q)\tau dz+\int_{2\zeta}^{2\zeta+H}(1-q)\left(q+\frac{C(z-2\zeta-H)}{H\sqrt{n}}\right)dz+\int_{2\zeta+H}^{4\zeta+H} q\tau dz\\
    &=2\zeta\tau+q(1-q)H-\frac{(1-q)CH}{2\sqrt{n}}\\
    &<2\zeta\tau+\frac{q(1-q)}{\gamma}
\end{align*}
where the inequality follows from $H\le\frac1\gamma$ and $\frac{(1-q)CH}{2\sqrt{n}}>0$. 
If $A_Q\le2\zeta+\frac H2$, then we can derive from \eqref{eqn:whpDiff} that under the true distribution $Q$,
\begin{align*}
L_Q(A_Q) - L_Q(a^*_Q)
&=\int_{A_Q}^{2\zeta+H}(q-F_Q(z))dz\\
&\ge\int_{2\zeta+\frac H2}^{2\zeta+H}(q-F_Q(z))dz\\
&=\int_{2\zeta+\frac H2}^{2\zeta+H}\frac{C(2\zeta+H-z)}{H\sqrt{n}} dz\\
&=\frac{CH}{8\sqrt{n}}\\
&=\frac{1}{8\gamma}\left(\frac{(q-\tau)(1-q-\tau)}{3\sqrt{n}}\right)^{\frac{\beta+2}{\beta+1}}.
\end{align*}
The proof then finishes analogous to the first case.
\end{proof}

\end{document}